\documentclass[11pt]{article}
\usepackage{latexsym,amsmath,url,caption,epsfig}
\usepackage{textcomp}
\usepackage{amsfonts,euscript}
\usepackage{graphicx}
\usepackage{amsmath}
\usepackage{amssymb}
\usepackage{amsfonts}
\usepackage{amsthm}
\usepackage{mathrsfs}
\usepackage[all]{xy}
\usepackage{epstopdf}
\usepackage{graphics}
\usepackage{subcaption}
\usepackage{rotating}
\usepackage{appendix}
\usepackage{multirow,rotating}
\usepackage{url}
\usepackage{setspace}
\usepackage{listings}
\usepackage{color}
\usepackage{tikz}
\usepackage[breaklinks=true]{hyperref}
\usepackage[english]{babel}
\usepackage{datetime}
\usepackage{booktabs}
\usepackage{rotfloat}
\usepackage[numbers]{natbib}
\usepackage{algorithmic}
\usepackage{algorithm}
\usepackage{framed}
\usepackage{authblk}
\graphicspath{{figures/}}
\newtheorem{theorem}{Theorem}
\newtheorem{corollary}{Corollary}

\newtheorem{lemma}{Lemma}
\newtheorem{proposition}{Proposition}
\newtheorem{assumption}{Assumption}
\newtheorem{definition}{Definition}
\newtheorem{example}{Example}
\newtheorem{remark}{Remark}

\allowdisplaybreaks

\newcommand{\Pilin}{\Pi_{\rm{Lin}}}
\DeclareMathOperator*{\essinf}{ess\,inf}

\newcommand{\var}{\textbf{Var}}
\newcommand{\E}{\mathbf{E}}
\def\argmax{\mathop{\rm arg\,max}}%
\def\argmin{\mathop{\rm arg\,min}}%
\renewcommand{\P}{\mathbf{P}}
\newcommand{\feas}{\mathcal{X}}
\newcommand{\actions}{\mathcal{A}}

\newcommand{\reals}{\mathbf{R}}
\newcommand{\W}{A}

\newcommand{\indep}{\rotatebox[origin=c]{90}{$\models$}}

\begin{document}

%

%%%%%%%%%%%%%%%%

% Outcomment only when entries are known. Otherwise leave as is and
%   default values will be used.
%\setcounter{page}{1}
%\VOLUME{00}%
%\NO{0}%
%\MONTH{Xxxxx}% (month or a similar seasonal id)
%\YEAR{0000}% e.g., 2005
%\FIRSTPAGE{000}%
%\LASTPAGE{000}%
%\SHORTYEAR{00}% shortened year (two-digit)
%\ISSUE{0000} %
%\LONGFIRSTPAGE{0001} %
%\DOI{10.1287/xxxx.0000.0000}%

% Author's names for the running heads
% Sample depending on the number of authors;
% \RUNAUTHOR{Jones}
% \RUNAUTHOR{Jones and Wilson}
% \RUNAUTHOR{Jones, Miller, and Wilson}
% \RUNAUTHOR{Jones et al.} % for four or more authors
% Enter authors following the given pattern:
%\RUNAUTHOR{}

% Title or shortened title suitable for running heads. Sample:
% \RUNTITLE{Bundling Information Goods of Decreasing Value}
% Enter the (shortened) title:

% Full title. Sample:
% \TITLE{Bundling Information Goods of Decreasing Value}
% Enter the full title:
\title{Distributionally Robust Batch Contextual Bandits}

\author[1]{Nian Si\thanks{niansi@stanford.edu}}
\author[1]{Fan Zhang\thanks{fzh@stanford.edu}}
\author[2]{Zhengyuan Zhou\thanks{zzhou@stern.nyu.edu}}
\author[1]{Jose Blanchet\thanks{jose.blanchet@stanford.edu}}
\affil[1]{Department of Management Science \& Engineering, Stanford University}
\affil[2]{Stern School of Business, New York University}
\date{}
\maketitle
\begin{abstract}
Policy learning using historical observational data is an important problem that has found widespread applications. Examples include selecting offers, prices, advertisements to send to customers, as well as selecting which medication to prescribe to a patient. However, existing literature rests on the crucial assumption that the future environment where the learned policy will be deployed is the same as the past environment that has generated the data -- an assumption that is often false or too coarse an approximation.
In this paper, we lift this assumption and aim to learn a distributionally robust policy with incomplete observational data. We first present a policy evaluation procedure that allows us to assess how well the policy does
under worst-case environment shift. We then establish a central limit theorem
type guarantee for this proposed policy evaluation scheme.
Leveraging this evaluation scheme, we further propose a novel learning algorithm that is able to learn a policy that is robust to adversarial perturbations and unknown covariate shifts with  a performance guarantee based on the theory of uniform convergence. Finally, we empirically test the effectiveness of our proposed algorithm in synthetic datasets and demonstrate that it provides the robustness that is missing using standard policy learning algorithms. We conclude the paper by providing a comprehensive application of our methods in the context of a real-world voting dataset.
\end{abstract}
\section{Introduction}

\label{sec:intro} As a result of the digitization of our economy, %the past decade has witnessed an explosion of
user-specific data has exploded across a variety of
application domains: electronic medical data in health care, marketing data
in product recommendation, and customer purchase/selection data in digital
advertising (\cite{bertsimas2007learning, LCLS2010, chapelle2014,
	bastani2020online, SBF2017}). Such growing availability of user-specific
data has ushered in an exciting era of personalized decision making, one
that allows the decision maker(s) to personalize the service decisions based
on each individual's distinct features. 
As such, heterogeneity across individuals (i.e. best recommendation decisions vary across individuals) can be intelligently exploited to achieve
better outcomes.

{%
	{ When abundant historical data are available, effective personalization can be achieved by learning a policy offline (i.e. from the collected data) that prescribes the right treatment/selection/recommendation based on individual characteristics. Such an approach has been fruitfully explored (see Section~\ref{sec:related}) and has witnessed tremendous success. However, this success is predicated on (and hence restricted to) the setting where the learned policy is deployed in the same environment from which past data has been collected. This restriction limits the applicability of the learned policy, because one would often want to deploy this learned policy in a new environment where the population characteristics are not entirely the same as before, even though the underlying personalization task is still the same. Such settings occur frequently in managerial contexts, such as when a firm wishes to enter a new market for the same business, hence facing a new, \textit{shifted} environment that is similar yet different. We highlight several examples below:
		
		$\bullet$\textit{ Product Recommendation in a New Market.}
		In product recommendation, different products and/or promotion offers are directed
		to different customers based on their covariates (e.g. age, gender, education background,
		income level, marital status) in order to maximize sales. Suppose the firm has enjoyed great success in the US market by deploying an effective personalized product recommendation scheme that is learned from its US operations data\footnote{Such data include a database of transactions, each of which records the consumer's individual characteristics, the recommended item, and the purchase outcome.}, and is now looking to enter a new market in Europe.
		What policy should the firm use initially, given that little transaction data in the new market is available? The firm could simply reuse the same recommendation policy that is currently deployed for the US market. However, this policy could potentially be ineffective because the population in the new market often has idiosyncratic features that are somewhat distinct from the previous market. For instance, the market demographics will be different; further, even two individuals with the same observable covariates in different markets could potentially have different preferences as a result of the cultural, political, and environmental divergences. Consequently, such an environment-level ``shift"  renders the previously learned policy \textit{fragile}.
		Note that in such applications, taking the standard online learning approach -- by gradually learning to recommend in the new market as more data in that market becomes available -- is both wasteful and risky. It is wasteful because it entirely ignores the US market data even though presumably the two markets still share similarities and useful knowledge/insights can be transferred. It is also risky because a ``cold start" policy may be poor enough to cause the loss of customers in an initial  phase, in which case little further data can be gathered. Moreover, there may be significant reputation costs associated with the choice of a poor cold start. Finally, many personalized content recommendation platforms -- such as news recommendation or video recommendation -- also face these problems when initiating a presence in a new market.
		
		$\bullet$\textit{ Feature-Based Pricing in a New Market.}
		In feature-based pricing, a platform sells a product with features $x_t$
		on day $t$ to a customer, and prices it at $p_t$, which corresponds to the action (assumed to take discrete values). The reward %$r_t$ is the
		is the revenue collected from the customer, which is $p_t$ if the customer decides to 
		purchase the product and $0$ otherwise. The (generally-unknown-to-the-platform) probability of the customer purchasing
		this product depends on both the price $p_t$ and the product $x_t$ itself. %Note that the expected reward in this case is $\mathbf{E}[p_t\mathbf{1}_{b_t = 1}]$.
		If the platform now wishes to enter a new market to sell its products, 
		it will need to learn a distributionally robust feature-based pricing policy (which maps $x_t$ to $p_t$) that takes into account possible distributional shifts which arise in the new market.

		$\bullet$\textit{ Loan Interest Rate Provisioning in a New Market}
		In loan interest rate provisioning, the loan provider (typically a bank) would gather individual information $x_t$ (such as personal credit history, outstanding loans, current assets and liabilities etc) from a potential borrower $t$, and based on that information, provision an interest rate $a_t$, which corresponds to our action here. In general, the interest rate $a_t$ will be higher for borrowers who have a larger default probability, and lower for borrowers who have little or no chance to default. Of course, the default probability is not observed, and depends on both the borrower's financial situation $x_t$ and potentially on $a_t$, which determines how much payment to make in each installment. For the latter, note that a higher interest rate would translate into a larger installment payment, which may deplete the borrower's cash flow and hence make default more likely. 
		What is often observed is the sequence of installment payments for many borrowers under a given environment over a certain horizon (say 30 years for a home loan). If the borrower defaults at any point, then all subsequent payments are zero. 
		With such information, the reward for the bank corresponding to an individual borrower is the present value of the stream of payments made by discounting the cash flows back to the time when the loan was made (using the appropriate market discount rate). A policy here would be one that selects the best interest rate to produce the largest expected present value of future installment payment streams. When opening up a new branch in a different area (i.e. environment), the bank may want to learn a distributionally robust interest rate provisioning policy so as to take into account the environment shift.
		A notable feature of home loans is that they often span a long period of time. As such, even in the same market, the bank may wish to have some built-in robustness level in case there are shifts over time.

		\subsection{Main Challenges}
		The aforementioned applications \footnote{Other applications include deploying a personalized education curriculum and/or digital tutoring plan based on students' characteristics in a different school/district than from where the data was collected.
			%2) dynamic pricing based on individual product features to maximize sales in a foreign market, using data collected for the domestic market. These problems are ubiquitous, and occur whenever the decision maker deploys a policy in a new environment.
		}  highlight the need to learn a personalization policy  that is robust to \textit{unknown shifts} in another environmentt, an environment of which the decision maker has little knowledge or data.
		Broadly speaking, there can be two sources of such shifts:
		\begin{enumerate}
			\item \textbf{Covariate shift:} The population composition -- and hence the marginal distribution
			of the covariates -- can change. For instance, in the product recommendation example,
			demographics in different markets will be different (e.g. markets in developed countries will have more people that are educated than those in developing countries), and certain segments of the population that are a majority in the old market might be a minority in the new market.
			
			\item \textbf{Concept drift:} How the outcomes depend on the covariates and prescription can also change, thereby resulting in different conditional distributions of the outcomes given the two. For instance, in product recommendation, large market segments in the US (e.g. a young population with high education level) choose to use cloud services to store personal data. In contrast, the same market segment in some emerging markets (where data privacy may be regulated differently) may prefer to buy flash drives to store at least some of the personal data. 
			%In the bond pricing example, foreign investors of similar profiles in the US may have different demand elasticities, which would lead to different sales under the same pricing policy.
		\end{enumerate}
		
		These two shifts, often unavoidable and unknown, bring forth significant challenges in provisioning a suitable policy in the new environment, at least during the initial deployment phase. For instance, a certain subgroup (e.g. educated females in their 50s or older that live in rural areas) may be under-represented in the old environment's population. In this case, the existing product recommendation data are insufficient to identify the optimal recommendation for this subgroup. This insufficiency is not a problem in the old environment, because the sub-optimal prescription for this subgroup will not significantly affect the overall performance of the policy given that the subgroup occurrence is not sufficiently frequent. However, with the new population, this subgroup may have a larger presence, in which case the incorrect prescription will be amplified and could translate into poor performance. In such cases, the old policy's performance will be particularly sensitive to the marginal distribution of the subgroup in the new environment, highlighting the danger of directly deploying the same policy that worked well before.
		
		Even if all subgroups are well-represented, the covariate shift will cause a problem if the decision maker is constrained to select a policy in a certain policy class such as trees or linear policy class, due to, for instance, interpretability and/or fairness considerations. In such cases, since one can only hope to learn the best policy in the policy class (rather than the absolute best policy), the optimal in-class policy will change when the underlying covariate marginal distribution shifts in the new environment, rendering the old policy potentially ineffective. Note that this could be the case even if the covariate marginal distribution is the only thing that changes.
		
		Additionally, even more challenging is the existence of concept drift.
		%, where conditioned on the same observable covariates and prescription, the outcomes are not the same between the two environments.
		Fundamentally, this type of shifts occurs because there are hidden characteristics at a population level (cultural, political, environmental factors or other factors that are beyond the decision maker's knowledge) that also influence the outcome, but are unknown, unobservable. and different across the environments. As such, the decision maker faces a challenging hurdle: because these population-level factors that may influence the outcome are unknown and unobservable,  making it is infeasible to explicitly model them in the first place, let alone deciding what policy to deploy as a result of them. therefore, the decision maker faces an ``\textit{unknown unknown}" dilemma when  choosing the right policy for the new environment.

		Situated in this challenging landscape, one naturally wonders if there is any hope to rigorously address the problem of policy learning in shifted environments with a significant degree of model uncertainty. This challenge leads to the following fundamental question:
		\textit{Using the (contextual bandits) data collected from one environment, can we learn a \textbf{robust} policy that would provide reliable worst-case guarantees in the presence of both types of environment shifts? If so, how can this be done in a data-efficient way?} Our goal is to answer this question in the affirmative, as we shall explain.
		
		\subsection{Our Message and Managerial Insights}
		
		To answer this question we adopt a mathematical framework which allows us to formalize and quantify environmental model shifts. First, we propose a distributionally robust
		formulation of policy learning in batch contextual bandits that accommodates both environment shifts. To overcome the aforementioned ``unknown unknown" challenge that presents  modelling difficulty, our formulation takes a general, fully non-parametric (and hence model-agnostic) approach to describe the shift at the distribution level: we allow the new environment to be an arbitrary distribution in a KL-neighborhood around the old environment's distribution. As such, the shift is succinctly represented by a single quantity: the KL-radius $\delta$. We then propose to learn a policy that has maximum value under the worst-case distribution,  that is optimal for a decision maker who wishes to maximize value against an adversary who selects a worst-case distribution to minimize value. Such a distributionally robust policy -- if learnable at all -- would provide decision makers with the guarantee that the value of deploying  this policy will never be worse -- and possibly better -- no matter where the new environment shifts  within this KL-neighborhood.
		
		Regarding the choice of $\delta$, we provide two complementary perspectives on its selection process from a managerial viewpoint; each useful in the particular context one is concerned with. First, when data across different environments are available, one can estimate $\delta$ using such data. Such an approach would work well (and is convenient) if the new environment is different in similar ways in nature compared to how those other environments differ. For instance, in the voting application we consider in this paper (the August 2006 primary election in Michigan \citep{gerber2008social}), voting turnout data from different cities have been collected. As such, when deploying a new policy to encourage voters to vote in a different city, one can use the $\delta$ that is estimated from data across those different cities (we describe the technical procedures for such estimation in Section~\ref{sec:real_data}). Second, we can view $\delta$ as a parameter that can vary and that trades off with the optimal distributionally robust value: the larger the $\delta$, the more conservative the decision maker, the smaller the optimal distributionally robust value. We can compute the optimal distributionally robust values (and the corresponding policies) -- one for each $\delta$ -- for a range of $\delta$s; see Figure~\ref{fig:insights} for an illustration. Inspecting Figure~\ref{fig:insights}, we see that the difference between the optimal distributionally robust value under $\delta = 0$ (i.e., no distributional shifts) and the optimal distributionally robust value for a given $\delta$ representing the price of robustness.
		%Under the default choice of using the old environment's optimal policy,
		\textit{If} the new environment had actually remained  unchanged, then deploying a robust policy ``eats" into the value. As such, the decision maker
		can think of this value reduction as a form of insurance premium budget in order to protect the downside, in case the new environment did shift in \textit{unexpected} ways. Consequently, under a given premium budget (i.e. the amount of per-unit profit/sales that the decision maker is willing to forgo), a conservative choice would be for the decision maker to select the largest $\delta$ where the difference is within this amount ($\delta^*$ in Figure~\ref{fig:insights}), and  have the maximum robustness coverage therein. Importantly, if the new environment ends up not shifting in the worst possible way or not as much, then the actual value will only be higher. In particular, if the new environment does not shift at all, then the insurance premium the decision maker ends up paying after using the distributionally robust policy (under $\delta^*$) is \textit{smaller} than the budget, because its value under the old environment is larger than its value under the corresponding worst-case shift. Consequently, selecting $\delta$ this way yields the optimal worst-case policy under a given budget\footnote{Of course, if the new environment ends up shifting a larger amount than $\delta^*$ \textit{and} also in the worst possible way, then the actual value could be even smaller. However, in such situations, one would be much worse off with just using the old environment's optimal policy, which is not robust at all.}.

		\begin{figure}[t!]
			\begin{center}
				\includegraphics[width = 8.5cm]{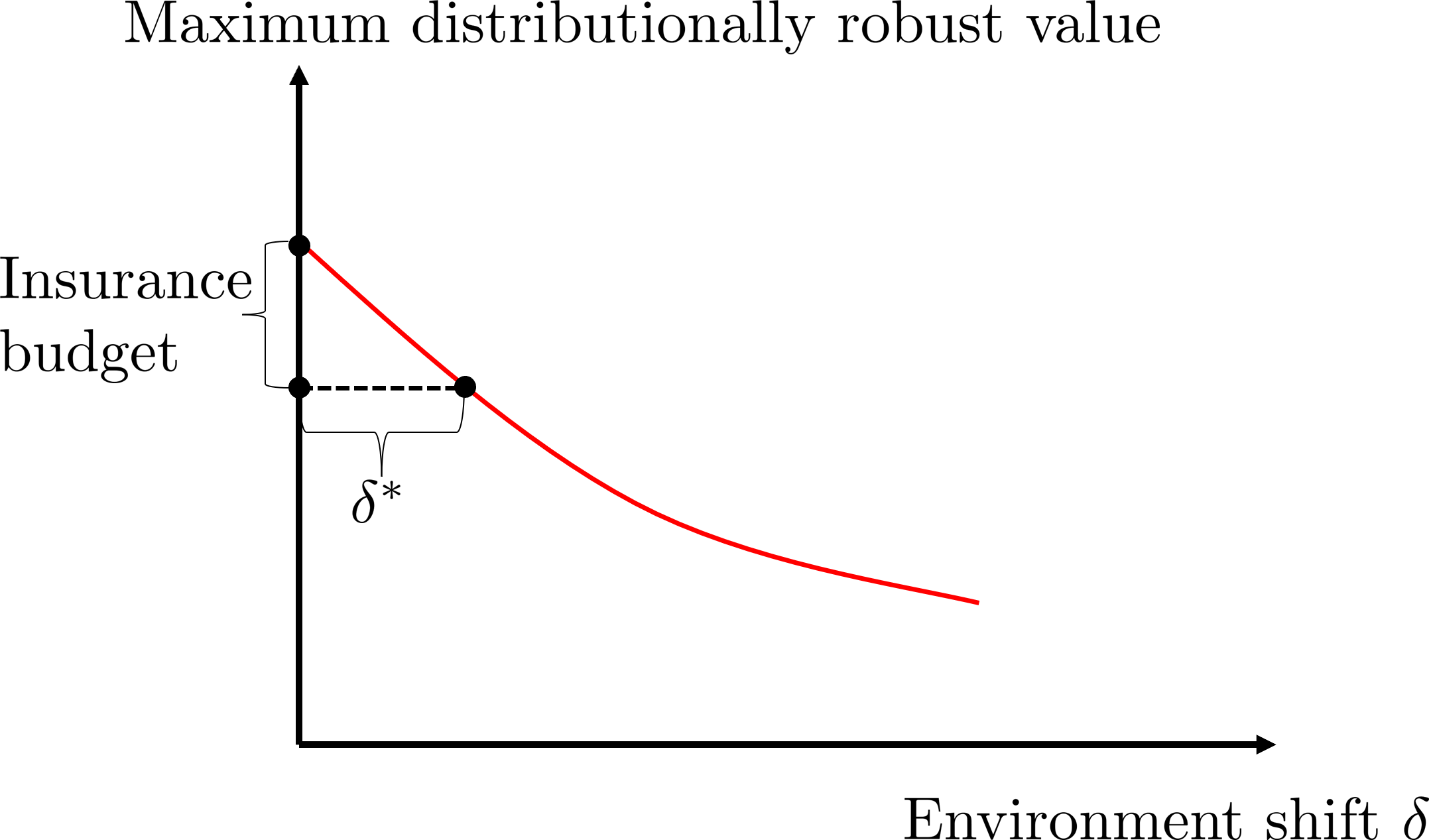}
			\end{center}
			\caption{Maximum distributionally robust values as a function of $\delta$. }
			\label{fig:insights}
		\end{figure}

		Second, we show that learning distributionally robust policies can indeed be done
		in a statistically efficient way. In particular, we provide an algorithmic framework
		that solves this problem \textit{optimally}. To achieve this, we first provide a novel scheme for distributionally robust policy evaluation (Algorithm~\ref{alg:DRO_evaluation}) that estimates the robust
		value of any given policy using historical data. We do so by drawing from
		duality theory and transforming the primal robust value estimation
		problem--an infinitely-dimensional problem--into a dual problem that is $1$-dimensional and convex, hence admitting an efficiently computable solution, which we can solve using Newton's method. We then establish, in the form of a
		central limit theorem, that the proposed estimator converges to the true
		value at an $O_p\left(n^{-1/2}\right)$ rate where $n $ is the number of data
		points. Building upon the estimator, we devise a distributionally robust policy learning algorithm
		(Algorithm~\ref{alg:DRO_policy_learning}) and establish that (Theorem~\ref{thm:pi_policy}) it achieves a $O_p(n^{-1/2})$ finite-sample regret. Such a finite-sample regret bound
		informs the decision maker that in order to learn an $\epsilon$-optimal distributionally robust policy, a dataset on the order of $\frac{1}{\epsilon^2}$ samples suffice with high probability.
		Note that this result is true for any $\delta$, where the regret bound itself does not depend on $\delta$.
		In addition, we also characterize the fundamental limit of this problem by establishing a $\Omega(n^{-1/2})$ lower bound for expected regret, thus making it clear that our policy learning algorithm is statistically optimal in the minimax sense. Taken together, these results highlight that we provide an optimal prescription framework for learning distributionally robust policies.
		
		Third, we demonstrate the empirical efficacy and efficiency of our proposed algorithm by
		providing extensive experimental results in Section~\ref{sec:numericals}. We dedicate Section \ref{sec:real_data} to the voting problem mentioned previously using our distributionally robust policy learning framework and demonstrate its applicability on a real-world dataset. Finally, we extend our results  to $f$-divergence measures and show our framework is still applicable
		even beyond KL-divergence (Section~\ref{sec:extension}).

		\subsection{Related Work}\label{sec:related}
		As mentioned, our work is closely related to the flourishing and rapidly developing  literature on offline policy learning in contextual bandits; see, e.g, \cite{langford2011doubly,
			zhang2012estimating,zhao2012estimating, zhao2014doubly,
			swaminathan2015batch, rakhlin2016bistro, kallus2017balanced,
			kitagawa2018should, kallus2018confounding, zhou2018offline,
			deep-learning-logged-bandit-feedback,chernozhukov2019semi}.
		Many valuable insights have been contributed: novel policy evaluation and policy learning algorithms
		have been developed; sharp minimax regret guarantees have been characterized
		in many different settings; and extensive and
		illuminating experimental results have been performed to offer practical
		advice for optimizing empirical performance. However,
		this line of work assumes the environment in which the learned policy will be deployed
		is the same as the the environment from which the training data is collected. In such settings,
		robustness is not a concern. Importantly, this line of work developed and used the family of doubly robust estimators for policy evaluation and learning~\citep{langford2011doubly,zhou2018offline}.
		For clarification, we point out that this family of estimators, although related to robustness, does not address
		the robustness discussed in this paper (i.e. robustness to environment shifts). In particular, those estimators aim to
		stabilize statistical noise and deal with mis-specified models for rewards and propensity scores, where the underlying environment distribution is the same across test and training environments.

		Correspondingly, there has also been an extensive literature on online
		contextual bandits, for example, ~\cite{LCLS2010,rusmevichientong2010linearly,FCGS2010, rigollet2010nonparametric,
			chu2011contextual, goldenshluger2013linear, AG2013a, AG2013b, RV2014,
			russo2016information, JBNW2017,
			LLZ2017,abeille2017linear,dimakopoulou2017estimation,LLZ2017}, whose focus
		is to develop online adaptive algorithms that effectively balance
		exploration and exploitation. This online setting is not the focus of our paper. See~\cite{BCN2012,
			lattimore2020bandit,slivkins2019introduction} for a few articulate
		expositions. Despite this, we do point out that
		as alluded to before and whenever possible, online learning can complement the distributionally robust policy learned and deployed initially. We leave this investigation for future work.
		
		Additionally, there is also rapidly growing literature in distributionally robust optimization (DRO); see, e.g, \cite{bertsimas2004price,delage2010distributionally,hu2013kullback,shafieezadeh-abadeh_distributionally_2015,bayraksan2015data,gao2016distributionally,namkoong2016stochastic,duchi2016statistics,staib2017distributionally,shapiro2017distributionally,lam2017empirical,volpi2018generalizing,raginsky_lee,nguyen2018distributionally,yang2018Wasserstein,MohajerinEsfahani2017,zhao2017distributionally,abadeh2018Wasserstein,ZHAO2018262,sinha2018certifiable,gao2018robust,wolfram2018,ghosh2019robust,blanchet2016quantifying,duchi2018learning,lam2019recovering,duchi2019distributionally,ho2020distributionally}.
		The existing DRO literature has mostly focused on the statistical learning
		aspects, including supervised learning and feature selection type problems,
		rather than on the decision making aspects.
		Furthermore, much of that literature uses DRO as tool to prevent over-fitting when it comes
		to making predictions, rather than  dealing with distributional shifts.
		To the best of our knowledge, we
		provide the first distributionally robust formulation for policy evaluation
		and learning under bandit feedback and shifted environments, in a general, non-parametric space.
		
		Some of the initial results appeared in the conference version~\citep{si2020distributionally}, which only touched a very limited aspect of the problem: policy evaluation under shifted environments. In contrast,
		this paper is substantially developed and fully addresses the entire policy learning problem. We summarize the main differences below:
		\begin{enumerate}
			\item The conference version focused on the policy evaluation problem and only studied a non-stable version of the policy evaluation scheme, which is outperformed by the stable policy evaluation scheme analyzed here (we simply dropped the non-stable version of the policy evaluation scheme in this journal version).
			\item The conference version did not study the policy learning problem, which is our ultimate objective. Here, we provide a policy learning algorithm and establish the minimax optimal rate $O_p(n^{-1/2})$ for the finite-sample regret by providing the regret upper bound as well as the matching regret lower bound.
			\item We demonstrate the applicability of our policy learning algorithms and provide results on a real-world voting data set, which is missing in the conference version.
			\item We provide practical managerial insights for the choice of the critical parameter $\delta$ which governs the size of distributional shifts.
			\item We finally extend our results to $f$-divergence measures, a broader class of divergence measures that include KL as a special case.
		\end{enumerate}
}}

\section{A Distributionally Robust Formulation of Batch Contextual Bandits}

\label{sec:model}

\subsection{Batch Contextual Bandits}

Let $\mathcal{A}$ be the set of $d$ actions: $\mathcal{A}=\{a^{1},a^{2},%
\dots ,a^{d}\}$ and let $\mathcal{X}$ be the set of contexts endowed with a $%
\sigma$-algebra (typically a subset of $\mathbf{R}^{p}$ with the Borel $%
\sigma$-algebra). Following the standard contextual bandits model, we posit
the existence of a fixed underlying data-generating distribution on $%
(X,Y(a^{1}),Y(a^{2}),\dots ,Y(a^{d}))\in \mathcal{X}\times \prod_{j=1}^{d}%
\mathcal{Y}_{j}$, where $X\in \mathcal{X}$ denotes the context vector, and
each $Y(a^{j})\in \mathcal{Y}_{j}\subset \mathbf{R}$ denotes the random
reward obtained when action $a^{j}$ is selected under context $X$.

Let $\{(X_i, A_i, Y_i)\}_{i=1}^n$ be $n$ \textbf{iid} observed triples that
comprise of the training data, where $(X_i, Y_i(a^1), \dots, Y_i(a^d))$ are
drawn \textbf{iid} from the fixed underlying distribution described above,
and we denote this underlying distribution by $\P _0$. Further, in the $i$%
-th datapoint $(X_i, A_i, Y_i)$, $A_i$ denotes the action selected and $Y_i
= Y_i(A_i)$. In other words, $Y_i$ in the $i$-th datapoint is the observed
reward under the context $X_i$ and action $A_i$. Note that all the other
rewards $Y_i(a)$ (i.e. for $a \in \mathcal{A} - \{A_i\}$), even though they
exist in the model (and have been drawn according to the underlying joint
distribution), are not observed.

We assume the actions in the training data are selected by some fixed
underlying policy $\pi_0$ that is known to the decision-maker, where $%
\pi_0(a\mid x)$ gives the probability of selecting action $a$ when the
context is $x$.
%We assume the policy used to select $A_i$ in the observational data is a regular conditional probability distribution $\pi_0(a\mid x)$.
%In other words, firstly $\pi_0(a\mid \cdot): x \rightarrow [0,1]$ is measurable as a function of $x$; secondly $\pi_0(\cdot \mid x): \mathcal{A}\rightarrow[0,1]$ is a well defined probability mass function as a function of $a$.
In other words, for each context $X_i$, a random action $A_i$ is selected
according to the distribution $\pi_0(\cdot \mid X_i)$, after which the
reward $Y_i(A_i)$ is observed. Finally, we use $\P _0*\pi_0 $ to denote
the product distribution on space $\mathcal{X} \times \prod_{j=1}^d \mathcal{%
	Y}_j\times \mathcal{A}$.
%For any $(x, a) \in \feas \times \actions$, we define the following two quantities:
%\begin{definition}\label{def:policy_data}\quad
%$e_{a}(x) \triangleq \pi_0[\W_i = a \mid X_i = x]$
%and $\mu_{a} (x)  \triangleq \E_{\overline{\P}_0}[Y_i(a) \mid X_i = x].$
%\end{definition}
%
%\begin{remark}
%	The above setup is a standard model that is also used in
%	% (for different purposes) in several domains. In the treatment effect estimation literature, this is known as potential outcome model (\cite{rubin1974estimating, rosenbaum1983central}).
%	contextual bandits~(\cite{bubeck2012regret}), where each action $a^j$ is known as an arm and the feature vector $x$ is called a context.
%	In general contextual bandits problems, $\mu_{a} (x)$ can be an arbitrary function of $x$ for each $a$. When $\mu_{a} (x)$ is a linear function of $x$, this is known as linear contextual bandits, an important and perhaps most extensively studied subclass of contextual bandits in the online learning context.
%	In this paper, we do not make any structural assumption on $\mu_{a} (x)$ and instead work with general
%	underlying data-generating distributions.
%	Furthermore, it should also be clear that the problem setup considered here is inherently offline (as opposed to online), since we work with data that is previously collected in one batch.
%\end{remark}
We make the following assumptions on the data-generating process.
\begin{assumption}\label{assump:classical}\quad
	The joint distribution $(X, Y(a^1), Y(a^2), \dots, Y(a^d), A)$ satisfies:
	\begin{enumerate}
		\item Unconfoundedness: $(Y(a^1),Y(a^2),\ldots,Y(a^d))$ is independent with $A$ conditional on $X$, i.e.,
		\[(Y(a^1),Y(a^2),\ldots,Y(a^d)) \indep A |X.
		\]
		
		\item Overlap:
		There exists some $\eta > 0$,  $\pi_0(a\mid x) \ge  \eta$, $\forall (x, a) \in \feas \times \actions$.
		\item Bounded reward support: $0 \leq Y(a^i)\leq M$ for $i=1,2,\ldots,d$.
	\end{enumerate}
\end{assumption}
\begin{assumption}[Positive densities/probabilities] The joint distribution $(X, Y(a^1), Y(a^2), \dots, Y(a^d))$ satisfies one of the following assumptions below:
	\begin{enumerate}
		\item In the continuous case, for any $i=1,2,\ldots ,d,$ $Y(a^{i})|X$ has a conditional density $f_{i}(y_i|x)$, which has a uniform non-zero lower bound, i.e., $ f_{i}(y_i|x)\geq \underline{b}>0$ over the interval $[0,M]$ for any $x \in \mathcal{X}$.
		\item In the discrete case, for any $i=1,2,\ldots ,d,$ $Y(a^{i})$ is supported on a finite set $\mathbb{D}$ with cardinality more than 1, and $Y(a^{i})|X$ satisfies  $ \P_0(Y(a^{i})=v|X)\geq \underline{b}>0$ almost surely for any $v \in \mathbb{D}$.
	\end{enumerate}
	\label{assump:pos_dens}
\end{assumption}

The overlap assumption ensures that some minimum positive probability is
guaranteed regardless of the context is. This assumption ensures sufficient
exploration in collecting the training data, and indeed, many operational
policies have $\epsilon$-greedy components. Assumption \ref{assump:classical}
is standard and commonly adopted in both the estimation literature (\cite%
{rosenbaum1983central, imbens2004nonparametric, imbens2015causal}) and the
policy learning literature (\cite{zhang2012estimating,zhao2012estimating,
	kitagawa2018should,swaminathan2015batch, zhou2015residual}). Assumption \ref%
{assump:pos_dens} is made to ensure the $O_p(n^{-1/2})$ convergence rate.

\begin{remark}
	In standard contextual bandits terminology,
	$\mu_{a} (x)  \triangleq \E_{\P_0}[Y_i(a) \mid X_i = x]$ is known as the mean reward function for action $a$. Depending on whether one assumes a parametric form of $\mu_{a} (x)$ or not, one needs to employ different statistical methodologies. In particular,  when $\mu_{a} (x)$ is a linear function of $x$, this setting is known as linear contextual bandits, an important and most extensively studied subclass of contextual bandits.
	In this paper, we do not make any structural assumption on $\mu_{a} (x)$: we are in the non-parametric contextual bandits regime and work with  general underlying data-generating distributions $\P_0$.
\end{remark}	

\subsection{Standard Policy Learning}

With the aforementioned setup, the standard goal is to learn a good policy from a
fixed deterministic policy class $\Pi$ using the training data,
often known as the batch contextual bandits problem (in contrast to online
contextual bandits), because all the data has already been collected
before the decision maker aims to learn a policy. A policy $\pi: \mathcal{X}
\rightarrow \mathcal{A}$ is a function that maps a context vector $x$ to an
action and the performance of $\pi$ is measured by the expected reward this
policy generates, as characterized by the policy value function:
\begin{definition}\label{def:policy_value}\quad
	The policy value function $Q: \Pi \rightarrow \reals$ is defined as: $Q(\pi) \triangleq \E_{\P_0}[Y(\pi(X))]$, where the expectation is taken with respect to the randomness in the underlying joint distribution $\P_0$ of $(X, Y(a^1), Y(a^2), \dots, Y(a^d))$.
\end{definition}

With this definition, the optimal policy is a policy that maximizes the
policy value function. The objective in the standard policy learning context
is to learn a policy $\pi$ that has the policy value as large as possible, which is equivalent to minimizing
the discrepancy between the performance of the
optimal policy and the performance of the learned policy $\pi$.

%This
%discrepancy is formalized by the notion of regret, as given in the next
%definition:
%\begin{definition}\quad
%	The regret $R(\pi)$ of a  policy $\pi \in \Pi$ is defined as:
%	$R(\pi) =\max_{\pi \in \Pi}\E_{\P_0}[Y(\pi(X))] - \E_{\P_0}[Y(\pi(X))]$.
%\end{definition}
%
%\begin{remark}
%	Recently, there has been a growing literature that studies how to efficiently learn (both statistically and computationally) a good policy $\hat{\pi}$ from the training data so that the regret is small.
%	Note that since $\hat{\pi}$ is learned from data and hence a random variable,
%	$R(\hat{\pi}_{\rm DRO})$ is a random variable.
%	A regret bound in such cases is customarily a high probability bound that highlights
%	how regret scales as a function of the size $n$ of the dataset, the error probability and other important parameters of the problem (e.g. the complexity of the policy class $\Pi$).
%	A quite extensive literature exists on this topic that gives minimax-optimal regret bounds under various assumptions.
%\end{remark}	

\subsection{Distributionally Robust Policy Learning}

Using the policy value function $Q(\cdot)$ as defined in Definition~\ref%
{def:policy_value} to measure the quality of a policy brings out an
implicit assumption that the decision maker is making: the environment that
generated the training data is the same as the environment where the policy
will be deployed. This is manifested in that the expectation in $Q(\cdot)$
is taken with respect to the same underlying distribution $\P _0$. However,
the underlying data-generating distribution may be different for the
training environment and the test environment. In such cases, the policy
learned with the goal to maximize the value under $\P _0$ may not work well
under the new test environment.

To address this issue, we propose a distributionally robust formulation for
policy learning, where we explicitly incorporate into the learning phase the
consideration that the test distribution may not be the same as the training
distribution $\P _0$. To that end, we start by introducing some terminology.
First, the KL-divergence between two probability measures $\P $ and $\P _0$,
denoted by $D(\P ||\P _0)$, is defined as $D(\P ||\P _0) \triangleq \int
\log\left(\frac{d\P }{d\P _0}\right)\mathrm{d}\P .$ With KL-divergence, we
can define a class of neighborhood distributions around a given
distribution. Specifically, the distributional uncertainty set $\mathcal{U}_{%
	\P _0}(\delta)$ of size $\delta$ is defined as $\mathcal{U}_{\P _0}(\delta)
\triangleq \{\P \ll \P _0\mid D(\P ||\P _0)\leq \delta\}$, where $\P \ll \P %
_0$ means $\P $ is absolutely continuous with respect to $\P _0$. When it is
clear from the context what the uncertainty radius $\delta$ is, we sometimes
drop $\delta$ for notational simplicity and write $\mathcal{U}_{\P _0}$
instead. {%
	{ We remark that in practice, $\delta$ can be selected
		empirically. For example, we can collect historical distributional data from different
		regions and compute distances between them. Then,
		although the distributional shift direction is unclear, a reasonable
		distributional shift size $\delta$ can be estimated.} Furthermore, we can also check the sensitivity of robust policy with respect to $\delta$ and choose an appropriate one according to a given insurance premium budget. We detail these two approaches in Section \ref{sec:select_delta}.}

\begin{definition}\quad
	For a given $\delta > 0$, the distributionally robust value function $Q_{\mathrm{\rm DRO}}:\Pi \rightarrow \reals$ is defined as: $Q_{\mathrm{\rm DRO}}(\pi) \triangleq\inf_{\P\in\mathcal{U}_{\P_0}(\delta)} \E_{\P}[Y(\pi(X))]$.
	\label{def:Q_DRO}
\end{definition}

In other words, $Q_{\mathrm{\mathrm{DRO}}}(\pi)$ measures the performance of
a policy $\pi$ by evaluating how it performs in the worst possible
environment among the set of all environments that are $\delta$-close to $\P %
_0$. With this definition, the optimal policy $\pi^*_\mathrm{\mathrm{DRO}}$
is a policy that maximizes the distributionally robust value function: $%
\pi^*_\mathrm{\mathrm{DRO}} \in \arg \max_{\pi \in \Pi} \{ Q_\mathrm{\mathrm{%
		DRO}} (\pi)\}$. { If such optimal policy does not exist, we can always construct a sequence of policies whose distributionally robust value converges to the supremum $\sup_{\pi \in \Pi}\{Q_{\rm DRO}(\pi)\}$. Then, all of our results can generalize to this case. Therefore, for simplicity, we assume the optimal policy exists.}
To be robust to the changes between the test environment and the
training environment, our goal is to learn a policy such that its
distributionally robust policy value is as large as possible, or
equivalently, as close to the best distributionally robust policy as
possible. We formalize this notion in Definition \ref{def:regret}.

\begin{definition}
	The distributionally robust regret $R_{\mathrm{\rm DRO}}(\pi)$ of a  policy $\pi \in \Pi$ is defined as
	\\
	$R_{\mathrm{\rm DRO}}(\pi) \triangleq\max_{\pi' \in \Pi}\inf_{\P\in\mathcal{U}_{\P_0}(\delta)}\E_{\P}[Y(\pi'(X))] - \inf_{\P\in\mathcal{U}_{\P_0}(\delta)}\E_{\P}[Y(\pi(X))]$.
	\label{def:regret}
\end{definition}

Several things to note. First, per its definition, we can rewrite regret as $%
R_{\mathrm{DRO}}(\pi) = Q_{\mathrm{DRO}}(\pi^*_\mathrm{\mathrm{DRO}}) - Q_{%
	\mathrm{DRO}}(\pi)$. {Second, the underlying random policy that has generated
	the observational data (specifically the $A_i$s) could be totally irrelevant with the policy class $\Pi$.}
Third, when a policy $\hat{\pi}$ is learned from data and hence  $R_{\mathrm{DRO}}(%
\hat{\pi})$ is a random variable, then a regret bound in such
cases is customarily a high probability bound that highlights how regret
scales as a function of the size $n$ of the dataset, the error probability
and other important parameters of the problem, e.g. the complexity of the
policy class $\Pi$.

Regarding some other definitions of regret, one may consider choices such as
\begin{equation*}
	\sup_{\P \in \mathcal{U}_{\P _0}(\delta)}\max_{\pi^{\prime }\in \Pi} (%
	\mathbf{E}_\P [Y(\pi^{\prime }(X))]-\mathbf{E}_\P [Y(\pi(X))]),
\end{equation*}
where for a fixed distribution $\P $, one compares the learned policy with
the best one that could be done under perfect knowledge of $\P $. In this
definition the adversary is very strong in the sense that it knows the test
domain. Therefore, a problem of this definition is that the regret does not
converge to zero when $n$ goes to infinity (even for a randomized policy $%
\pi $). We will not discuss this notion in this paper.

\section{Distributionally Robust Policy Evaluation}

\subsection{Algorithm}

\label{sec:alg} In order to learn a distributionally robust policy -- one
that maximizes $Q_{\mathrm{\mathrm{DRO}}}(\pi)$ -- a key step lies in
accurately estimating the given policy $\pi$'s distributionally robust
value. We devote this section to tackling this problem.

\begin{lemma}[Strong Duality] \quad\label{thm:strong_duality}
	For any policy $\pi\in\Pi$, we have
	\begin{align}
		\inf_{\P \in \mathcal{U}_{\P_0}(\delta)}\E_{\P}\left[Y(\pi(X))\right]
		&= \sup_{\alpha\geq 0}\left\{ -\alpha \log\E_{\P_0}\left[\exp(-Y(\pi(X))/\alpha)\right] - \alpha \delta\right\}\\
		& = \sup_{\alpha\geq 0}\left\{ -\alpha \log\E_{\P_{0}*\pi_{0}}\left[
		\frac{\exp(-Y(A)/\alpha)\mathbf{1}\{\pi(X) = A\}}{\pi_0(A\mid X)}\right] - \alpha \delta\right\},
		\label{eq:true_dual}
	\end{align}
	where $\mathbf{1}\{\cdot \}$ denotes the indicator function.
\end{lemma}%{
	%{Lemma \ref{thm:strong_duality} allows us to
		%transform an intractable infinite-dimensional optimization problem to a
		%tractable finite-dimensional problem. }}
The proof of Lemma \ref%
{thm:strong_duality} is in  \ref{sec:lma_cvx}.
\begin{remark}
	\label{remark:alpha=0}
	When $\alpha=0$, by  the discussion of Case 1 after Assumption 1 in \citet{hu2013kullback}, we  define
	\begin{equation*}
		-\alpha\log\E_{\P_0}\left[\exp(-Y(\pi(X))/\alpha)\right] - \alpha \delta |_{\alpha=0}= \essinf\{Y(\pi(X)\},
	\end{equation*}
	where $\essinf$ denotes the essential infimum.
	Therefore, $-\alpha\log\E_{\P_0}\left[\exp(-Y(\pi(X))/\alpha)\right] - \alpha \delta$ is right continuous at zero. { In fact, Lemma \ref{lemma:ld} in   \ref{sec:thm1} shows that the optimal value is not attained at $\alpha = 0$ if Assumption \ref{assump:pos_dens}.1 is enforced}.
\end{remark}

The above strong duality allows us to transform the original problem of
evaluating $\inf_{\P \in \mathcal{U}_{\P _0(\delta)}}\mathbf{E}_{\P }\left[%
Y(\pi(X))\right]$, where the (primal) variable is a infinite-dimensional
distribution $\P $ into a simpler problem where the (dual) variable is a
positive scalar $\alpha$. Note that in the dual problem, the expectation is
taken with respect to the same underlying distribution $\P _0$. This then
allows us to use an easily-computable plug-in estimate of the
distributionally robust policy value. To easily  reference  the
subsequent analysis of our algorithm, we capture the important terms in the
following definition.

\begin{definition}
	\label{def:phi_alpha}
	Let $\{(X_i, \W_i, Y_i)\}_{i=1}^n$ be a given dataset.
	We define
	\begin{equation*}W_i(\pi, \alpha) \triangleq \frac{\mathbf{1}\{\pi(X_i) = A_i\}}{\pi_0(A_i\mid X_i)}\exp(-Y_i(A_i)/\alpha), \ S_n^{\pi}\triangleq \frac{1}{n}\sum_{i=1}^{n}\frac{\mathbf{1}\{\pi(X_{i})=A_{i}\mathbf{\}}}{\pi _{0}\left( A_{i}|X_{i}\right) }
	\end{equation*}
	and
	\begin{equation*}
		\hat{W}_n(\pi,\alpha) \triangleq \frac{1}{nS_n^\pi}\sum_{i=1}^{n} W_i(\pi, \alpha).\end{equation*}
	We also define the dual objective function and the empirical dual objective function as
	\[\phi(\pi,\alpha)\triangleq -\alpha \log\E_{\P_0}\left[\exp(-Y(\pi(X))/\alpha)\right] - \alpha \delta ,\]
	and
	\[\hat{\phi}_{n}(\pi,\alpha) \triangleq -\alpha \log \hat{W}_n(\pi, \alpha) - \alpha \delta,\]
	respectively.
	
	Then, we define the distributionally robust value estimators and the optimal dual variable using the following notations.
	\begin{enumerate}
		\item
		The distributionally robust value estimator $\hat{Q}_{\mathrm{\rm DRO}}:\Pi \rightarrow \reals$ is defined by $\hat{Q}_{\mathrm{\rm DRO}}(\pi) \triangleq\sup_{\alpha\geq 0}\left\{ \hat{\phi}_{n}(\pi,\alpha)\right\}$.
		\item
		The optimal dual variable  $\alpha ^{\ast }(\pi)$ is  defined by $\alpha ^{\ast }(\pi)\triangleq\arg \max_{\alpha \geq 0}\left\{\phi(\pi,\alpha)\right\}$, and the empirical dual value is denoted as $\alpha_n(\pi) \in \arg\max_{\alpha \geq 0}\left\{\hat{\phi}_n(\pi,\alpha)\right\}$.
	\end{enumerate}
\end{definition}
$W_i(\pi,\alpha)$ is a realization of the random variable inside the
expectation in equation (\ref{eq:true_dual}), and we approximate $\mathbf{E}%
_{\P _{0}*\pi_{0}}\left[ \frac{\exp(-Y(A)/\alpha)\mathbf{1}\{\pi(X) = A\}}{%
	\pi_0(A\mid X)}\right]$ by its empirical average $\hat{W_n}(\pi,\alpha)$
with a normalization factor $S_n^{\pi}$. Note that $\mathbf{E}[S_n^{\pi}]=1$
and $S_n^{\pi}\rightarrow 1$ almost surely. Therefore, the normalized $\hat{%
	W_n}(\pi,\alpha)$ is asymptotically equivalent with the unnormalized $\frac{1%
}{n} \sum_{i=1}^n W_i(\pi,\alpha)$. The reason for dividing a normalization
factor $S_n^{\pi}$ is that it makes our evaluation more stable; see
discussions in \cite{si2020distributionally} and \cite{swaminathan2015self}.
The upper bound of $\alpha^{\ast }(\pi)$ proven in Lemma \ref{lemma:upper_bd}
of   \ref{sec:thm1} establishes the validity of the definitions $%
\alpha^{\ast}(\pi)$, namely, $\alpha^{\ast}(\pi)$ is attainable and unique.

\begin{remark}
	Another recent paper \citep{faury2020distributionally}  also discusses a similar problem. The estimator they propose is  equivalent to
	\begin{equation}
		\sup_{a\geq 0}-\alpha \log \left( \frac{1}{n}\sum_{i=1}^{n}\exp \left( -%
		\frac{1\left\{ \pi (X_{i})=A_{i}\right\} Y_{i}}{\alpha \pi _{0}\left(
			A_{i}|X_{i}\right) }\right) \right) -\alpha \delta .  \label{their_estimator}
	\end{equation}%
	We remark that their estimator is not consistent, namely, the estimator (\ref%
	{their_estimator}) does not converge to $\inf_{\P \in \mathcal{U}_{\P_0}(\delta)}\E_{\P}\left[Y(\pi(X))\right]$, when $n$ goes to infinity.
\end{remark}

By \citet[Proposition 1 and  their discussion following the
proposition]{hu2013kullback}, we provide a characterization of the worst case distribution in
Proposition \ref{prop:worst_case}.
\begin{proposition}[The Worst Case Distribution]
	\label{prop:worst_case}
	Suppose that Assumption \ref{assump:classical} is imposed. For any policy $\pi \in \Pi$, when $\alpha^*(\pi)>0$, we define  a probability measure $\P(\pi)$ supported on $\mathcal{X}\times
	\prod_{j=1}^{d}\mathcal{Y}_{j}$ such that
	\[
	\frac{{\rm d}\P(\pi)}{{\rm d}\P_0} = \frac{\exp(-Y(\pi(X)/\alpha^*(\pi)))}{\E_{\P_0}[\exp(-Y(\pi(X)/\alpha^*(\pi)))]},
	\]
	where ${{\rm d}\P(\pi)}/{{\rm d}\P_0}$ is the Radon-Nikodym derivative;  when $\alpha^*(\pi)=0$, we define
	\[\frac{{\rm d}\P(\pi)}{{\rm d}\P_0}  =\frac{\mathbf{1}\{Y(\pi(X))=\essinf\{Y(\pi(X)\}\}}{\P_0(Y(\pi(X))=\essinf\{Y(\pi(X)\})}.
	\]
	Then,  we have that ${\P(\pi)}$ is the unique worst case distribution, namely
	\[\P(\pi) = \argmin_{\P\in\mathcal{U}_{\P_0}(\delta)} \E_\P[Y(\pi(X)].
	\]
	
\end{proposition}
Proposition \ref{prop:worst_case} shows that the worst case measure $\P %
(\pi) $ is an exponentially tilted measure with respect to the underlying
measure $\P _0$, where $\P (\pi)$ puts more weights on the low end. Since $%
\alpha^*(\pi)$ can be approximated by $\alpha_n(\pi)$, and $\alpha_n(\pi)$
is explicitly computable as we shall see in Algorithm \ref%
{alg:DRO_evaluation}, we are able to understand how the worst case measure
behaves. Moreover, we show that the worst case measure $\P (\pi)$ maintains
mutual independence when $Y(a^{1}),\ldots,Y(a^{d})$ are mutually
independent conditional on $X$ under $\P _0$ in the following Corollary.
\begin{corollary}
	\label{cor:worst_case}
	Suppose that Assumptions \ref{assump:classical} and \ref{assump:pos_dens} is imposed and under $\P_0$, $Y(a^{1}),\ldots,Y(a^{d})$ are mutually independent conditional on $X$.  Then, for any policy $\pi \in \Pi$, under the worst case measure $\P(\pi)$,  $Y(a^{1}),\ldots,Y(a^{d})$ are still mutually independent conditional on $X$.
\end{corollary}
The proofs of Proposition \ref{prop:worst_case} and Corollary \ref%
{cor:worst_case} are in   \ref{sec:lma_cvx}.

To compute $\hat{Q}_{\mathrm{\mathrm{DRO}}}$, one needs to solve an
optimization problem to obtain the distributionally robust estimate of the
policy $\pi$. As the following lemma indicates, this optimization problem is easy
to solve.

\begin{lemma}	
	\label{lma:cvx}
	The empirical dual objective function $\hat{\phi}_{n}(\pi,\alpha)$ is concave in $\alpha$ and its partial derivative admits the expression
	\begin{eqnarray*}
		\frac{\partial}{\partial \alpha}
		\hat{\phi}_{n}(\pi,\alpha)
		&=& -\frac{\sum_{i=1}^n Y_i(A_i)W_i(\pi,\alpha)} {\alpha S_n^{\pi}\hat{W}_n(\pi, \alpha)}-\log \hat{W}_{n}(\pi,\alpha) - \delta,\\
		\frac{\partial^2}{\partial \alpha^2}
		\hat{\phi}_{n}(\pi,\alpha)
		&=& \frac{(\sum_{i=1}^n Y_i(A_i)W_i(\pi,\alpha))^2} {\alpha^3(S_n^{\pi})^2(\hat{W}_{n}(\pi,\alpha))^2}
		-\frac{\sum_{i=1}^n Y^2_i(A_i)W_i(\pi,\alpha)} {\alpha^3S_n^{\pi}\hat{W}_{n}(\pi,\alpha)}. \\
	\end{eqnarray*}
	Further, if the array $\{Y_i(A_i)\mathbf{1}\{\pi(X_i) = A_i\}\}_{i=1}^n$ has at least two different non-zero entries, then $\hat{\phi}_{n}(\pi,\alpha)$ is strictly-concave in $\alpha$.
\end{lemma}	
The proof of Lemma \ref{lma:cvx} is in   \ref{sec:lma_cvx}. Since the
optimization problem $\hat{Q}_{\mathrm{\mathrm{DRO}}}=\max_{\alpha \geq
	0}\left\{ \hat{\phi}_{n}(\pi,\alpha) \right\}$ is maximizing a concave
function, it can be computed using the Newton-Raphson method. Based on
all of the discussions above, we formally give the distributionally robust
policy evaluation algorithm in Algorithm~\ref{alg:DRO_evaluation}. By %
\citet[Section 8.8]{luenberger2015linear}, we have that $\hat{\phi}_n(\pi,\alpha)$ converges to the global maximum $\hat{Q}_{\mathrm{DRO}%
}(\pi)$ quadratically in Algorithm~\ref%
{alg:DRO_evaluation} if the initial value of $\alpha$ is sufficiently
closed to the optimal value.
\begin{algorithm}[ht]
	\caption{Distributionally Robust Policy Evaluation}
	\label{alg:DRO_evaluation}
	\begin{algorithmic}[1]
		\STATE \textbf{Input:} Dataset $\{(X_i, \W_i, Y_i)\}_{i=1}^n$, data-collecting policy $\pi_0$, policy $\pi \in \Pi$, and initial value of dual variable $\alpha$.
		\STATE \textbf{Output:} Estimator of the distributionally robust policy value $\hat{Q}_{\mathrm{\rm DRO}}(\pi)$.
		
		\REPEAT
		\STATE Let $W_i(\pi, \alpha) \gets  \frac{\mathbf{1}\{\pi(X_i) = A_i\}}{\pi_0(A_i \mid X_i)}\exp(-Y_i(A_i)/\alpha)$.
		\STATE Compute $S_n^{\pi}\gets  \frac{1}{n}\sum_{i=1}^{n}\frac{\mathbf{1\{}\pi(X_{i})=A_{i}\mathbf{\}}}{\pi _{0}\left( A_{i}|X_{i}\right) }$.
		\STATE Compute $\hat{W}_n(\pi,\alpha) \gets  \frac{1}{nS_n^{\pi}}\sum_{i=1}^{n} W_i(\pi, \alpha)$.
		\STATE
		Update $\alpha \gets  \alpha - (\frac{\partial}{\partial \alpha}\hat{\phi}_{n})/ (\frac{\partial^2}{\partial \alpha^2}\hat{\phi}_{n})$.
		\UNTIL{$\alpha$ converges.}
		
		\STATE \textbf{Return} $\hat{Q}_{\mathrm{\rm DRO}}(\pi) \gets  \hat{\phi}_{n}(\pi,\alpha)$.
	\end{algorithmic}
\end{algorithm}

\subsection{Theoretical Guarantee of Distributionally Robust Policy
	Evaluation}

In the next theorem, we demonstrate that the approximation error for policy
evaluation function $\hat{Q}_{\mathrm{\mathrm{DRO}}}(\pi)$ is $O_p(1/\sqrt{n}%
)$ for a fixed policy $\pi$.

\begin{theorem}
	\label{DRO_CLT}
	Suppose Assumptions \ref{assump:classical} and \ref{assump:pos_dens} are enforced, and
	define \[\sigma ^{2}(\alpha )=\frac{ \alpha^{2} }{\mathbf{E}[\exp \left( -Y(\pi (X))/\alpha
		\right) ]^2} \mathbf{E}\left[ \frac{1}{\pi _{0}\left( \pi
		(X)|X\right) } \left( \exp \left( -Y(\pi (X))/\alpha \right)
	-\mathbf{E}\left[ \exp \left( -Y(\pi (X))/\alpha \right) \right] \right)
	^{2} \right].\]
	Then, for any policy $\pi\in\Pi$, we have
	\begin{align*}
		\sqrt{n}\left( \hat{Q}_{\mathrm{\rm DRO}}(\pi)  - Q_{\mathrm{\rm DRO}}(\pi)
		\right) &\Rightarrow \mathcal{N}\left( 0,\sigma ^{2}(\alpha ^{\ast }(\pi))\right),  \text { if } \alpha^{\ast }(\pi)>0, \text{ and} \\
		\sqrt{n}\left( \hat{Q}_{\mathrm{\rm DRO}}(\pi)  - Q_{\mathrm{\rm DRO}}(\pi)
		\right) &\rightarrow 0  \text { in probability, if } \alpha^{\ast }(\pi)=0,
	\end{align*}
	where $\alpha^{\ast}(\pi)$ is defined in Definition \ref{def:phi_alpha}, $\Rightarrow$ denotes convergence in distribution, and $\mathcal{N}(0,\sigma^2)$ is the normal distribution with mean zero and variance $\sigma^2$.
\end{theorem}

{%
	{Theorem \ref{DRO_CLT} ensures that we are able to evaluate the
		performance of a policy in a new environment  using only the training data
		even if the new environment is different from the training environment.}}
The proof of Theorem \ref{DRO_CLT} is in   \ref{sec:thm1}. \label%
{sec:CLT}
%The proof of Theorem \ref{DRO_CLT} is an analog of Theorem 5.7 in \cite{shapiro2009lectures}.

\section{Distributionally Robust Policy Learning}

\label{sec:algorithm}

In this section, we study the policy learning aspect of the problem and
discuss both the algorithm and its corresponding finite-sample theoretical
guarantee. {The aim is to find a robust policy that has reasonable performance in a new environment with unknown distributional shifts.} First, with the distributionally robust policy evaluation scheme
discussed in the previous section, we can in principle compute the
distributionally robust optimal policy $\hat{\pi}_{\mathrm{DRO}}$ by picking
a policy in the given policy class $\Pi$ that maximizes the value of $\hat{Q}%
_{\mathrm{DRO}}$, i.e.
\begin{equation}
	\hat{\pi}_{\mathrm{DRO}}\in\argmax_{\pi \in \Pi }\hat{Q}_{\mathrm{DRO}}(\pi)
	= \argmax_{\pi \in \Pi}\;\sup_{\alpha\geq 0}\left\{ -\alpha \log \hat{W}%
	_n(\pi, \alpha) - \alpha \delta\right\}.  \label{eqn:pi_hat}
\end{equation}

How do we compute $\hat{\pi}_{\mathrm{DRO}}$? In general, this problem is
computationally intractable  since it is highly non-convex in its
optimization variables ($\pi$ and $\alpha$ jointly). However, following the
standard tradition in the machine learning and optimization literature , we can employ certain
approximate schemes that, although do not guarantee global convergence, are
computationally efficient and practically effective {(for example, greedy tree search \citep[Section 9.2]{friedman2001elements} for decision-tree policy classes and gradient descent \citep{ruder2016overview} for linear policy classes)}.

A simple and quite effective scheme is alternate minimization, given in
Algorithm \ref{alg:DRO_policy_learning}, where we learn $\hat{\pi}_{\mathrm{%
		\mathrm{DRO}}}$ by fixing $\alpha$ and minimizing on $\pi$ and then fixing $%
\pi$ and maximizing on $\alpha$ in each iteration. Since the value of $\hat{Q%
}_\mathrm{DRO}(\pi)$ is non-decreasing along the iterations of Algorithm \ref%
{alg:DRO_policy_learning}, the converged solution obtained from Algorithm %
\ref{alg:DRO_policy_learning} is a local maximum of $\hat{Q}_\mathrm{DRO}$.
In practice, to accelerate the algorithm, we only iterate once for $\alpha$
(line 8) using the Newton-Raphson step $\alpha \gets \alpha - \left(\frac{%
	\partial}{\partial \alpha}\hat{\phi}_{n}\right)/ \left(\frac{\partial^2}{%
	\partial \alpha^2}\hat{\phi}_{n}\right)$. Subsequent simulations (see next
section) show that this is often sufficient.

\begin{algorithm}
	\caption{Distributionally Robust Policy Learning}
	\label{alg:DRO_policy_learning}
	\begin{algorithmic}[1]
		\STATE \textbf{Input:} Dataset $\{(X_i, \W_i, Y_i)\}_{i=1}^n$, data-collecting policy $\pi_0$, and initial value of dual variable $\alpha$.
		\STATE \textbf{Output:} Distributionally robust optimal policy $\hat{\pi}_{\mathrm{\rm DRO}} $.
		
		\REPEAT
		\STATE Let $W_i(\pi, \alpha) \gets  \frac{\mathbf{1}\{\pi(X_i) =  A_i\}}{\pi_0(A_i \mid X_i)}\exp(-Y_i(A_i)/\alpha)$.
		\STATE Compute $S_n^{\pi}\gets  \frac{1}{n}\sum_{i=1}^{n}\frac{\mathbf{1\{}\pi(X_{i})=A_{i}\mathbf{\}}}{\pi _{0}\left( A_{i}|X_{i}\right) }$.
		\STATE Compute $\hat{W}_n(\pi,\alpha) \gets  \frac{1}{nS_n^{\pi}}\sum_{i=1}^{n} W_i(\pi, \alpha)$.
		\STATE
		Update $\pi \gets  \argmin
		_{\pi\in \Pi}\hat{W}_n(\pi,\alpha)$.
		\STATE
		Update $\alpha \gets \argmax_{\alpha>0}  \{\hat{\phi}_n(\pi,\alpha)\}$. %\alpha - (\frac{\partial}{\partial \alpha}\hat{\phi}_{n})/ (\frac{\partial^2}{\partial \alpha^2}\hat{\phi}_{n})$.
		\UNTIL{$\alpha$ converges.}
		
		\STATE \textbf{Return} $\pi$.
	\end{algorithmic}
\end{algorithm}

\subsection{Statistical Performance Guarantee}

We now establish the finite-sample statistical performance guarantee for the
distributionally robust optimal policy $\hat{\pi}_{\mathrm{DRO}}$. Before
giving the theorem, we first need to  define entropy integral in the policy
class, which represents the class complexity.
\begin{definition}
	Given the feature domain $\mathcal{X}$, a policy class $\Pi ,$ a set of $n$
	points $\{x_{1},\ldots ,x_{n}\}\subset \mathcal{X}$, define:
	\begin{enumerate}
		\item Hamming distance between any two policies $\pi _{1}$ and $\pi _{2}$ in
		$\Pi :H(\pi _{1},\pi _{2})=\frac{1}{n}\sum_{j=1}^{n}\mathbf{1}\{\pi
		_{1}(x_{j})\neq \pi _{2}(x_{j})\}.$
		
		\item $\epsilon $-Hamming covering number of the set $\{x_{1},\ldots
		,x_{n}\}:N_{H}^{\left( n\right) }\left( \epsilon ,\Pi ,\{x_{1},\ldots
		,x_{n}\}\right) $ is the smallest number $K$ of policies $\{\pi _{1},\ldots
		,\pi _{K}\}$ in $\Pi $, such that $\forall \pi \in \Pi ,\exists \pi
		_{i},H(\pi ,\pi _{i})\leq \epsilon .$
		
		\item $\epsilon $-Hamming covering number of $\Pi :N_{H}^{\left( n\right)
		}\left( \epsilon ,\Pi \right) \triangleq\sup \left\{ N_{H}^{(n)}\left( \epsilon ,\Pi
		,\{x_{1},\ldots ,x_{n}\}\right) |x_{1},\ldots ,x_{n}\in \mathcal{X}\right\} .
		$
		
		\item Entropy integral: $\kappa ^{(n)}\left( \Pi \right) \triangleq\int_{0}^{1}\sqrt{%
			\log N_{H}^{\left( n\right) }\left( \epsilon^2 ,\Pi \right) }{\rm d}\epsilon .$
	\end{enumerate}
\end{definition}
The defined entropy integral  is the same as Definition 4 in \cite%
{zhou2018offline}, which is a variant of the classical entropy integral
introduced in \cite{dudley1967sizes}, and the Hamming distance is a
well-known metric for measuring the similarity between two equal-length
arrays whose elements are supported on on discrete sets \citep{hamming1950error}. We
then discuss the entropy integrals $\kappa ^{(n)}\left( \Pi \right)$ for
different policy classes $\Pi$.
\begin{example}[Finite policy classes]
	For a policy class $\Pi_{\rm{Fin}}$ containing a finite number of policies, we have $\kappa ^{(n)}\left( \Pi_{\rm{Fin}} \right) \leq \sqrt{\log(|\Pi_{\rm{Fin}}|)},$ where $|\Pi_{\rm{Fin}}|$ denotes the cardinality of the set $\Pi_{\rm{Fin}}$.
	%\item \textbf{Linear policy class:} Lemma 5.
	% \item \textbf{Decision tree policy class:} \cite{zhou2018offline}.
\end{example}
The entropy integrals for the linear policy classes and decision-tree policy
classes are discussed in Section \ref{sec:numerical_linear_class} and
Section \ref{sec:decision_tree}, respectively. For the special case of
binary action, we have the following bound for the entropy integral by \cite%
{zhengyuan2021} (see the discussion following Definition 4).
%by Theorem 1 in \cite{haussler1995sphere}.
\begin{lemma}
	If $d=2$, we have $\kappa ^{(n)}\left( \Pi \right) \leq 2.5 \sqrt{VC(\Pi)}$, where $VC(\cdot)$ denotes the VC dimension defined in \cite{vapnik1971uniform}.
	\label{lemma:VC_bound}
\end{lemma}
This result can be further generalized to the multi-action policy learning
setting, where $d$ is greater than 2; see the proof of Theorem 2 in \cite%
{zhengyuan2021}.
\begin{lemma}
	\label{lemma:graph}
	We have $\kappa ^{(n)}\left( \Pi \right) \leq 2.5 \sqrt{\log(d)Graph(\Pi)}$, where $Graph(\cdot)$ denotes the
	\emph{%
		graph dimension} (see the definition in \cite{bendavid1995characterizations}).
\end{lemma}
Graph dimension is a direct generalization of VC dimension. There are many
papers that discuss the graph dimension and also a closely related
concept, Natarajan dimension; see, for example, \cite%
{daniely2011multiclass,daniely2012multiclass,zhengyuan2021,natarajan1989learning}%
.

Theorem \ref{DRO_Uniform} demonstrates that with high probability, the
distributionally robust regret of the learned policy $R_{\mathrm{DRO}}(\hat{%
	\pi}_{\mathrm{DRO}})$ decays at a rate upper bounded by $O_{p}(\kappa ^{(n)}/%
\sqrt{n})$.
%Notice that by Theorem \ref{DRO_CLT}, the lower bound of $R_{\mathrm{DRO}}(\hat{\pi}_{\mathrm{DRO}})$ is $O_{p}(1/% \sqrt{n})$. Therefore, if $\kappa ^{(n)}=O(1)$, the upper bound and lower bound match up to a constant.
\begin{theorem}
	Suppose Assumption \ref{assump:classical} is enforced. Then,  with probability  at least $1-\varepsilon$, under Assumption \ref{assump:pos_dens}.1, we have
	\begin{equation}
		R_{\rm DRO}(\hat{\pi}_{\rm DRO})  \leq \frac{4}{\underline{b}\eta\sqrt{n}}\left( 24(\sqrt{2}+1)\kappa ^{(n)}\left( \Pi \right)  +\sqrt{2\log\left(\frac{2}{\varepsilon}\right)}+C\right), \label{eqn:uniform_bd}
	\end{equation}
	where $C$ is a universal constant, and under Assumption \ref{assump:pos_dens}.2, when
	\[
	n \geq\left\{\frac{4}{\underline{b}\eta
	}\left( 24(\sqrt{2}+1)\kappa ^{(n)}\left( \Pi \right) +48\sqrt{|%
		\mathbb{D}|\log \left( 2\right) }+\sqrt{2\log \left( \frac{2}{\varepsilon}\right) }\right)\right\}^2  ,
	\]we have
	\begin{equation}
		\label{eqn:uniform_bd_discrete}
		R_{\mathrm{DRO}}(\hat{\pi}_{\mathrm{DRO}})\leq \frac{4M}{\underline{b}\eta
			\sqrt{n}}\left( 24(\sqrt{2}+1)\kappa ^{(n)}\left( \Pi \right) +48\sqrt{|%
			\mathbb{D}|\log \left( 2\right) }+\sqrt{2\log \left( \frac{2}{\varepsilon}\right) }\right),
	\end{equation}
	where $|\mathbb{D}|$ denotes the cardinality of the set $\mathbb{D}$.
	\label{DRO_Uniform}
\end{theorem}
{The key challenge to the proof of Theorem
	\ref{DRO_Uniform} is that $Q_{\rm DRO}(\pi)$ is hard to quantify, since it is a non-linear functional of the probability measure $\mathbf{P}$.
	Thanksfully, Lemmas \ref{lma:quantile} and \ref{lma:discrete_key} allow us to transform the hardness of analysis of $Q_{\rm DRO}(\pi)$ into the well-studied terms such as the quantile and the total variation distance. }
\begin{lemma}
	\label{lma:quantile}
	For any probability measures $\mathbf{P}_{1},\mathbf{P}_{2}$ supported on $\reals$, we have
	\begin{equation*}
		\left\vert \sup_{\alpha \geq 0}\left\{ -\alpha \log \mathbf{E}_{\mathbf{P}%
			_{1}}\left[ \exp \left( -Y/\alpha \right) \right] -\alpha \delta \right\}
		-\sup_{\alpha \geq 0}\left\{ -\alpha \log \mathbf{E}_{\mathbf{P}_{2}}\left[
		\exp \left( -Y/\alpha \right) \right] -\alpha \delta \right\} \right\vert
		\leq \sup_{t\in \lbrack 0,1]}\left\vert q_{\mathbf{P}_{1}}\left( t\right)
		-q_{\mathbf{P}_{2}}\left( t\right) \right\vert,
	\end{equation*}
	where $q_{\mathbf{P}}\left( t\right)$ denotes the $t$-quantile of a probability measure $\P$, defined as
	\begin{equation*}
		q_{\mathbf{P}}\left( t\right) \triangleq\inf \left\{ x\in \reals:t\leq F_{%
			\mathbf{P}}\left( x\right) \right\} ,
	\end{equation*}%
	where $F_\P$ is the CDF of $\mathbf{P.}$
\end{lemma}

\begin{lemma}Suppose $\mathbf{P}_{1}$ and $\mathbf{P}_{2}$ are supported on $%
	\mathbb{D}$ and satisfy Assumption \ref{assump:classical}.3. We further assume $\mathbf{P}_{2}$ satisfies Assumption \ref{assump:pos_dens}.2. When $\mathrm{TV}(\mathbf{P}_{1},\mathbf{%
		P}_{2})<\underline{b}/2,$ we have
	\[
	\left\vert \sup_{\alpha \geq 0}\left\{ -\alpha \log \mathbf{E}_{\mathbf{P}%
		_{1}}\left[ \exp \left( -Y/\alpha \right) \right] -\alpha \delta \right\}
	-\sup_{\alpha \geq 0}\left\{ -\alpha \log \mathbf{E}_{\mathbf{P}_{2}}\left[
	\exp \left( -Y/\alpha \right) \right] -\alpha \delta \right\} \right\vert
	\leq \frac{2M}{\underline{b}}\mathrm{TV}(\mathbf{P}_{1},\mathbf{P}_{2}),
	\]
	where $\mathrm{TV}$ denotes the total variation distance.
	\label{lma:discrete_key}
\end{lemma}

The detailed proof is in   \ref{sec:proof_uniform}. We see those
bounds in \eqref{eqn:uniform_bd} and \eqref{eqn:uniform_bd_discrete} for the
distributionally robust regret does not depend on the uncertainty size $%
\delta $. {Furthermore, if $\sup_{n}\kappa ^{(n)}<\infty $ including the finite policy classes, linear classes, decision-tree policy classes and the case where $VC(\Pi)$ or $Graph(\Pi)$ is finite, we have a
	parametric convergence rate $O_{p}(1/\sqrt{n})$. Further, if $\kappa ^{(n)}=o_{p}(%
	\sqrt{n})$, we have $R_{\mathrm{DRO}}(\hat{\pi}_{\mathrm{DRO}})\rightarrow 0$
	in probability. Generally, we may expect the complexities of parametric classes are $O(1)$. %  and the complexities of  non-parametric classes are not.
	Theorem \ref{DRO_Uniform} guarantees the
	robustness of the policy learned from the training environment given
	sufficient training data and low complexity of the policy class. This result means that the test environment performance is guaranteed as long as test and training environments do not differ too much. We will show this 
	rate is optimal up to a constant in Theorem \ref{thm:ld}. \label%
	{sec:uniform_convergence}
	
	\subsection{Statistical Lower Bound}
	\label{sec:lower_bd}
	In this subsection, we provide a tight lower bound of the distributionally
	robust batch contextual bandit problem. First we define $\mathcal{P}(M)$ as the collection of all joint distributions of $(X,Y(a^{1}),Y(a^{2}),%
	\ldots ,Y(a^{d}),A)$ satisfying Assumption \ref{assump:classical}. To emphasize the dependence on the underlying distribution $\P_0$, we rewrite $R_\mathrm{DRO}(\pi) =R_\mathrm{DRO}(\pi,\P_0)$. We further denote  $\mathbf{P}^{\pi_0}_0$ to be the distribution of the observed triples $\left\{ X,A,Y(A)\right\}$.
	
	\begin{theorem}
		\label{thm:ld}
		Let $d=2$ and $\delta \leq 0.226$. Then, for any policy $\pi $ as a function of $\{X_{i},A_{i},Y_{i}%
		\}_{i=1}^{n},$ it holds that
		\begin{equation*}
			\sup_{\P_0 * \pi_0\in \mathcal{P}(M)}\E_{\left(\mathbf{P}^{\pi_0}_0\right)^n}\left[R_{\mathrm{DRO}}(\pi,\P_0 )\right]\geq \frac{M \kappa
				^{(n)}\left(\Pi \right) }{160\sqrt{n}}, \text{ for }n\geq \kappa
			^{(n)}\left(\Pi \right)^2,
		\end{equation*}
		where $\left(\mathbf{P}^{\pi_0}_0\right)^n$ denotes the $n$-times product measure of $\mathbf{P}^{\pi_0}_0$.
	\end{theorem}
	The proof of Theorem \ref{thm:ld} is in   \ref{sec:proof_ld}. Theorem \ref{thm:ld} shows that  the dependence of the regret on the complexity $\kappa ^{(n)} (\Pi)$; the number of samples, $n$; and the bound of the reward, $M$, is optimal up to a constant. It means that it is impossible to find a good robust policy with a small amount of training data or a relatively large policy class.}
\section{Simulation Studies}

\label{sec:numericals} In this section, we provide discussions on simulation
studies to justify the robustness of the proposed DRO policy $\hat{\pi}_{%
	\mathrm{DRO}}$ in the linear policy class. Specifically, Section \ref%
{sec:numerical_bayes} discusses a notion of the Bayes DRO policy, which is
viewed as a benchmark; Section \ref{sec:numerical_linear_class} presents an
approximation algorithm to efficiently learn a linear policy; Section \ref%
{sec:numerical_full} gives a visualization of the learned DRO policy, with
a comparison to the benchmark Bayes DRO policy, and demonstrates  the performance of our proposed
estimator.

\subsection{Bayes DRO Policy}

\label{sec:numerical_bayes} In this section, we give a characterization of
the Bayes DRO policy $\overline{\pi}^*_{\mathrm{DRO}} $, which maximizes the
distributionally robust value function within the class of all measurable
policies, i.e.,
\begin{equation*}
	\overline{\pi}^*_{\mathrm{DRO}}\in \argmax_{\pi \in \overline{\Pi}}\{Q_{%
		\mathrm{DRO}}(\pi)\},
\end{equation*}
where $\overline{\Pi}$ denotes the class of all measurable mappings from $%
\mathcal{X}$ to the action set $\mathcal{A}$. Despite the Bayes DRO
policy is not being learnable given finitely many training samples, it could be a
benchmark in a simulation study. Proposition \ref{thm:pi_policy} shows how
to compute $\overline{\pi}^*_{\mathrm{DRO}}$ if we know the population
distribution.
\begin{proposition} \label{thm:pi_policy}Suppose that for any $\alpha>0$ and any $a \in \mathcal{A}$, the mapping $x\mapsto \E_{\P_0}\left[\left. \exp\left(-Y(a)/\alpha\right)\right| X=x \right]$ is measurable. Then, the  Bayes DRO policy is
	\[
	\overline{\pi}^*_\mathrm{\rm DRO}(x) \in  \argmin_{a\in \mathcal{A}}\left\{ \E_{\P_0}\left[\left. \exp\left(-\frac{Y(a)}{\alpha^*(\overline{\pi}^*_{\rm DRO})}\right)\right| X=x \right]\right\},
	\]
	where $\alpha^*(\pi^*_{\rm DRO})$ is an optimizer of the following optimization problem:
	\begin{equation}
		\alpha^*(\overline{\pi}^*_{\rm DRO})\in \argmax_{\alpha \geq 0} \left \{-\alpha \log\E_{\P_0}\left [\min_{a\in \mathcal{A}}\left\{ \E_{\P_0}\left[\left. \exp\left(-Y(a)/\alpha\right)\right| X\right] \right\}\right] - \alpha \delta \right \}.
		\label{eq:alpha_DRO_small}
	\end{equation}
\end{proposition}
See   \ref{sec:proof_numerical_small} for the proof.
\begin{remark} $\overline{\pi}_{\rm DRO}^*$ only depends on the marginal distribution of $X$ and the conditional distributions of $Y(a^i)|X,i=1,2,\ldots,d$. Therefore, the conditional correlation structure of $Y(a^i)|X,i=1,2,\ldots,d$ does not affect $\overline{\pi}_{\rm DRO}^*$.
\end{remark}

\subsection{Linear Policy Class and Logistic Policy Approximation}

\label{sec:numerical_linear_class}

In this section, we introduce the linear policy class $\Pi_{\mathrm{Lin}}$.
We consider $\mathcal{X}$ to be a subset of $\mathbf{R}^p$, and the action
set $\mathcal{A}=\{1,2,\ldots,d\}$. To capture the intercept, it is
convenient to include the constant variable 1 in $X\in\mathcal{X}$, thus in
the rest of Section \ref{sec:numerical_linear_class}, $X$ is a $p+1$
dimensional vector and $\mathcal{X}$ is a subset of $\mathbf{R}^{p+1}$.
Each policy $\pi\in \Pi_{\mathrm{Lin}}$ is parameterized by a set of $d$
vectors $\Theta = \{\theta_a\in\mathbf{R}^{p+1}:a\in\mathcal{A}\} \in
\mathbf{R}^{(p+1)\times d}$, and the mapping $\pi: \mathcal{X}\rightarrow
\mathcal{A}$ is defined as
\begin{equation*}
	\pi_{\Theta}(x) \in \argmax_{a\in \mathcal{A}} ~ \left\{ \theta_a^{\top} x
	\right\}.
\end{equation*}
The optimal parameter for linear policy class is characterized by the
optimal solution of 
$$\max_{\Theta\in \mathbf{R}^{(p+1)\times d}} \mathbf{E}_{%
	\mathbf{P}_0}[ Y(\pi_{\Theta}(X))]. $$ Due to the fact that $\mathbf{E}_{%
	\mathbf{P}_0}[ Y(\pi_{\Theta}(X))] = \mathbf{E}_{\P *\pi_0}\left[\frac{Y(A)%
	\mathbf{1}\{\pi_{\Theta}(X) = A\}}{\pi_0(A\mid X)}\right]$, the associated
sample average approximation problem for optimal parameter estimation is
\begin{equation*}
	\max_{\Theta\in \mathbf{R}^{(p+1)\times d}} \frac{1}{n}\sum_{i=1}^{n}\frac{%
		Y_i(A_i)\mathbf{1\{}\pi_{\Theta}(X_{i})=A_{i}\mathbf{\}}}{\pi_{0}\left(
		A_{i}|X_{i}\right) }.
\end{equation*}
However, the objective in this optimization problem is non-differentiable
and non-convex, thus we approximate the indicator function using a softmax
mapping by
\begin{equation*}
	\mathbf{1\{}\pi_{\Theta}(X_{i})=A_{i}\mathbf{\}} \approx \frac{
		\exp(\theta_{A_i}^{\top}X_i)}{ \sum_{a = 1}^{d}\exp(\theta_{a}^{\top}X_i)},
\end{equation*}
which leads to an optimization problem with smooth objective:
\begin{equation*}
	\max_{\Theta\in \mathbf{R}^{(p+1)\times d}} \frac{1}{n}\sum_{i=1}^{n}\frac{%
		Y_i(A_i) \exp(\theta_{A_i}^{\top}X_i)}{\pi_{0}\left( A_{i}|X_{i}\right)
		\sum_{a = 1}^{d}\exp(\theta_{a}^{\top}X_i)}.
\end{equation*}
We employ the gradient descent method to solve for the optimal parameter
\begin{equation*}
	\hat{\Theta}_{\mathrm{Lin}} \in \argmax_{\Theta\in \mathbf{R}^{(p+1)\times
			d}} \left\{ \frac{1}{n}\sum_{i=1}^{n}\frac{Y_i(A_i)
		\exp(\theta_{A_i}^{\top}X_i)}{\pi_{0}\left( A_{i}|X_{i}\right) \sum_{a =
			1}^{d}\exp(\theta_{a}^{\top}X_i)}\right\},
\end{equation*}
and define the policy $\hat{\pi}_{\mathrm{Lin}}\triangleq\pi_{\hat{\Theta}_{%
		\mathrm{Lin}}}$ as our linear policy estimator. In Section \ref{sec:numerical_full}, we justify the
efficacy of $\hat{\pi}_{\mathrm{Lin}}$ by empirically showing $\hat{\pi}_{%
	\mathrm{Lin}}$ is capable of discovering the (non-robust) optimal decision
boundary.

As an oracle in Algorithm \ref{alg:DRO_policy_learning}, a similar smoothing
technique is adopted to solve $\argmin_{\pi\in \Pi_{\mathrm{Lin}}}\hat{W%
}_n(\pi,\alpha)$ for linear policy class $\Pi_{\mathrm{Lin}}$. We omit the
details here due to space limitations.

We will present an upper bound of the entropy integral $\kappa^{(n)}(\Pi_{%
	\mathrm{Lin}})$ in Lemma \ref{lemma:linear-class-complexity}. By plugging
the result of Lemma \ref{lemma:linear-class-complexity} into Theorem \ref%
{DRO_Uniform}, one can quickly remark that the regret $R_{\mathrm{DRO}}(\hat{%
	\pi}_{\mathrm{DRO}})$ achieves the optimal asymptotic convergence rate $%
O_p(1/\sqrt{n})$ given by Theorem \ref{DRO_Uniform}.

\begin{lemma}
	\label{lemma:linear-class-complexity} There exists a universal constant $C$ such that $ \kappa^{(n)}(\Pilin) \leq C\sqrt{dp \log(d)\log(dp)}.$
\end{lemma}

The proof of Lemma \ref{lemma:linear-class-complexity} is achieved by upper
bounding $\epsilon $-Hamming covering number $N_{H}^{\left( n\right)
}\left(\epsilon ,\Pi_{\mathrm{Lin}} \right)$ in terms of  the \emph{%
	graph dimension} in Lemma \ref{lemma:graph}, then by deploying an upper bound of graph dimension for the linear policy class
provided in \cite{daniely2012multiclass}.
\subsection{Experiment Results}
\label{sec:numerical_full}
In this section, we present two simple examples
with an explicitly computable optimal linear DRO policy. We illustrate the behavior of distributionally robust policy learning  in Section \ref{sec:numerical_small} and we demonstrate the effectiveness of the distributionally robust policy in Section \ref{sec:numerical_nonlinear}.
% order to justify
%the effectiveness of linear policy learning introduced in Section \ref%
%{sec:numerical_linear_class} and distributionally robust policy learning in
% Section \ref{sec:algorithm}.
\subsubsection{A Linear Boundary Example}

\label{sec:numerical_small}

We consider $\mathcal{X} =\{x = (x(1),\ldots,x(p))\in \mathbf{R}^p
:\sum_{i=1}^p x(i)^2 \leq1 \}$ to be a $p$-dimensional closed unit ball, and
the action set $\mathcal{A}=\{1,\ldots,d\}$. We assume that $Y(i)$'s are
mutually independent conditional on $X$ with conditional distribution
\begin{equation*}
	Y(i)|X \sim \mathcal{N}(\beta_i^\top X,\sigma_i^2),\text{ for }i=1,\ldots,d.
\end{equation*}
for vectors $\{\beta_1,\ldots,\beta_d\} \subset \mathbf{R}^p$ and $%
\{\sigma_1^2,\ldots,\sigma_d^2\} \subset \mathbf{R}_{+}$. In this case, by
directly computing the moment generating functions and applying Proposition %
\ref{thm:pi_policy}, we have
\begin{equation*}
	\overline{\pi}^*_\mathrm{\mathrm{DRO}}(x) \in \argmax_{i\in \{
		1,\ldots,d\}}\left\{\beta_i^\top x-\frac{\sigma^2_i}{2\alpha^*(\pi^*_{%
			\mathrm{DRO}})}\right\}.
\end{equation*}
We consider the linear policy class $\Pi_{\mathrm{Lin}}$. Apparently, the
DRO Bayes policy $\overline{\pi}^*_\mathrm{\mathrm{DRO}}(x)$ is in the class
$\Pi_{\mathrm{Lin}}$, thus it is also the optimal linear DRO policy, i.e., $%
\overline{\pi}_{\mathrm{DRO}}^* \in \argmax_{\pi\in\Pi_{\mathrm{Lin}}}
Q_{DRO}(\pi)$. Consequently, we can check the efficacy of the
distributionally robust policy learning algorithm by comparing $\hat{\pi}_{%
	\mathrm{DRO}}$ against $\overline{\pi}_{\mathrm{DRO}}^*$.

Now we describe the parameter in the experiment. We choose $p = 5$ and $d =
3 $. To facilitate visualization of the decision boundary, we set all the
entries of $\beta_i$ to be $0$ except for the first two dimensions. Specifically,
we choose
\begin{align*}
	\beta_{1} = (1,0,0,0,0),& & \beta_{2} = (-1/2,\sqrt{3}/2,0,0,0),& &
	\beta_{3} = (-1/2,-\sqrt{3}/2,0,0,0).
\end{align*}
and $\sigma_{1} = 0.2, \sigma_{2} = 0.5, \sigma_{3} = 0.8.$ We define the
Bayes policy $\overline{\pi}^{\ast}$ as the policy that maximizes $\mathbf{E}%
_{\mathbf{P}_0}[Y(\pi(X))]$ within the class of all measurable policies.
Under this setting, $\overline{\pi}^{\ast}(x)\in \argmax_{i =
	1,2,3}\{\beta_i^{\top}x\} $. The feature space $\mathcal{X}$ is
partitioned into three regions based on $\overline{\pi}^{\ast}$: for $i =
1,2,3$, we say $x\in\mathcal{X}$ belongs to Region $i$ if $\overline{\pi}%
^{\ast}(x) = i$. Given $X$, the action $A$ is drawn according to the
underlying data collection policy $\pi_0$, which is described in Table \ref%
{tab:policy-prob}.
\begin{table}[!ht]
	\centering
	\begin{tabular}{|c|c|c|c|}
		\hline
		& Region 1 & Region 2 & Region 3 \\ \hline
		Action 1 & 0.50 & 0.25 & 0.25 \\ \hline
		Action 2 & 0.25 & 0.50 & 0.25 \\ \hline
		Action 3 & 0.25 & 0.25 & 0.50 \\ \hline
	\end{tabular}%
	\caption{The probabilities of selecting an action based on $\protect\pi_0$
		in the linear example.}
	\label{tab:policy-prob}
\end{table}

We generate $\{X_i, A_i, Y_i\}_{i = 1}^n$ according to the procedure
described above as training dataset, from which we learn the non-robust
linear policy $\hat{\pi}_{\mathrm{Lin}}$ and the distributionally robust linear
policy $\hat{\pi}_{\mathrm{DRO}}$. Figure \ref{fig:simple-data} presents the
decision boundary of four different policies: (a) $\overline{\pi}^{\ast}$;
(b) $\hat{\pi}_{\mathrm{Lin}}$; (c) $\overline{\pi}^{\ast}_{\mathrm{DRO}}$;
(d) $\hat{\pi}_{\mathrm{DRO}}$, where $n = 5000$ and $\delta = 0.2$. One can
quickly remark that the decision boundary of $\hat{\pi}_{\mathrm{Lin}}$
resembles $\overline{\pi}^{\ast}$; and the decision boundary of $\hat{\pi}_{%
	\mathrm{DRO}}$ resembles $\overline{\pi}^{\ast}_{\mathrm{DRO}}$, which
demonstrates that $\hat{\pi}_{\mathrm{Lin}}$ is the (nearly) optimal non-DRO
policy and $\hat{\pi}_{\mathrm{DRO}}$ is the (nearly) optimal DRO policy.

This distinction between $\overline{\pi}^{\ast}$ and $\overline{\pi}^{\ast}_{%
	\mathrm{DRO}}$ is also apparent in Figure \ref{fig:simple-data}: $\overline{%
	\pi}^{\ast}_{\mathrm{DRO}}$ is less likely to choose Action 3, but more
likely to choose Action 1. In other words, a distributionally robust policy prefers
action with smaller variance.
We remark that this finding is
consistent with  \cite{duchi2019variance} and \cite{duchi2016statistics}
as they find the DRO problem with KL-divergence is a good approximation to
the variance-regularized quantity when $\delta\rightarrow 0$.

\begin{figure}[t!]
	\begin{center}
		\includegraphics[width = \textwidth]{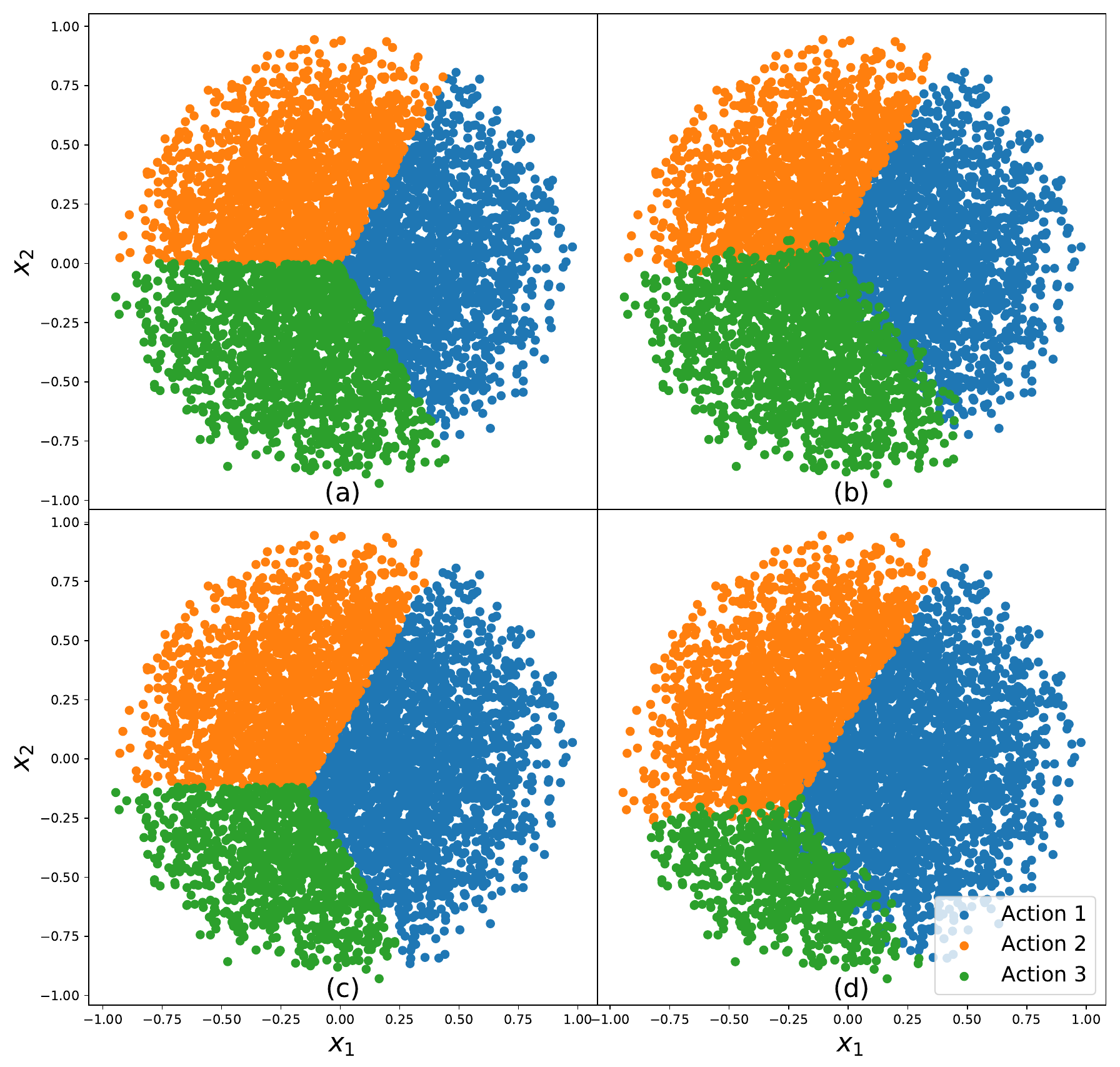}
	\end{center}
	\caption{Comparison of decision boundaries for different policies in the
		linear example: (a) Bayes policy $\overline{\protect\pi}^{\ast}$; (b) linear
		policy $\hat{\protect\pi}_{\mathrm{Lin}}$; (c) Bayes distributionally robust
		policy $\overline{\protect\pi}_{\mathrm{DRO}}^*$; (d) distributionally
		robust linear policy $\hat{\protect\pi}_{\mathrm{DRO}}$. We visualize the
		actions selected by different policies against the value of $(X(1),X(2))$.
		Training set size $n = 5000$; size of distributional uncertainty set $%
		\protect\delta = 0.2$.}

	%We denote $\hat{\pi}_{\mathrm{DRO}}^{0.2}$ to be the DRO policy with $\delta = 0.2$.
	\label{fig:simple-data}
\end{figure}

\subsubsection{A Non-linear Boundary Example}

\label{sec:numerical_nonlinear}

In this section, we compare the performance of different estimators in a
simulation environment where the Bayes decision boundaries are nonlinear.

We consider $\mathcal{X} = [-1,1]^5$ to be a $5$-dimensional cube, and the
action set to be $\mathcal{A}=\{1,2,3\}$. We assume that $Y(i)$'s are mutually
independent conditional on $X$ with conditional distribution
\begin{equation*}
	Y(i)|X \sim \mathcal{N}(\mu_i(X), \sigma_i^2),\text{ for }i=1,2,3.
\end{equation*}
where $\mu_i: \mathcal{X}\rightarrow \mathcal{A}$ is a measurable function
and $\sigma_i \in \mathbf{R}_{+}$ for $i = 1,2,3$. In this setting, we are
still able to analytically compute the Bayes policy $\overline{\pi}%
^{\ast}(x) \in \argmax_{i = 1,2,3} \{\mu_i(x)\} $ and the DRO Bayes $%
\overline{\pi}^{\ast}_{\mathrm{DRO}}(x) \in \argmax_{i = 1,2,3}
\left\{\mu_i(x)-\frac{\sigma^2_i}{2\alpha^*(\pi^*_{\mathrm{DRO}})}\right\}$.

In this section, the conditional mean $\mu_i(x)$ and conditional variance $%
\sigma_i$ are chosen as
\begin{align*}
	\begin{array}{lr}
		\mu_1(x) = 0.2 x(1), & \hspace{1.5in} \sigma_1 = 0.8, \\
		\mu_2(x) = 1 - \sqrt{ (x(1)+0.5)^2+(x(2) - 1)^2 }, & \hspace{1.5in} \sigma_2
		= 0.2, \\
		\mu_3(x) = 1 - \sqrt{ (x(1)+0.5)^2+(x(2) + 1)^2 }), & \hspace{1.5in}\sigma_3
		= 0.4. \\
		&
	\end{array}%
\end{align*}
Given $X$, the action $A$ is drawn according to the underlying data
collection policy $\pi_0$ described in Table \ref{tab:policy-prob-linear}.
\begin{table}[!ht]
	\centering
	\begin{tabular}{|c|c|c|c|}
		\hline
		& Region 1 & Region 2 & Region 3 \\ \hline
		Action 1 & 0.50 & 0.25 & 0.25 \\ \hline
		Action 2 & 0.30 & 0.40 & 0.30 \\ \hline
		Action 3 & 0.30 & 0.30 & 0.40 \\ \hline
	\end{tabular}%
	\caption{The probabilities of selecting an action based on $\protect\pi_0$
		in nonlinear example.}
	\label{tab:policy-prob-linear}
\end{table}

Now we generate the training set $\{X_i, A_i, Y_i\}_{i = 1}^n$ and learn the
non-robust linear policy $\hat{\pi}_{\mathrm{Lin}}$ and distributionally
robust linear policy $\hat{\pi}_{\mathrm{DRO}}$ in linear policy class $\Pi_{%
	\mathrm{Lin}}$, for $n = 5000$ and $\delta = 0.2$. Figure \ref%
{fig:complex-data} presents the decision boundary of four different
policies: (a) $\overline{\pi}^{\ast}$; (b) $\hat{\pi}_{\mathrm{Lin}}$; (c) $%
\overline{\pi}^{\ast}_{\mathrm{DRO}}$; (d) $\hat{\pi}_{\mathrm{DRO}}$. As $%
\overline{\pi}^{\ast}$ and $\overline{\pi}^{\ast}_{\mathrm{DRO}}$ have
nonlinear decision boundaries, any linear policy is incapable of accurate
recovery of Bayes policy. However, we quickly notice that the boundary
produced by $\hat{\pi}_{\mathrm{Lin}}$ and $\hat{\pi}_{\mathrm{DRO}}$ are
reasonable linear approximation of $\overline{\pi}^{\ast}$ and $\overline{\pi%
}^{\ast}_{\mathrm{DRO}}$, respectively. Especially noteworthy is the
robust policy prefers action with small variance (Action 2), which is
consistent with our finding in Section \ref{sec:numerical_linear_class}.

\begin{figure}[!ht]
	\begin{center}
		\includegraphics[width = \textwidth]{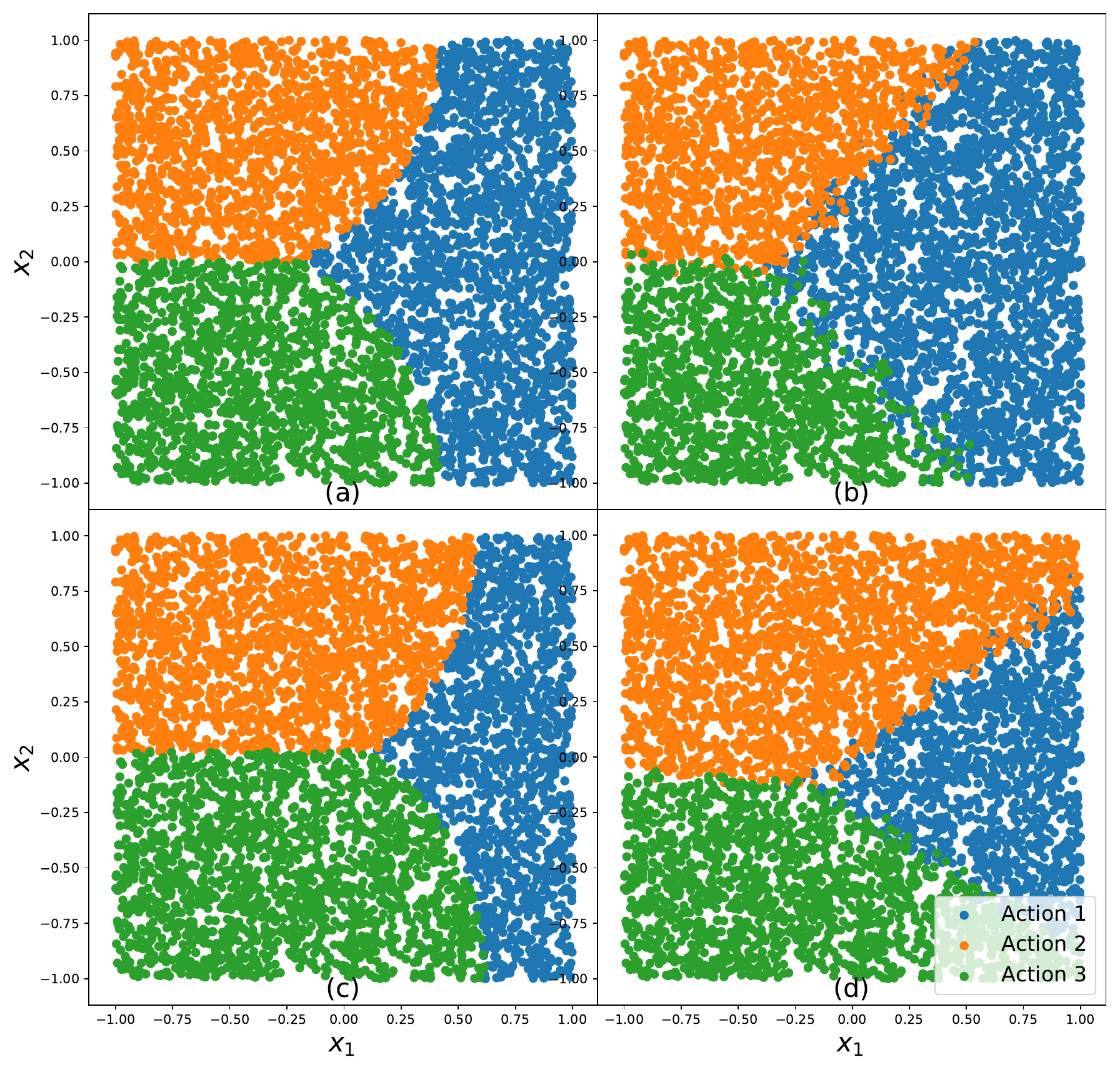}
	\end{center}
	\caption{Comparison of decision boundaries for different policies in
		nonlinear example: (a) optimal policy under population distribution $\mathbf{%
			P}_0$; (b) optimal linear policy $\hat{\protect\pi}_{\mathrm{Lin}}$ learned
		from data; (c) Bayes distributionally robust policy $\overline{\protect\pi}_{%
			\mathrm{DRO}}^*$; (d) distributionally robust linear policy $\hat{\protect\pi%
		}_{\mathrm{DRO}}$. We visualize the actions selected by different policies
		against the value of $(X(1),X(2))$. Training size is 5000; size of
		distributional uncertainty set $\protect\delta = 0.2$.}
	\label{fig:complex-data}
\end{figure}

Now we introduce two evaluation metrics in order to quantitatively
characterize the adversarial performance for different policies.

\begin{enumerate}
	\item We generate a test set with $n^{\prime }=2500$ i.i.d. data points
	sampled from $\P _0$ and evaluate the worst case performance of each policy
	using $\hat{Q}_{\mathrm{DRO}}$ with a radius $\delta^\mathrm{test}$. Note
	that $\delta^\mathrm{test}$ may be different from $\delta$ in the training
	procedure. The results are reported in the first row of Tables \ref%
	{tab:policy_learning} and \ref{tab:new}.
	
	\item We first generate $M=100$ independent test sets, where each test set
	consists of $n^{\prime }=2500$ i.i.d. data points sampled from $\P _0$. We
	denote them by $\left\{\left\{\left(X^{(j)}_i, Y^{(j)}_i(a^1), \dots,
	Y^{(j)}_i(a^d)\right)\right\}_{i=1}^{n^{\prime }}\right\}_{j=1}^{M}$. Then,
	we randomly sample a new dataset around each dataset, i.e., $\left(\tilde{X}%
	^{(j)}_i, \tilde{Y}^{(j)}_i(a^1), \dots, \tilde{Y}^{(j)}_i(a^d)\right)$ is
	sampled on the KL-sphere centered at $\left(X^{(j)}_i, Y^{(j)}_i(a^1),
	\dots, Y^{(j)}_i(a^d)\right)$ with a radius $\delta^{\mathrm{test}}$. Then,
	we evaluate each policy using $\hat{Q}_{\mathrm{min}}$, defined by
	\begin{equation*}
		\hat{Q}_{\mathrm{min}}(\pi) \triangleq\min_{1 \leq j \leq M}\left\{ \frac{1}{%
			n^{\prime }} \sum_{i = 1}^{n^{\prime }} \tilde{Y}^{(j)}_i\left(\pi\left({%
			\tilde{X}_i^{(j)}}\right)\right)\right\}.
	\end{equation*}
	The results are reported in the second row in Tables \ref%
	{tab:policy_learning} and \ref{tab:new}.
\end{enumerate}

We compare the robust performance of $\hat{\pi}_{\mathrm{Lin}}$ and $\hat{\pi}%
_{\mathrm{DRO}}$ and the POEM policy $\hat{\pi}_{\mathrm{POEM}}$ introduced
in \cite{swaminathan2015batch}. The regularization parameter of the POEM
estimator is chosen from $\{0.05,0.1,0.2,0.5,1\}$, and we find the results
are insensitive to the regularization parameter. We fix the uncertainty
radius $\delta=0.2$ used in the training procedure and size of test set $%
n^{\prime }= 2500$. In Table \ref{tab:policy_learning}, we let the training
set size range from $500$ to $2500$, and we fix $\delta^{\mathrm{test}%
}=\delta=0.2$, while in Table \ref{tab:new}, we fix the training set size to
be $n=2500$, and we let the magnitude of ``environment change" $\delta^{%
	\mathrm{test}}$ range from $0.02$ to $0.4$. We denote $\hat{\pi}_{\mathrm{DRO}}^{0.2}$ to be the DRO policy with $\delta = 0.2$. Tables \ref{tab:policy_learning}
and \ref{tab:new} report the mean and the standard error of the mean of $%
\hat{Q}_{\mathrm{DRO}}$ and $\hat{Q}_{\mathrm{min}}$ computed using $T=1000$
i.i.d. experiments, where an independent training set and an independent
test set are generated in each experiment. {Figure \ref{fig:difference_loss} visualizes the relative differences between   $\hat{\pi}_{\mathrm{Lin}} / \hat{\pi}_{\mathrm{POEM}}$ and $\hat{\pi}^{0.2}_{\mathrm{DRO}}$ in distributional shift environments.
}
%\begin{table}[!ht]
%\centering
%\begin{tabular}{|c|c|c|c|c|c|c|}
%    \hline
%    \multicolumn{2}{|c|}{} & $n = 500$ & $n = 1000$ & $n = 1500$ & $n = 2000$ & $n = 2500$ \\ \hline
%\multirow{3}{*}{$\hat{Q}_{\mathrm{DRO}}$} & $\hat{\pi}_{\mathrm{Lin}}$
%& $0.124\pm 0.042$ & $0.145\pm 0.018$ & $0.152\pm 0.012$ & $0.152\pm 0.012$ & $0.153\pm 0.011$ \\ \cline{2-7}
%& $\hat{\pi}_{\mathrm{DRO}}$
%& $0.128\pm 0.040$ & $0.148\pm 0.018$ & $0.154\pm 0.012$ & $0.154\pm 0.013$ & $0.154\pm 0.010$ \\ \cline{2-7}
%& $\hat{\pi}_{\mathrm{POEM}}$
%& $0.107\pm 0.038$ & $0.132\pm 0.022$ & $0.138\pm 0.021$ & $0.144\pm 0.016$ & $0.147\pm 0.014$ \\ \hline
%\multirow{3}{*}{$\hat{Q}_{\mathrm{min}}$} & $\hat{\pi}_{\mathrm{Lin}}$
%& $0.221\pm 0.038$ & $0.242\pm 0.018$ & $0.249\pm 0.014$ & $0.251\pm 0.015$ & $0.250\pm 0.016$ \\ \cline{2-7}
%& $\hat{\pi}_{\mathrm{DRO}}$
%& $0.227\pm 0.035$ & $0.245\pm 0.019$ & $0.250\pm 0.014$ & $0.251\pm 0.015$ & $0.251\pm 0.015$ \\ \cline{2-7}
%& $\hat{\pi}_{\mathrm{POEM}}$
%& $0.212\pm 0.035$ & $0.233\pm 0.020$ & $0.239\pm 0.019$ & $0.244\pm 0.016$ & $0.246\pm 0.013$ \\ \hline
%\end{tabular}
%\caption{Comparison of adversarial performance in the nonlinear example,$\delta =0.1$.}
%\label{tab:policy_learning_delta_01}
%\end{table}

\begin{table}[!ht]
	\centering
	\resizebox{\textwidth}{!}{
		\begin{tabular}{l|lccccc}
			\toprule
			\multicolumn{2}{c}{} & $n = 500$ & $n = 1000$ & $n = 1500$ & $n = 2000$ & $n = 2500$ \\ \midrule
			\multirow{3}{*}{$\hat{Q}_{\mathrm{DRO}}$} & $\hat{\pi}_{\mathrm{Lin}}$
			& $0.0852\pm 0.0013$ & $0.1031\pm 0.0008$ & $0.1093\pm 0.0005$ & $0.1120\pm 0.0005$ & $0.1135\pm 0.0004$ \\
			& $\hat{\pi}_{\mathrm{POEM}}$
			& $0.0621\pm 0.0014$ & $0.0858\pm 0.0009$ & $0.0972\pm 0.0007$ & $0.1013\pm 0.0006$ & $0.1057\pm 0.0005$ \\
			& $\hat{\pi}^{0.2}_{\mathrm{DRO}}$
			& $0.0998\pm 0.0011$ & $0.1120\pm 0.0007$ & $0.1152\pm 0.0005$ & $0.1166\pm 0.0004$ & $0.1170\pm 0.0004$ \\
			\midrule
			\multirow{3}{*}{$\hat{Q}_{\mathrm{min}}$} & $\hat{\pi}_{\mathrm{Lin}}$
			& $0.2183\pm 0.0011$ & $0.2347\pm 0.0007$ & $0.2398\pm 0.0005$ & $0.2426\pm 0.0005$ & $0.2437\pm 0.0005$ \\
			& $\hat{\pi}_{\mathrm{POEM}}$
			& $0.2030\pm 0.0011$ & $0.2230\pm 0.0007$ & $0.2311\pm 0.0006$ & $0.2344\pm 0.0006$ & $0.2378\pm 0.0005$ \\
			& $\hat{\pi}^{0.2}_{\mathrm{DRO}}$
			& $0.2249\pm 0.0009$ & $0.2384\pm 0.0006$ & $0.2428\pm 0.0005$ & $0.2439\pm 0.0005$ & $0.2460\pm 0.0005$ \\
			\bottomrule
		\end{tabular}
	}
	\caption{Comparison of robust performance for different training sizes $n$
		when $\protect\delta=\protect\delta^{\mathrm{test}} =0.2$.}
	\label{tab:policy_learning}
\end{table}
\begin{table}[tbp]
	\centering
	\resizebox{\textwidth}{!}{
		\begin{tabular}{l|lcccccc|}
			\toprule
			\multicolumn{2}{c}{}& $\delta^{\rm test} = 0.02$& $\delta^{\rm test} = 0.06$& $\delta^{\rm test} = 0.10$& $\delta^{\rm test} = 0.20$& $\delta^{\rm test} = 0.30$& $\delta^{\rm test} = 0.40$\\ \hline
			\multirow{3}{*}{$\hat{Q}_{\mathrm{DRO}}$} & $\hat{\pi}_{\mathrm{Lin}}$
			& $0.2141\pm 0.0003$ & $0.1783\pm 0.0003$ & $0.1546\pm 0.0004$ & $0.1132\pm 0.0004$ & $0.0840\pm 0.0005$ & $0.0601\pm 0.0005$ \\
			& $\hat{\pi}_{\mathrm{POEM}}$
			& $0.2097\pm 0.0003$ & $0.1734\pm 0.0004$ & $0.1497\pm 0.0004$ & $0.1087\pm 0.0005$ & $0.0787\pm 0.0005$ & $0.0543\pm 0.0005$ \\
			& $\hat{\pi}^{0.2}_{\mathrm{DRO}}$
			& $0.2164\pm 0.0003$ & $0.1805\pm 0.0003$ & $0.1574\pm 0.0003$ & $0.1170\pm 0.0004$ & $0.0882\pm 0.0004$ & $0.0646\pm 0.0005$ \\\midrule
			\multirow{3}{*}{$\hat{Q}_{\mathrm{min}}$} & $\hat{\pi}_{\mathrm{Lin}}$
			& $0.2602\pm 0.0005$ & $0.2545\pm 0.0005$ & $0.2516\pm 0.0005$ & $0.2443\pm 0.0005$ & $0.2378\pm 0.0005$ & $0.2305\pm 0.0005$ \\
			& $\hat{\pi}_{\mathrm{POEM}}$
			& $0.2556\pm 0.0005$ & $0.2511\pm 0.0005$ & $0.2472\pm 0.0005$ & $0.2400\pm 0.0005$ & $0.2334\pm 0.0005$ & $0.2263\pm 0.0005$ \\
			& $\hat{\pi}^{0.2}_{\mathrm{DRO}}$
			& $0.2613\pm 0.0004$ & $0.2563\pm 0.0004$ & $0.2532\pm 0.0004$ & $0.2461\pm 0.0005$ & $0.2397\pm 0.0005$ & $0.2329\pm 0.0005$ \\\bottomrule
		\end{tabular}
	}
	\caption{Comparison of robust performance for different test environments $%
		\protect\delta_{\mathrm{test}}$ when $\protect\delta= 0.2$ and $n=2500$.}
	\label{tab:new}
\end{table}

\begin{figure}[!ht]
	\centering
	\begin{subfigure}[b]{0.45\textwidth}
		\centering
		\includegraphics[width = \textwidth]{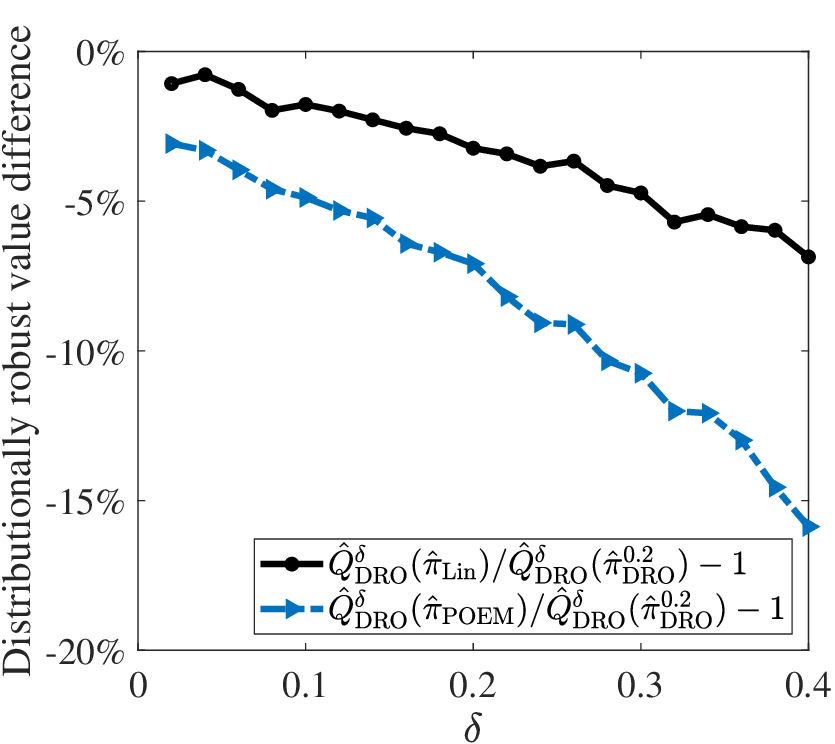}
		\caption{$\hat{Q}_{\mathrm{DRO}}$}
	\end{subfigure}
	\hfill
	\begin{subfigure}[b]{0.45\textwidth}
		\centering
		\includegraphics[width = \textwidth]{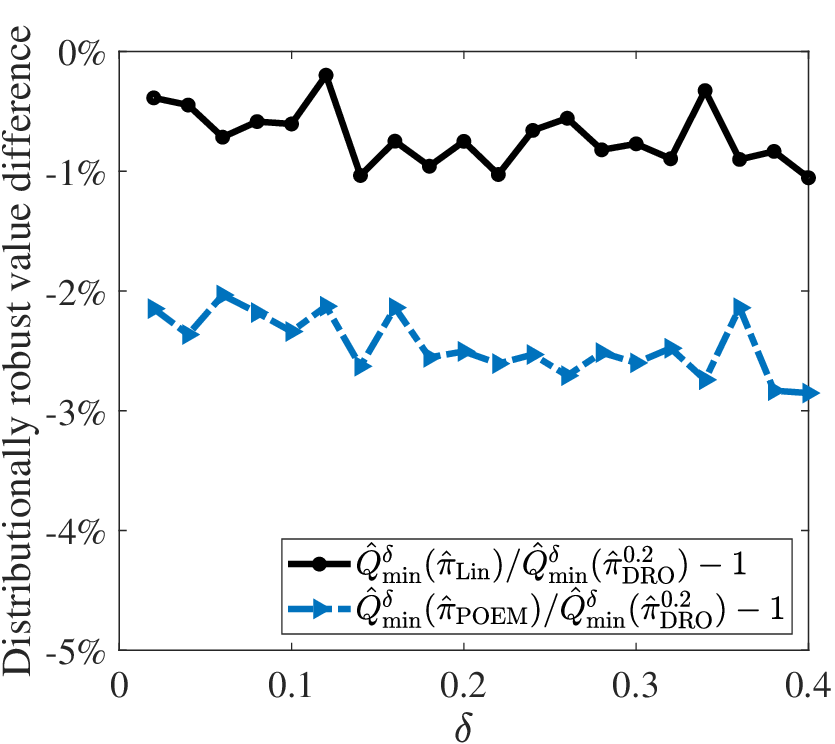}
		\caption{$\hat{Q}_{\mathrm{min}}$}
		\label{subfig:min}
	\end{subfigure}
	\caption{Difference of robust performance for different test environments $%
		\protect\delta_{\mathrm{test}}$ when $\protect\delta= 0.2$ and $n=2500$.}
	\label{fig:difference_loss}
\end{figure}

{We can easily observe from Table \ref{tab:policy_learning} that $\hat{\pi}_{\mathrm{DRO}}^{0.2}$ achieves the best robust
	performance among all three policies and the superiority is significant in the most of cases, which implies $%
	\hat{\pi}_{\mathrm{DRO}}$ is more resilient to adversarial perturbations.} We
also highlight that $\hat{\pi}_{\mathrm{DRO}}$ has smaller standard
deviation ($\sqrt{T}\times$standard error) in Table \ref{tab:policy_learning}
and the superiority of $\hat{\pi}_{\mathrm{DRO}}$ is more manifest under
smaller training set, indicating $\hat{\pi}_{\mathrm{DRO}}$ is a more stable
estimator compared with $\hat{\pi}_{\mathrm{Lin}}$ and $\hat{\pi}_{\mathrm{%
		POEM}}$. In Table \ref{tab:new} and Figure \ref{fig:difference_loss}, we find that $\hat{\pi}_{\mathrm{DRO}}$ significantly
outperforms $\hat{\pi}_{\mathrm{Lin}}$ and $\hat{\pi}_{\mathrm{POEM}}$ for a
wide range of $\delta^{\mathrm{test}}$ even if the model is
misspecified in the sense that $\delta^{\mathrm{test}}\neq\delta$, and the
results of small $\delta^{\mathrm{test}}$ indicate that our method may potentially alleviate overfitting. These results show that our
method is insensitive to the choice of the uncertainty radius $\delta$ in
the training procedure. %We clear see $\hat{\pi}_{\mathrm{Lin}} $ and $\hat{\pi}_{\mathrm{POEM}}$ incur losses compared to $\hat{\pi}^{0.2}_{\mathrm{DRO}}$ since they ignore the distributional shifts. Although in Figure \ref{subfig:min}, the losses seem small, we note that it is purely due to the parameter specifications and $1\%$ losses  usually mean millions of dollars in practice.
%but if the training $\delta$ is closed to the actual
%magnitude of ``environment change" such that $\delta=\delta^{\mathrm{test}}$%
%, the performance gain is relatively large.

\section{Real Data Experiments: Application on a Voting Dataset}

\label{sec:real_data} In this section, we compare the empirical performance
of different estimators on a voting dataset concerned with the August 2006
primary election.\footnote{Data available in https://github.com/gsbDBI/ExperimentData/tree/master/Social} This dataset was originally collected by \cite%
{gerber2008social} to study the effect of social pressure on electoral
participation rates. Later, the dataset was employed by \cite%
{zhou2018offline} to study the empirical performance of several offline
policy learning algorithms. In this section, we  apply different policy
learning algorithms to this dataset and illustrate some interesting findings.

\subsection{Dataset Description}

For completeness, we borrow the description of the dataset from
\cite{zhou2018offline} since we use  almost the same (despite different
reward) data preprocessing procedure. We only focus on aspects that are
relevant to our current policy learning context.

The dataset contains $180002$ data points (i.e. $n = 180002$), each
corresponding to a single voter in a different household. The voters span
the entire state of Michigan. There are ten voter characteristics in the
dataset: \emph{year of birth, sex, household size, city, g2000, g2002,
	g2004, g2000, p2002, and p2004}. The first four features are self-explanatory.
The next three features are outcomes for whether a voter  voted in the
general elections in 2000, 2002 and 2004 respectively: $1$ was recorded if
the voter did vote and $0$ was recorded if the voter did not
vote. The last three features are outcomes for whether a voter  voted in the
primary in 2000, 2002 and 2004. As \cite{gerber2008social}  pointed out,
these 10 features are commonly used as covariates for predicting whether an
individual voter will vote.

There are five actions in total, as listed below:

\textbf{Nothing:} No action is performed.

\textbf{Civic:} A letter with "Do your civic duty" is mailed to the
household before the primary election.

\textbf{Monitored:} A letter with "You are being studied" is mailed to the
household before the primary election. Voters receiving this letter are
informed that whether they vote or not in this election will be observed.

\textbf{Self History:} A letter with the voter's past voting records as well
as the voting records of other voters who live in the same household is
mailed to the household before the primary election. The letter also
indicates that, once the election is over, a follow-up letter on whether the
voter voted will be sent to the household.

\textbf{Neighbors:} A letter with the voting records of this voter, the
voters living in the same household, and the voters who are neighbors of this
household is mailed to the household before the primary election. The letter
also indicates that all your neighbors will be able to see
your past voting records and that follow-up letters will be
sent so that whether this voter voted in the upcoming election will
become public knowledge among the neighbors.

In collecting this dataset, these five actions are randomly chosen
independent of everything else, with probabilities equal to $\frac{10}{18},%
\frac{2}{18},\frac{2}{18},\frac{2}{18},\frac{2}{18}$ (in the same order as
listed above). The outcome is whether a voter has voted in the 2006 primary
election, which is either 1 or 0. {It is not hard to imagine that \textbf{Neighbors} is the best
	policy for the whole population as it adds the highest social pressure for people to vote.} Therefore, instead of directly using the voting
outcome as a reward,  we define $Y_i$, the reward associated to
voter $i$, as the voting outcome minus the social cost of deploying an action to
this voter, namely,
\begin{equation*}
	Y_i(a) = \mathbf{1}\{\mbox{voter $i$ votes under action $a$}\} - c_a, \quad
	\forall a\in\mathcal{A},
\end{equation*}
where $c_a$ is the vector of cost for deploying certain actions. Here, we
set $c_a = (0.3,0.32,0.34,0.36,0.38)$ to be close to the empirical average
of each action.
%such that $\frac{1}{n}\sum_{i=1}^{n}Y_i(a)\mathbf{1}\{\pi_0(X_i)=a\}\approx 0$ in the dataset. {specific numbers for $c_a$.}

\subsection{Decision Trees and Greedy Tree Search}

\label{sec:decision_tree} We  introduce the decision-tree policy
classes. We follow the convention in \cite{bertsimas2017optimal}. A
depth-$L$ tree has $L$ layers in total: branch nodes live in the first $L-1$
layers, while the leaf nodes live in the last layer. Each branch node is
specified by the variable to be split on and the threshold $b$. At a branch
node, each component of the $p$-dimensional feature vector $x$ can be chosen
as a split variable. The set of all depth-$L$ trees is denoted by $\Pi_L$.
Then, Lemma 4 in \cite{zhou2018offline} shows that
\begin{equation*}
	\kappa^{(n)}(\Pi_L)\leq \sqrt{(2^L - 1) \log p + 2^L \log d} + \frac{4}{3}
	L^{1/4} \sqrt{2^L-1}.
\end{equation*}
In the voting dataset experiment, we concentrate on the policy class $%
\Pi_{L} $. %\begin{example}[Decision trees]
%\label{example:decisiontree}
%
%\end{example}

The algorithm for decision tree learning needs to be computationally
efficient, since algorithm will be iteratively executed in Line 7 of
Algorithm \ref{alg:DRO_policy_learning} to compute $\argmin
_{\pi\in \Pi_{L}}\hat{W}_n(\pi,\alpha)$. Since finding an optimal
classification tree is generally intractable, see \cite{bertsimas2017optimal}%
, here we adopt an heuristic algorithm called \emph{greedy tree search}.
This procedure  can be inductively defined. First, to
learn a depth-2 tree, greedy tree search will brute force search all the
possible spliting choices of the branch node, and all the possible
actions of the leaf nodes. Suppose that the learning procedure for depth-$%
(L-1)$ tree has been defined. To learn a depth-$L$ tree, we first learn a
depth-2 tree with the optimal branching node, which partitions all the
training data into two disjointed groups associated with two leaf nodes. Then
each leaf node is replaced by the depth-$(L-1)$ tree trained using the data
in the associated group.

\subsection{Training and Evaluation Procedure}

Consider a hypothetical experiment of designing distributionally robust
policy. In the experiment, suppose that the training data is collected from
some cities, and our goal is to learn a robust policy to be deployed to the
other cities. Dividing the training and test population based on the city
voters live creates both a covariate shift and a concept drift between training
set and test set. For example, considering the covariate shift first, the
distribution of \emph{year of birth} is generally different across different
cities. As for the concept drift, it is conceivable that different groups of
population may have different response to the same action, depending on some
latent factors that are not reported in the dataset, such as occupation and
education. The distribution of such latent factors also varies among
different cities, which results in concept drift. Consequently, we use the
feature \emph{city} to divide the training set and the test set, in order to
test policy performance under ``environmental change''.

The voting data set contains 101 distinct cities. To comprehensively
evaluate the out-of-sample policy performance, we adapt 
leave-one-out cross-validation to generate 101 pairs of the training and
test set, each test set contains exactly one district city and the
corresponding training set is the complement set of the test set. On each
pair of the training  and  test set, we learn a non-robust depth-$3$ decision
tree policy $\hat{\pi}_{3}$ and distributionally robust decision tree
policies $\hat{\pi}_{\mathrm{DRO}}$ in $\Pi_{3}$ for $\delta\in%
\{0.1,0.2,0.3,0.4\}$, then on the test set the policies are evaluated using
the unbiased IPW estimator
\begin{equation*}
	\hat{Q}_{\mathrm{IPW}}(\pi) \triangleq \frac{1}{n}\sum_{i=1}^{n} \frac{%
		\mathbf{1}\{\pi(X_i) = A_i\}}{\pi_0(A_i\mid X_i)} Y_i(A_i).
\end{equation*}
Consequently, for each policy $\pi$ we get $101$ of $\hat{Q}_{\mathrm{IPW}%
}(\pi)$ scores on $101$ different test sets.

{
	\subsection{Selection of Distributional Shift Size $\protect\delta$}
	\label{sec:select_delta}
	
	The distributional shift size $\delta$ quantifies the level of robustness of
	the distributionally robust policy learning algorithm. The empirical
	performance of the algorithm substantially depends on the selection of $%
	\delta$. On one hand, if $\delta$ is too small, the robustification effect
	is negligible and the algorithm would learn an over-aggressive policy; on
	the other hand, if $\delta$ is overly large, the policy  is over-conservative, always choosing the action subject to the smallest reward
	variation. We remark that the selection of $\delta$ is more a
	managerial decision rather than a scientific procedure. It depends to the
	decision-makers' own risk-aversion level and their own perception of the new
	environments. In this section, we provide a guide to help select $\delta$ in
	this voting dataset.
	
	A natural approach to select $\delta$ is to empirically estimate the size of
	distributional shift using the training data. From the training set, we 
	partition the data in 20\% of cities as our validation set with distribution
	denoted by $\P ^{20}$, and we use $\P ^{80}$ to denote the distribution of
	the remaining 80\% of the training set. We  estimate the $D(\P ^{20}||\P %
	^{80})$, which reasonably quantifies the size of distributional shift across
	different cities. To this end, we decompose distributional shift into two
	parts,
	\begin{equation*}
		D(\P ^{20}||\P ^{80}) = \underbrace{D(\P _X^{20}||\P _X^{80})}_{%
			\mbox{marginal distribution of $X$}} + \underbrace{\mathbf{E}_{\P _X^{20}}[D(%
			\P _Y^{20}|X||\P _Y^{80}|X)]}_{%
			\mbox{conditional distribution of $Y$ given
				$X$}},
	\end{equation*}
	where $\P _X^{i}$ denote the $X$-marginal distribution of $\P ^{i}$, and $\P %
	_Y^{i}|X$ denote the conditional distribution of $Y$ given $X$ for $\P ^{i}$%
	, for $i = 20,80$. To estimate the size of marginal distributional shift $D(%
	\P _X^{20}||\P _X^{80})$, we first apply grouping to features such as
	year of birth, in order to avoid of infinite KL-divergence. Next we focus on
	the conditional distributional shift $D(\P _Y^{20}|X||\P _Y^{80}|X)$.
	Noticing that the value of $Y(a)$ is binary for each $a$, we fit two logistic regression models separately for $\P ^{20}$
	and $\P ^{80}$ to estimate the conditional distribution of $Y(a)$ given $X$%
	. We estimate $D(\P _Y^{20}|X||\P _Y^{80}|X)$ using the fitted logistic
	regression model, then take the expectation of $X$ over $\P _X^{20}$. We
	repeat the 80\%/20\% random splitting  100 times, and compute $D(\P %
	^{20}||\P ^{80})$ using this procedure. Additional experimental details are reported in   \ref{appendix:experiments_details}.  The empirical CDF of the
	estimated $\delta$ from those $100$ experiments is reported in Figure \ref%
	{fig:delta_cdf}. It is easy to see approximately 90\% percent of $\delta$s are less
	than $0.2$.
	
	\begin{figure}[!ht]
		\centering
		\begin{subfigure}[b]{0.45\textwidth}
			\centering
			\includegraphics[width = \textwidth]{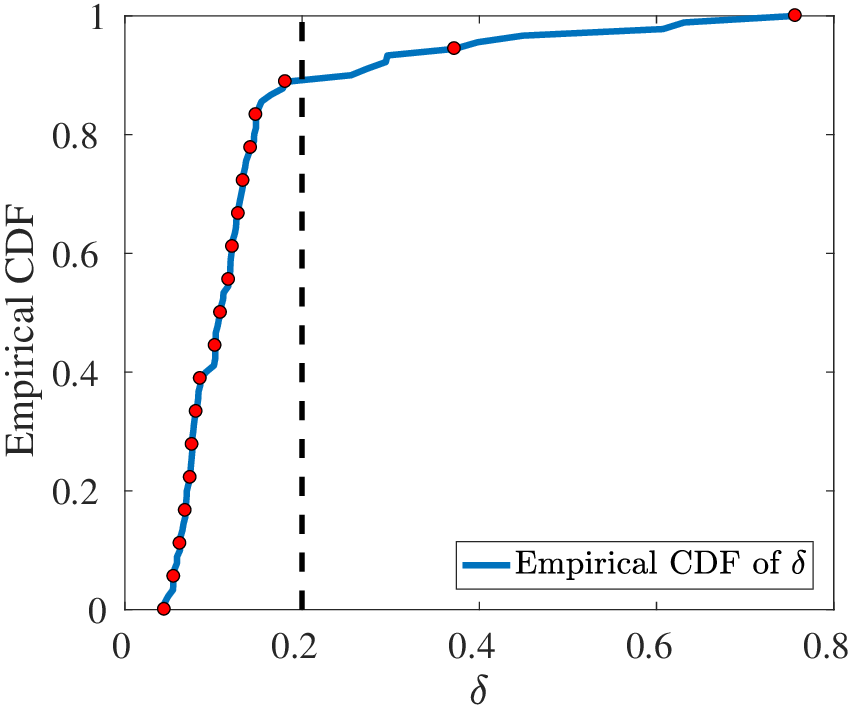}
			\caption{Empirical CDF for estimated $\delta$.}
			\label{fig:delta_cdf}
		\end{subfigure}
		\hfill
		\begin{subfigure}[b]{0.45\textwidth}
			\centering
			\includegraphics[width = \textwidth]{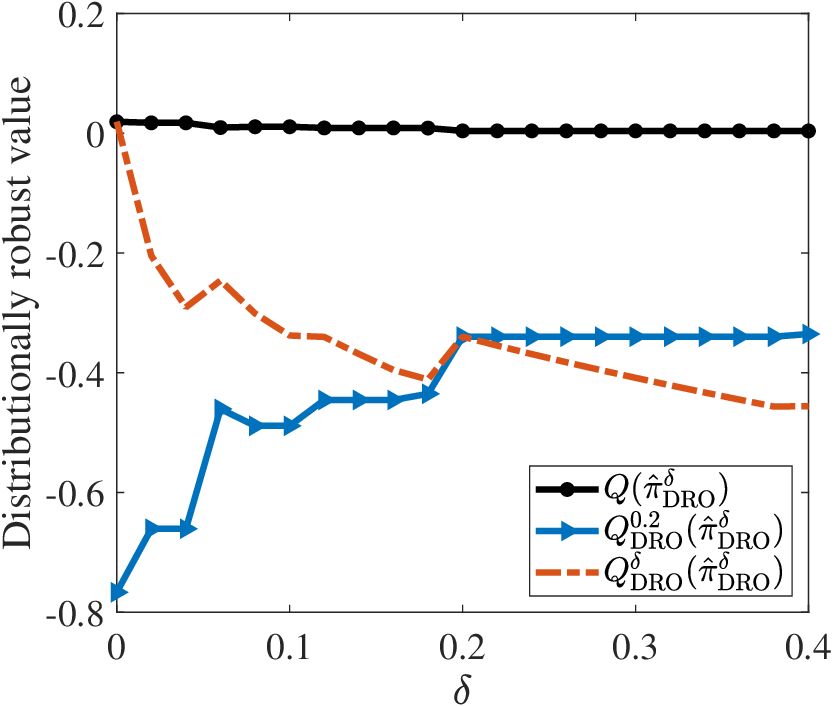}
			\caption{Sensitivity of policy on $\delta$.}
			\label{fig:delta_loss}
		\end{subfigure}
		\caption{Selection of distributional shift size $\protect\delta$.}
	\end{figure}
	
	Beside explicitly estimating the uncertainty size, we also check the
	sensitivity of our policy on $\delta$. We present the reward profile for our
	distributionally robust policy in Figure \ref{fig:delta_loss}. In the
	figure, we employ the x-axis to represent the $\delta$ used in the policy
	training process. The top black line is the non-robust pure value function,
	which appears to be almost invariant among policies with different robust level
	for $\delta \in [0,0.4]$. The blue line is the distributionally robust
	value function with $\delta$ fixed to $0.2$. We remark that the non-robust
	policy $\hat{\pi}_{\mathrm{DRO}}^{0}$ has a deficient performance in terms of
	robust value function, yet the robust value function improves as the robust
	level of the policy is increasing and becomes non-sensitive to the delta
	when $\delta$ larger than $0.2$. Finally, this red line is the
	distributionally robust value function with respect to the same $\delta$ in
	the training process. Thus, it is the actual training reward.
}
\subsection{Experimental Result and Interpretation}

We summarize some important statistics of $\hat{Q}_{\mathrm{IPW}}(\pi)$
scores in Table \ref{tab:voting}, including mean, standard deviation,
minimal value, 5th percentile, 10th percentile, and 20th percentile. All the
statistics are calculated based on the result of 101 test sets.
\begin{table}[!ht]
	\centering
	\begin{tabular}{l|ccccccc}
		\toprule
		\multicolumn{2}{l}{} & mean & std & min & 5th percentile & 10th percentile
		& 20th percentile \\ \bottomrule
		\multicolumn{2}{l}{$\hat{Q}_{\mathrm{IPW}}(\hat{\pi}_3)$} & 0.0386 & 0.0991
		& -0.2844 & -0.1104 & -0.0686 & -0.0358 \\ \midrule %\hline
		\multirow{4}{*}{$ \hat{Q}_{\mathrm{IPW}}(\hat{\pi}_{\mathrm{DRO}})$} & $%
		\delta = 0.1$ & 0.0458 & 0.0989 & -0.2321 & -0.1007 & -0.0489 & -0.0223 \\
		%\cline{2-8}
		& $\delta = 0.2$ & 0.0368 & 0.0895 & -0.2314 & -0.0785 & -0.0518 & -0.0217
		\\ %\cline{2-8}
		& $\delta = 0.3$ & 0.0397 & 0.0864 & -0.2313 & -0.0677 & -0.0407 & -0.0190
		\\ %\cline{2-8}
		& $\delta = 0.4$ & 0.0383 & 0.0863 & -0.2312 & -0.0677 & -0.0429 & -0.0202
		\\ \bottomrule
	\end{tabular}%
	\caption{Comparison of important statistics for voting dataset.}
	\label{tab:voting}
\end{table}

We remark that the mean value of $\hat{Q}_{%
	\mathrm{IPW}}(\hat{\pi}_{\mathrm{DRO}})$ is comparable to $\hat{%
	Q}_{\mathrm{IPW}}(\hat{\pi}_{3})$, and it is even better when an appropriate
value of $\delta$ (such as $\delta = 0.1$) is selected. One can also observe
that $\hat{Q}_{\mathrm{IPW}}(\hat{\pi}_{\mathrm{DRO}})$ has
a smaller standard deviation and a larger minimal value when comparing to $%
\hat{Q}_{\mathrm{IPW}}(\hat{\pi}_{3})$, and the difference becomes larger as
$\delta$ increases. The comparison of 5th, 10th, and 20th percentiles  also indicates that $\hat{\pi}_{\mathrm{DRO}}$ perform better 
than $\hat{\pi}_{3}$ in ``bad'' (or ``adversarial'') scenarios of
environmental change, which is exactly the desired behavior of $\hat{\pi}_%
\mathrm{DRO}$ by design.

To reinforce our observation in Table \ref{tab:voting}, we visualize and
compare the distribution of $\hat{Q}_{\mathrm{IPW}}(\hat{\pi}_{\mathrm{DRO}%
}) $ and $\hat{Q}_{\mathrm{IPW}}(\hat{\pi}_{3})$ in Figure \ref%
{fig:voting-exp-hist}, for (a) $\delta = 0.1$ and (b) $\delta = 0.4$. We
notice that the histogram of $\hat{Q}_{\mathrm{IPW}}(\hat{\pi%
}_{\mathrm{DRO}})$ is more concentrated than the histogram of $\hat{Q}_{%
	\mathrm{IPW}}(\hat{\pi}_{3})$, which supports our observation that $\hat{\pi}%
_{\mathrm{DRO}}$ is more robust.

\begin{figure}[!ht]
	\centering
	\begin{subfigure}[b]{0.45\textwidth}
		\centering
		\includegraphics[width = \textwidth]{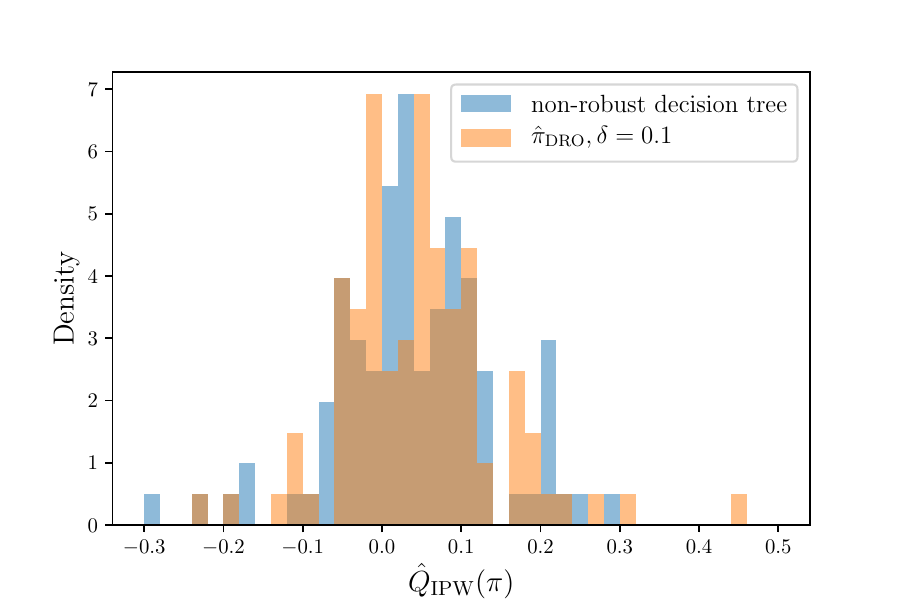}
		\caption{$\delta = 0.1$}
	\end{subfigure}
	\hfill
	\begin{subfigure}[b]{0.45\textwidth}
		\centering
		\includegraphics[width = \textwidth]{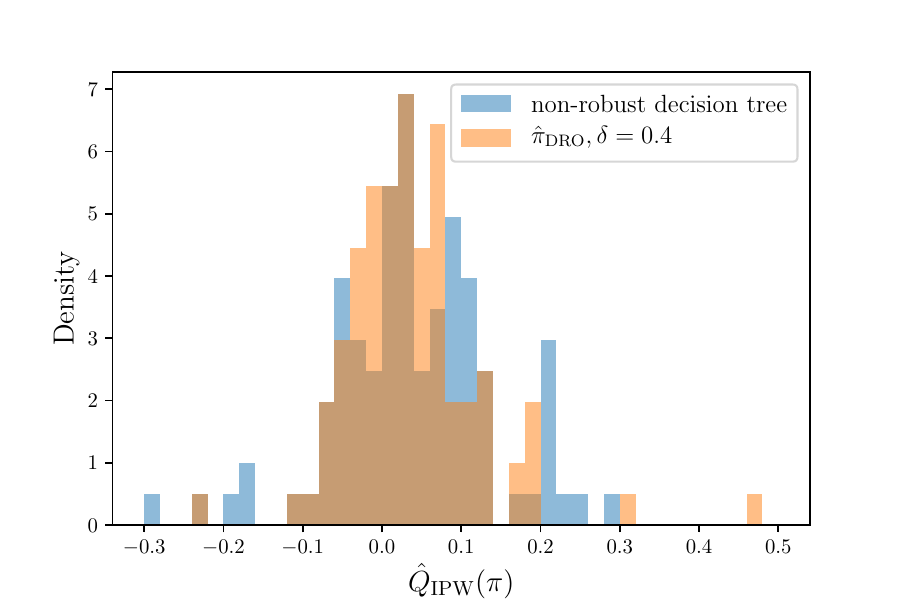}
		\caption{$\delta = 0.4$}
	\end{subfigure}
	\caption{Comparison of the distribution of $\hat{Q}_{\mathrm{IPW}}$ between
		distributionally robust decision tree against non-robust decision tree. (a) $%
		\protect\delta = 0.1$, (b) $\protect\delta = 0.4$.}
	\label{fig:voting-exp-hist}
\end{figure}

We present two instances of distributionally robust decision trees in Figure %
\ref{fig:voting-tree}: (a) is an instance of robust tree with $\delta =0.1$,
and (b) is an instance of robust tree with $\delta =0.4$. We remark that the
decision tree in (b) deploys the action \textbf{Nothing} to most of the
potential voters, because almost all the individuals in the dataset were born
after 1917 and have a household size fewer than 6. For a large value of $%
\delta $, the distributionally robust policy $\hat{\pi}_{\mathrm{DRO}}$
becomes almost degenerate, which only selects \textbf{Nothing}, the action
with a minimal reward variation.

\begin{figure}[!ht]
	\centering
	\begin{subfigure}[b]{0.45\textwidth}
		\centering
		\includegraphics[width = \textwidth]{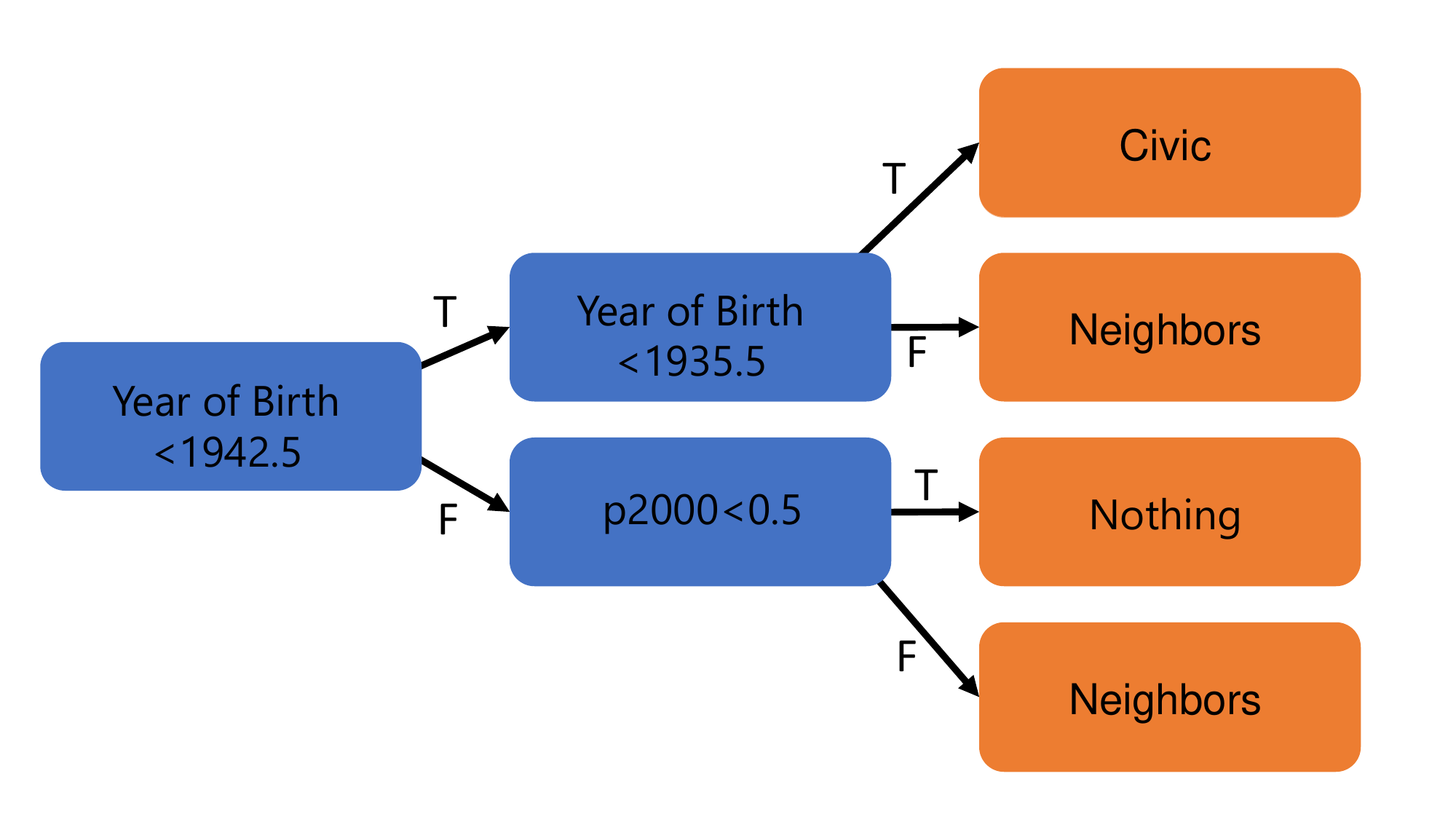}
		\caption{$\delta = 0.1$}
	\end{subfigure}
	\hfill
	\begin{subfigure}[b]{0.45\textwidth}
		\centering
		\includegraphics[width = \textwidth]{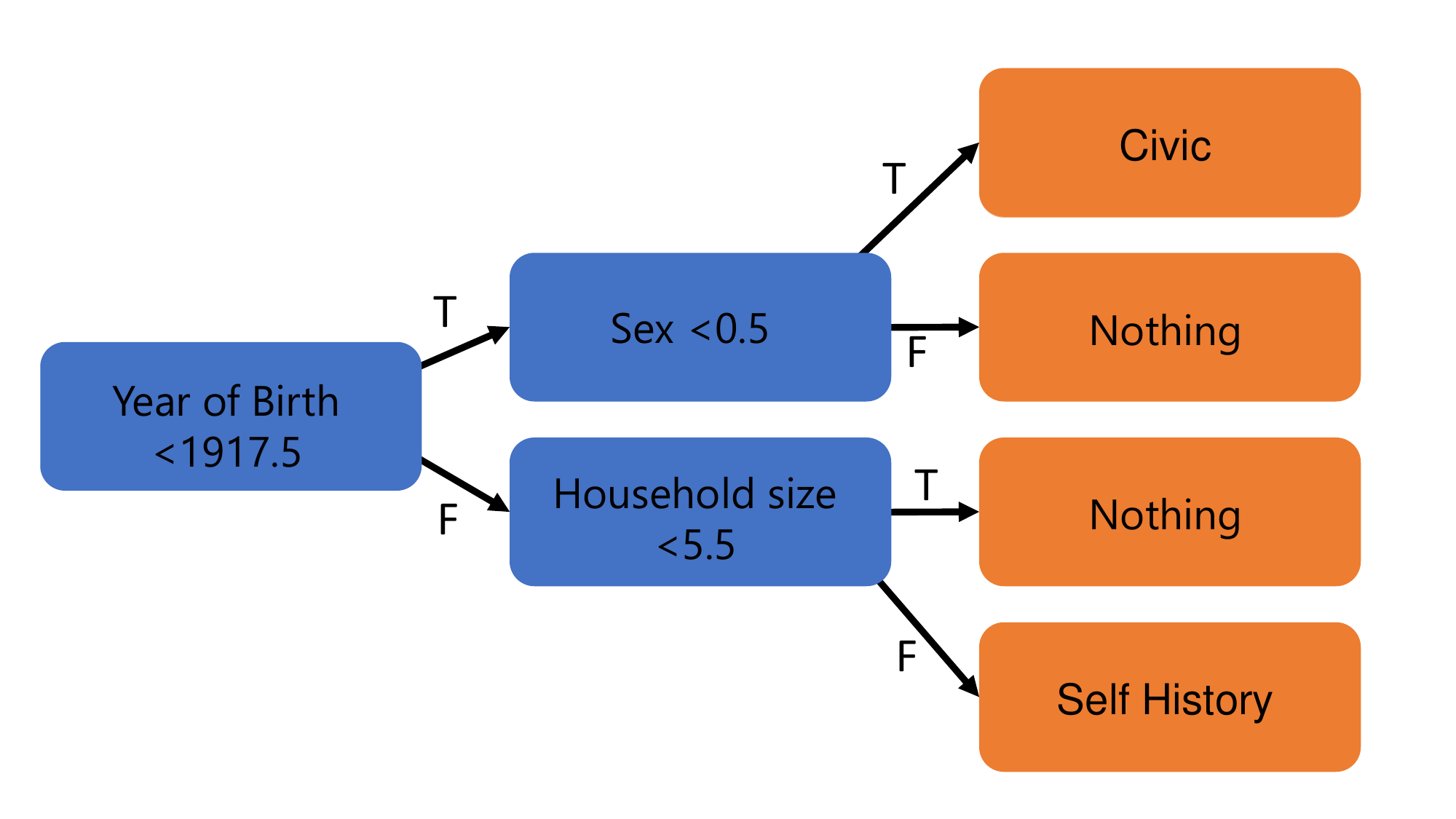}
		\caption{$\delta = 0.4$}
	\end{subfigure}
	\caption{Examples of distributionally robust decision trees when (a) $%
		\protect\delta = 0.1$, (b) $\protect\delta = 0.4$.}
	\label{fig:voting-tree}
\end{figure}

\section{Extension to $f$-divergence Uncertainty Set}

\label{sec:extension} In this section, we generalize the KL-divergence to $f$%
-divergence. Here, we define $f$-divergence between $\P $ and $\P _0 $ as
\begin{equation*}  \label{def:f-divergence}
	D_f(\P ||\P _0)\triangleq\int f\left(\frac{d\P }{d\P _0} \right)d\P _0,
\end{equation*}%
where $f:\mathbf{R} \rightarrow \mathbf{R}_{+} \cup \{+\infty\}$ is a convex
function satisfying $f(1)=0$ and $f(t)=+\infty$ for any $t<0$. Then, we define
the $f$-divergence uncertainty set as $\mathcal{U}^f_{\P _0}(\delta)
\triangleq \{\P \ll \P _0\mid D_f(\P ||\P _0)\leq \delta\}.$ Accordingly,
the distributionally robust value function is defined below.
\begin{definition}\quad
	\label{def:f-divergence_dro}
	For a given $\delta > 0$, the distributionally robust value function $Q^f_{\mathrm{\rm DRO}}:\Pi \rightarrow \reals$ is defined as: $Q^f_{\mathrm{\rm DRO}}(\pi) \triangleq\inf_{\P\in\mathcal{U}^f_{\P_0}(\delta)} \E_{\P}[Y(\pi(X))]$.
\end{definition}
%\begin{equation*}
%Q_\mathrm{\rm DRO}^f=\inf_{\P\ll \P_0}\left\{\mathbf{E}_\P [Y]:D_f(\P ||\P _0)\leq \delta
%\right\}
%\end{equation*}

We focus on Cressie-Read family of $f$-divergence, defined in \cite%
{cressie1984multinomial}. For $k\in (1,\infty ),$ function $f_{k}$ is
defined as
\begin{equation*}
	f_{k}(t) \triangleq\frac{t^{k}-kt+k-1}{k(k-1)}.
\end{equation*}%
As $k\rightarrow 1,$ $f_{k}\rightarrow f_{1}(t)=t\log t-t+1,$which becomes
KL-divergence. For the ease of notation, we use $Q_{\mathrm{\mathrm{DRO}}%
}^{k}\left( \cdot \right) $, $\mathcal{U}_{\P _0}^k (\delta)$, and $%
D_{k}\left( \cdot ||\cdot \right) $ as shorthands of $Q_{\mathrm{\mathrm{DRO}%
}}^{f}\left( \cdot \right) $, $\mathcal{U}_{\P _0}^f (\delta)$, and $%
D_{f}\left( \cdot ||\cdot \right) , $ respectively, for $k\in \lbrack
1,\infty ).$ We further define $k_{\ast }\triangleq k/(k-1),$ and $%
c_{k}(\delta )\triangleq(1+k(k-1)\delta )^{1/k}.$ Then, \cite%
{duchi2018learning} give the following duality results.

\begin{lemma}
	For any Borel measure $\P$ supported on the space $\mathcal{X}
	\times \prod_{j=1}^d \mathcal{Y}_j$ and $k\in (1,\infty ),$ we have
	\begin{equation*}
		\inf_{\mathbf{Q}\in\mathcal{U}^k_{\P}(\delta)} \E_{\mathbf{Q}}[Y(\pi(X))] =\sup_{\alpha \in \reals}\left\{
		-c_{k}\left( \delta \right) {\mathbf{E}}_{\P}\left[ \left( -Y(\pi(X))+\alpha \right) _{+}^{k_{\ast }}\right] ^{\frac{1}{k_{\ast }}%
		}+\alpha \right\} .
	\end{equation*}
	\label{lma:f_duality}
\end{lemma}
We then generalize Lemma \ref{lma:quantile} to Lemma \ref%
{lma:quantile_f_divergence} and Lemma \ref{lma:discrete_key} to Lemma \ref%
{lma:discrete_f_divergence} for $f_k$-divergence uncertainty set.
\begin{lemma}
	For any probability measures $\mathbf{P}_{1},\mathbf{P}_{2}$ supported on $\reals$ and  $k \in [1,+\infty)$,  we have
	\label{lma:quantile_f_divergence}
	\begin{eqnarray*}
		&&\left\vert \sup_{\alpha \in \reals}\left\{ -c_{k}\left( \delta \right)
		\mathbf{E}_{\mathbf{P}_{1}}\left[ \left( -Y+\alpha \right) _{+}^{k_{\ast }}%
		\right] ^{\frac{1}{k_{\ast }}}+\alpha \right\} -\sup_{\alpha \in \reals%
		}\left\{ -c_{k}\left( \delta \right) \mathbf{E}_{\mathbf{P}_{2}}\left[
		\left( -Y+\alpha \right) _{+}^{k_{\ast }}\right] ^{\frac{1}{k_{\ast }}%
		}+\alpha \right\} \right\vert \\
		&\leq& c_k(\delta) \sup_{t\in \lbrack 0,1]}\left\vert q_{\mathbf{P}_{1}}\left( t\right)
		-q_{\mathbf{P}_{2}}\left( t\right) \right\vert.
	\end{eqnarray*}
\end{lemma}

\begin{lemma}Suppose $\mathbf{P}_{1}$ and $\mathbf{P}_{2}$ are supported on $%
	\mathbb{D}$ and satisfy Assumption \ref{assump:classical}.3. We further assume $\mathbf{P}_{2}$ satisfies Assumption \ref{assump:pos_dens}.2. When $\mathrm{TV}(\mathbf{P}_{1},\mathbf{%
		P}_{2})<\underline{b}/2,$ we have for $k>1$
	\begin{eqnarray*}
		&&\left\vert \sup_{\alpha \in \reals}\left\{ -c_{k}\left( \delta \right)
		\mathbf{E}_{\mathbf{P}_{1}}\left[ \left( -Y+\alpha \right) _{+}^{k_{\ast }}%
		\right] ^{\frac{1}{k_{\ast }}}+\alpha \right\} -\sup_{\alpha \in \reals%
		}\left\{ -c_{k}\left( \delta \right) \mathbf{E}_{\mathbf{P}_{2}}\left[
		\left( -Y+\alpha \right) _{+}^{k_{\ast }}\right] ^{\frac{1}{k_{\ast }}%
		}+\alpha \right\} \right\vert \\
		\leq && \frac{2c_{k}\left( \delta \right) M}{\underline{b}k^{\ast }}\left( \underline{b}%
		/2\right) ^{1/k_{\ast }}\mathrm{TV}(\mathbf{P}_{1},\mathbf{P}_{2}).
	\end{eqnarray*}
	where $\mathrm{TV}$ denotes the total variation distance.
	\label{lma:discrete_f_divergence}
\end{lemma}

The proofs of Lemma \ref{lma:quantile_f_divergence} and \ref{lma:discrete_f_divergence} and are in  \ref%
{sec:extension_proof}. As an analog of Definitions \ref{def:regret}, \ref%
{def:phi_alpha} and Equation \eqref{eqn:pi_hat}, we further make the
following definition.
\begin{definition}
	\begin{enumerate}
		\item
		The distributionally robust value estimator $\hat{Q}^k_{\mathrm{\rm DRO}}:\Pi \rightarrow \reals$ is defined by
		\[
		\hat{Q}^k_{\mathrm{\rm DRO}}(\pi) \triangleq \sup_{\alpha \in \reals}\left\{
		-c_{k}\left( \delta \right) \frac{1}{nS_n^\pi}\sum_{i=1}^n\frac{\mathbf{1}\{\pi(X_i) = A_i\}}{\pi_0(A_i\mid X_i)} \left( \left( -Y(\pi(X))+\alpha \right) _{+}^{k_{\ast }}\right) ^{\frac{1}{k_{\ast }}%
		}+\alpha \right\}.
		\]
		\item The distributionally robust regret $R^k_{\mathrm{\rm DRO}}(\pi)$ of a  policy $\pi \in \Pi$ is defined as
		\begin{equation*}
			R^k_{\mathrm{\rm DRO}}(\pi) \triangleq\max_{\pi' \in \Pi}\inf_{\P\in\mathcal{U}^k_{\P_0}(\delta)}\E_{\P}[Y(\pi'(X))] - \inf_{\P\in\mathcal{U}^k_{\P_0}(\delta)}\E_{\P}[Y(\pi(X))].
		\end{equation*}
		\item The optimal policy $\hat{\pi}^k_{\rm DRO}$  which maximizes the value of $\hat{Q}^k%
		_{\rm DRO}$ is defined as
		\[
		\hat{\pi}^k_{\rm DRO}\triangleq\arg \max_{\pi \in \Pi }\hat{Q}^k_{\rm DRO}(\pi).
		\]
		
	\end{enumerate}
\end{definition}
By applying Lemmas \ref{lma:quantile_f_divergence} and \ref{lma:discrete_f_divergence} and executing the same
lines as the proof of Theorem \ref{DRO_Uniform}, we have an analogous
theorem below.
\begin{theorem}
	Suppose Assumption \ref{assump:classical} is enforced and $k>1$. With probability at least $1-\varepsilon$, under Assumption \ref{assump:pos_dens}.1, we have
	\begin{eqnarray*}
		R^k_{\rm DRO}(\hat{\pi}^k_{\rm DRO})  &\leq& \frac{4c_k(\delta)}{\underline{b}\eta\sqrt{n}}\left( 24(\sqrt{2}+1)\kappa ^{(n)}\left( \Pi \right)  +\sqrt{2\log\left(\frac{2}{\varepsilon}\right)}+C\right),
	\end{eqnarray*}
	where $C$ is a universal constant; and under Assumption \ref{assump:pos_dens}.2, when
	\[
	n \geq\left\{\frac{4}{\underline{b}\eta
	}\left( 24(\sqrt{2}+1)\kappa ^{(n)}\left( \Pi \right) +48\sqrt{|%
		\mathbb{D}|\log \left( 2\right) }+\sqrt{2\log \left( \frac{2}{\varepsilon}\right) }\right)\right\}^2  ,
	\]we have
	\begin{equation*}
		R_{\mathrm{DRO}}(\hat{\pi}_{\mathrm{DRO}})\leq \frac{4c_{k}\left( \delta \right) M}{k^{\ast }\underline{b}\eta \sqrt{n} }\left( \underline{b}%
		/2\right) ^{1/k_{\ast }} \left( 24(\sqrt{2}+1)\kappa ^{(n)}\left( \Pi \right) +48\sqrt{|%
			\mathbb{D}|\log \left( 2\right) }+\sqrt{2\log \left( \frac{2}{\varepsilon}\right) }\right) .
	\end{equation*}
\end{theorem}

%Then, we also provide a match lower bound.
%\begin{theorem}
%\label{thm:ld2}
%Let $d=2$, and $\delta \leq 0.226$. Then, for any policy $\pi $ as a function of $\{X_{i},A_{i},Y_{i}%
%\}_{i=1}^{n},$ it holds that
%\begin{equation*}
%\sup_{\P_0\in \mathcal{P}(M)}\E_{\P_0^n}\left[R^k_{\mathrm{DRO}}(\pi,\P_0 )\right]\geq \frac{M \kappa
	%^{(n)}\left(\Pi \right) }{100\sqrt{n}}, \text{ for }n\geq \kappa
%^{(n)}\left(\Pi \right).
%\end{equation*}
%\end{theorem}

We emphasize on the importance of Assumption \ref{assump:pos_dens}. Without
Assumption \ref{assump:pos_dens}, \cite{duchi2018learning} show a minimax
rate (in the supervise learning setting)
\begin{equation}
	R^k_{\mathrm{\mathrm{DRO}}}(\hat{\pi}^k_{\mathrm{DRO}}) =O_p\left( n^{-\frac{%
			1}{k_* \vee 2}} \log n \right),
\end{equation}
which is much slower than our results under a natural assumption (Assumption %
\ref{assump:pos_dens}).

\section{Conclusion}

%{Doubly robust estimator:
	%Our previous requires that we know the data-collecting policy $\pi_0$. When we do not know $\pi_0$, we can estimate it and we propose a doubly robust estimator here:
	%\begin{align*}
	%&\sup_{\alpha\geq 0}\left\{ -\alpha \log\left\{\frac{1}{\sum_{i=1}^n\frac{\mathbf{1}\{\pi(X_i)=A_i\}}{\hat{\pi}_0(A_i\mid X_i)}}\sum_{i=1}^n
	%	\exp(-\hat{Y}_i(A_i)/\alpha)+\frac{(\exp(-Y_i(A_i)/\alpha)-\exp(-\hat{Y}_i(A_i)/\alpha))\mathbf{1}\{\pi(X_i) = A_i\}}{\hat{\pi}_0(A_i\mid X_i)}\right\} \right. \\
	%&\left.- \alpha \delta\right\},
	%\end{align*}}
	We have provided a distributionally robust formulation for policy evaluation
	and policy learning in batch contextual bandits. Our results focus on
	providing finite-sample learning guarantees. Especially interesting  is that such
	learning is enabled by a dual optimization formulation.
	
	{
		A natural
		subsequent direction would be to extend the algorithm and results to the
		Wasserstein distance case for batch contextual bandits, which cannot be
		classified as a special case in our $f$-divergence framework. We remark that the extension is non-trivial.
		Given a lower semicontinuous function $c$, recall that the Wasserstein distance between two measures, $\P$ and $\mathbf{Q}$, is defined as
		$D_{W_c}(\P,\mathbf{Q})= \min_{\pi \in \Pi(\P,\mathbf{Q})}  E_{\pi}\left\{c(X,Y)\right\},$
		where $\Pi(\P,\mathbf{Q})$ denotes the set of all joint distributions of the
		random vector $(X,Y)$ with marginal distributions $\P$ and $\mathbf{Q}$,
		respectively.  The key distinguishing feature of Wasserstein distance is that unlike $f$-divergence, it does not restrict the perturbed distributions to have the same support as $\P_0$, thus including more realistic scenarios. This feature, although desirable, also makes distributionally robust policy learning more challenging. To illustrate the difficulty, we consider the separable cost function family  $c\left(\left(x,y_1,y_2,\ldots,y_d\right), \left(x',y_1',y_2',\ldots,y_d'\right)\right)=d(x,x')+ \alpha \sum_{i=1}^d d(y_i,y_i')$ with $\alpha>0$, where $d$ is a metric.
		We aim to find $\pi_{\mathrm{W-DRO}}^*$ that maximizes $Q_{{\mathrm{W-DRO}}}(\pi) \triangleq\inf_{D_{W_c}(\P,\P_0) \leq \delta} \E_{\P}[Y(\pi(X))]$.  Leveraging strong duality results \citep{blanchet2016quantifying,gao2016distributionally,MohajerinEsfahani2017}, we can write:
		\begin{equation*}
			Q_{\mathrm{W-DRO}}(\pi) = \sup_{\gamma \geq 0}\left\{-\gamma \delta +\E_{\P_0}\left[\inf_{u,v} \left\{v+\gamma (d(X,u) + \alpha d(v,Y(\pi(u)))\right\} \right]\right\}.
		\end{equation*}
		However, the difficulty now is that $Y(\pi(u))$ is not observed if $\pi_0(u) \neq \pi(u)$. We leave this challenge for future work.}

\bibliographystyle{plainnat}
\bibliography{references_sim,DR2,DR4,f_divergence}
\setcounter{section}{0} \setcounter{subsection}{0} \setcounter{equation}{0}
\ \renewcommand
{\thesection}{Appendix \Alph{section}}
\renewcommand\thesubsection{\thesection.\arabic{subsection}}
\renewcommand{\theequation}{\Alph{section}.\arabic{equation}}
\renewcommand{\thelemma}{\Alph{section}\arabic{lemma}}
\renewcommand{\thetheorem}{\Alph{section}\arabic{theorem}} \newpage

\section{Proofs of Main Results}
\subsection{Auxiliary Results}
\label{sec:aux} To prove Theorem \ref{DRO_CLT} and \ref{DRO_Uniform},
we first collect some theorems in stochastic optimization \cite{shapiro2009lectures} and complexity theory (\citet%
[Section 4]{wainwright2019high}).

%\begin{theorem}[Functional central limit theorem, Corollary 7.17 in
%\protect\cite{araujo1980central}]
%Let $S$ be a compact subspace of $\mathbf{R}^{d}$ and $\mathcal{C}(S)$ be
%the space of continuous bounded random functions on $S$ equipped with the
%sup-norm. Let $\left\{ X_{n}\right\} _{n=1}^{\infty }$ be a sequence of
%centered i.i.d. $\mathcal{C}(S)-$valued random functions such that $\mathbf{E%
%}X_{n}^{2}\left( s\right) <\infty $ for some $s\in S.$ Suppose there exists
%a constant $M$ such that
%\begin{equation*}
%|X_{1}(s)-X_{2}(t)|\leq M\left\Vert s-t\right\Vert \text{ almost surely,}
%\end{equation*}%
%for all $s,t\in S.$ Then,
%\begin{equation*}
%\frac{1}{\sqrt{n}}\sum_{i=1}^{n}X_{n}\Rightarrow Y,
%\end{equation*}%
%where $Y$ is a Gaussian process on $\mathcal{C}(S).$
%\end{theorem}

\begin{definition}[G\^{a}teaux and Hadamard directional
differentiability]
Let $B_{1}$ and $B_{2}$ be Banach spaces and $G:B_{1}\rightarrow B_{2}$ be a
mapping. It is said that $G$ is directionally differentiable at a considered
point $\mu \in B_{1}$ if the limits
\begin{equation*}
G_{\mu }^{\prime }(d)=\lim_{t\downarrow 0}\frac{G\left( \mu +td\right)
-G(\mu )}{t}
\end{equation*}%
exists for all $d\in B_{1}.$

Furthermore, it is said that $G$ is G\^{a}teaux directionally differentiable
at $\mu $ if the directional derivative $G_{\mu }^{\prime }(d)$ exists for
all $d\in B_{1}$ and $G_{\mu }^{\prime }(d)$ is linear and continuous in $d.$
For ease of notation, we also denote $D_{\mu }(\mu _{0})$ be the operator $%
G_{\mu _{0}}^{\prime }(\cdot ).$

Finally, it is said that $G$ is Hadamard directionally differentiable at $%
\mu $ if the directional derivative $G_{\mu }^{\prime }(d)$ exists for all $%
d\in B_{1}$ and
\begin{equation*}
G_{\mu }^{\prime }(d)=\lim_{\substack{ t\downarrow 0  \\ d^{\prime
}\rightarrow d}}\frac{G\left( \mu +td^{\prime }\right) -G(\mu )}{t}.
\end{equation*}
\end{definition}

\begin{theorem}[Danskin theorem, Theorem 4.13 in \protect\cite%
{bonnansperturbation}]
Let $\Theta \in \mathbf{R}^{d}$ be a nonempty compact set and $B$ be a
Banach space. Suppose the mapping $G:B\times \Theta \rightarrow \mathbf{R}$
satisfies that $G(\mu ,\theta )$ and $D_{\mu }\left( \mu ,\theta \right) $
are continuous on $O_{\mu _{0}}\times \Theta $, where $O_{\mu _{0}}\subset B$
is a neighborhood around $\mu _{0}.$ Let $\phi :B\rightarrow \mathbf{R}$ be
the inf-functional $\phi (\mu )=\inf_{\theta \in \Theta }G(\mu ,\theta )$
and $\bar{\Theta}(\mu )=\arg \max_{\theta \in \Theta }G(\mu ,\theta ).$
Then, the functional $\phi $ is directionally differentiable at $\mu _{0}$
and
\begin{equation*}
G_{\mu _{0}}^{\prime }(d)=\inf_{\theta \in \bar{\Theta}(\mu _{0})}D_{\mu
}\left( \mu _{0},\theta \right) d.
\end{equation*}
\end{theorem}

\begin{theorem}[Delta theorem, Theorem 7.59 in \protect\cite%
{shapiro2009lectures}]
Let $B_{1}$ and $B_{2}$ be Banach spaces, equipped with their Borel $\sigma $%
-algrebras, $Y_{N}$ be a sequence of random elements of $B_{1}$, $%
G:B_{1}\rightarrow B_{2}$ be a mapping, and $\tau _{N}$ be a sequence of
positive numbers tending to infinity as $N\rightarrow \infty .$ Suppose that
the space $B_{1}$ is separable, the mapping $G$ is Hadamard directionally
differentiable at a point $\mu \in B_{1},$ and the sequence $X_{N}=\tau
_{N}\left( Y_{N}-\mu \right) $ converges in distribution to a random element
$Y$ of $B_{1}.$ Then,%
\begin{equation*}
\tau _{N}\left( G\left( Y_{N}\right) -G\left( \mu \right) \right)
\Rightarrow G_{\mu }^{\prime }\left( Y\right) \text{ in distribution,}
\end{equation*}%
and
\begin{equation*}
\tau _{N}\left( G\left( Y_{N}\right) -G\left( \mu \right) \right) =G_{\mu
}^{\prime }\left( X_{N}\right) +o_{p}(1).
\end{equation*}
\end{theorem}

\begin{proposition}[Proposition 7.57 in \protect\cite{shapiro2009lectures}]
Let $B_{1}$ and $B_{2}$ be Banach spaces, $G:B_{1}\rightarrow B_{2},$ and $%
\mu \in B_{1}.$ Then the following hold: (i) If $G\left( \cdot \right) $ is
Hadamard directionally differentiable at $\mu ,$ then the directional
derivative $G_{\mu }^{\prime }\left( \cdot \right) $ is continuous. (ii) If $%
G(\cdot )$ is Lipschitz continuous in a neighborhood of $\mu $ and
directionally differentiable at $\mu ,$ then $G(\cdot )$ is Hadamard
directionally differentiable at $\mu .$
\end{proposition}
\begin{definition}[Rademacher complexity]
	Let $\mathcal{F}$ be a family of real-valued functions $f:Z\rightarrow
	\mathbf{R.}$ Then, the Rademacher complexity of $\mathcal{F}$ is defined as
	\begin{equation*}
	\mathcal{R}_{n}\left( \mathcal{F}\right)\triangleq\mathbf{E}_{z,\sigma}
\left[ \sup_{f\in \mathcal{F}}\left|\frac{1}{n}%
	\sum_{i=1}^{n}\sigma _{i}f(z_{i})\right|\right]  ,
	\end{equation*}%
	where $\sigma _{1},\sigma _{2},\ldots ,\sigma _{n}$ are i.i.d with the
	distribution $\P \left( \sigma _{i}=1\right) =\P \left( \sigma
	_{i}=-1\right) =1/2.$
\end{definition}

\begin{theorem}[Theorem 4.10 in \cite{wainwright2019high}]
	\label{lemma:rademacher}
	
	If $f(z)\in  [-B,B],$ we have with probability at least $%
	1-\exp \left(- \frac{n\epsilon^2}{2B^2} \right) $,
	\begin{equation*}
		\sup_{f\in \mathcal{F}}\left| \frac{1}{n}\sum_{i=1}^{n}f(z_{i})-\mathbf{E}%
		f(z)\right| \leq
		2\mathcal{R}_{n}\left( \mathcal{F}\right) +\epsilon.
	\end{equation*}
\end{theorem}
\begin{theorem}[Dudley's Theorem, (5.48) in \cite{wainwright2019high}]If $f(z)\in  [-B,B],$ we have a bound for the Rademacher complexity,
\label{Dudley}
\begin{equation*}
\mathcal{R}_{n}\left( \mathcal{F}\right) \leq \E\left[ \frac{24}{\sqrt{n}}%
\int_{0}^{2B}\sqrt{\log N(t,\mathcal{F}\text{,}\left\Vert \cdot \right\Vert
_{\P_{n}})}{\rm d}t\right] ,
\end{equation*}%
where $N(t,\mathcal{F}$,$\left\Vert \cdot \right\Vert _{\P_{n}})$ is $t$%
-covering number of set $\mathcal{F}$ and the metric $\left\Vert \cdot
\right\Vert _{\P_{n}}$ is defined by
\begin{equation*}
\left\Vert \cdot \right\Vert _{\P_{n}}\triangleq\sqrt{\frac{1}{n}\sum_{i=1}^{n}\left(
f(z_{i}\right) -g(z_{i}))^{2}}.
\end{equation*}%
\end{theorem}
\subsection{Proofs of Lemma \ref{thm:strong_duality} Corollary \ref{cor:worst_case} and  Lemma \protect\ref{lma:cvx} in Section \ref{sec:alg}}

\begin{proof}[Proof of Lemma \ref{thm:strong_duality}]
	The first equality follows from \citet[Theorem 1]{hu2013kullback}. The second equality holds, because
	for any (Borel measurable) function $f:\reals\rightarrow\reals$ and any policy $\pi\in\Pi$, we have
	\begin{equation}
	\E_{\P}\left[f(Y(\pi(X)))\right] = \E_{\P*\pi_0}\left[\frac{f(Y(\pi(X)))\mathbf{1}\{\pi(X) = A\}}{\pi_0(A\mid X)}\right] =
	\E_{\P*\pi_0}\left[\frac{f(Y(A))\mathbf{1}\{\pi(X) = A\}}{\pi_0(A\mid X)}\right].
	\end{equation}
	Plugging in $f(x) = \exp(-x/{\alpha})$ yields the result.
\end{proof}
\label{sec:lma_cvx}
\begin{proof}[Proof of Corollary \ref{cor:worst_case}]
Since $Y(a^1),Y(a^2),\ldots,Y(a^d)$ are mutually independent conditional on $X$, and $Y(a^i)|X$ has a density if $Y(a^i)$ is a continuous random variable, we can write the joint measure of $Y(a^1),Y(a^2),\ldots,Y(a^d),X$ as
\[
\left(\prod_{i=1}^df_i(y_i|x)   \lambda({\rm d} y_i) \right)\mu({\rm d} x),
\]
where $\lambda$ denotes the Lebesgue measure in $\reals$ if Assumption \ref{assump:pos_dens}.1 is enforced, and  denotes a uniformly distributed measure on $\mathbb{D}$ that $\lambda(d) = 1$ for $d \in \mathbb{D}$ if Assumption \ref{assump:pos_dens}.2 is enforced, and $\mu({\rm d}x)$ denotes the measure induced by $X$ on the space $\mathcal{X}$. Without loss of generality, we assume $\essinf\{Y(\pi(X)\} =0$. Then, For the simplicity of  notation, when $\alpha^*(\pi) =0,$ we write
\[
\exp(-Y(\pi(X)/\alpha^*(\pi)))=\mathbf{1}\{Y(\pi(X))=\essinf\{Y(\pi(X)\}\},
\]
which means
\[
{\E_{\P_0}[\exp(-Y(\pi(X)/\alpha^*(\pi)))]}={\P_0(Y(\pi(X))=\essinf\{Y(\pi(X)\})}.
\]Then, by Proposition \ref{prop:worst_case}, we have under $\P(\pi)$, $Y(a^1),Y(a^2),\ldots,Y(a^d),X$ have a joint measure
\begin{eqnarray*}
\frac{\exp(-Y(\pi(X)/\alpha^*(\pi)))}{\E_{\P_0}[\exp(-Y(\pi(X)/\alpha^*(\pi)))]}
\left(\prod_{i=1}^df_i(y_i|x)  \lambda(  {\rm d} y_i) \right)\mu({\rm d} x)
=\left(\prod_{i=1}^df'_i(y_i|x)   \lambda({\rm d} y_i) \right)\mu'({\rm d} x),
\end{eqnarray*}
where
\begin{equation*}
f_i'(y_i|x)=\left\{
\begin{array}{c}
f_i(y_i|x) \\
\frac{\exp(-y_{i} /\alpha^*(\pi))f_i(y_i|x)}{\int\exp(-y_i/\alpha^*(\pi))f_i(y_i|x)\lambda({\rm d} y_i)}%
\end{array}%
\begin{array}{c}
\text{for } i\neq \pi(x) \\
\text{for } i= \pi(x)
\end{array}%
\right.,
\end{equation*}
and
\begin{equation*}
\mu'({\rm d} x)=\frac{\int\exp(-y_i/\alpha^*(\pi))f_i(y_i|x)\mu({\rm d} y_i)}{\E_{\P_0}[\exp(-Y(\pi(X)/\alpha^*(\pi)))]},
\end{equation*}
which completes the proof.
\end{proof}
\begin{proof}[Proof of Lemma \ref{lma:cvx}]
    The closed form expression of $\frac{\partial}{\partial \alpha}\hat{\phi}_{n}(\pi,\alpha)$ and $\frac{\partial^2}{\partial \alpha^2}
	\hat{\phi}_{n}(\pi,\alpha)$ follows from elementary algebra. By the Cauchy Schwartz's inequality, we have
	\[
	    \Big(\sum_{i=1}^n Y_i(A_i)W_i(\pi,\alpha)\Big)^2
	    \leq n S_n^\pi \hat{W}_{n}(\pi ,\alpha)
	    \sum_{i=1}^n Y^2_i(A_i)W_i(\pi,\alpha).
	\]
	Therefore, it follows that
	$\frac{\partial^2}{\partial \alpha^2}
	\hat{\phi}_{n}(\pi,\alpha)\leq 0$. Note that the Cauchy Schwartz's inequality is actually an equality if and only if
	\[
	Y^2_i(A_i)W_i(\pi,\alpha) = cW_i(\pi,\alpha) \quad \mbox{ if } \quad
	W_i(\pi,\alpha)\neq 0
	\]
	for some constant $c$ independent of $i$. Since the above condition is violated if  $\{Y_i(A_i)\mathbf{1}\{\pi(X_i) = A_i\}\}_{i=1}^n$ has at least two different non-zero values, we have in this case $\hat{\phi}_{n}(\pi,\alpha)$ is strictly-concave in $\alpha$.
\end{proof}	

\subsection{Proof of the central limit theorem in Section \ref{sec:CLT}}
\label{sec:thm1}

We first give the upper and lower bounds for $\alpha ^{\ast }(\pi)$ in Lemmas \ref{lemma:upper_bd} and \ref{lemma:ld}.

\begin{lemma}[Uniform upper bound of $\protect\alpha ^{\ast }(\pi)$]

\label{lemma:upper_bd} Suppose that Assumption \ref{assump:classical}.3 is
imposed, we have the optimal dual solution $\alpha ^{\ast }(\pi)\leq \overline{%
\alpha }$ and the empirical dual solution $\alpha_n(\pi)\leq \overline{%
\alpha }$, where $ \overline{%
\alpha }=M/\delta $ .
%and if $S_n^{\pi}\geq \exp(-\delta/2),$ the optimal empirical dual solution $\alpha ^{(n)}\leq 2\overline{\alpha }.$
\end{lemma}

\begin{proof}[Proof]
First note that $\inf_{\P \in \mathcal{U}_{\P _{0}(\delta )}}\mathbf{E}_{\P }%
\left[ Y(\pi (X))\right] \geq \essinf{Y(\pi(X))}\geq 0$ and
\begin{equation*}
-\alpha \log \mathbf{E}_{\mathbf{P}_{0}}\left[ \exp \left( -Y(\pi (X)\right)
/\alpha \right] -\alpha \delta \leq M-\alpha \delta .
\end{equation*}%
$M-\alpha \delta \geq 0$ gives the upper bound $\alpha ^{\ast }(\pi)\leq
\overline{\alpha }:=M/\delta .$

%Secondly, for empirical dual solution $\alpha ^{(n)},$ we have $\lim_{\alpha
%\rightarrow 0}-\alpha \log \hat{W}_{n}(\pi ,\alpha )-\alpha \delta \geq
%\min_{i\in \lbrack n]}(Y_{i})\geq 0$. Thus the maximizer $\alpha ^{(n)}$ is
%upper bounded by $2\overline{\alpha }$ if for all $\alpha >2\overline{\alpha
%}$ we have $-\alpha \log \hat{W}_{n}(\pi ,\alpha )-\alpha \delta \leq 0$.
%Actually, in view of the following inequality
%\begin{equation*}
%-\alpha \log \hat{W}_{n}(\pi ,\alpha )-\alpha \delta \leq M-\alpha \log
%\left( \frac{1}{n}\sum_{i=1}^{n}\frac{\mathbf{1}\{\pi (X_{i})=A_{i}\}}{%
%e_{A_{i}}(X_{i})}\right) -\alpha \delta ,
%\end{equation*}%
%If $\frac{1}{n}\sum_{i=1}^{n}\frac{\mathbf{1}\{\pi (X_{i})=A_{i}\}}{%
%e_{A_{i}}(X_{i})}\geq \exp \left( -\delta /2\right) ,$ we have
%\begin{equation*}
%\log \left( \frac{1}{n}\sum_{i=1}^{n}\frac{\mathbf{1}\{\pi (X_{i})=A_{i}\}}{%
%e_{A_{i}}(X_{i})}\right) \geq -\delta /2.
%\end{equation*}%
%By following an analogous argument, we have $\alpha ^{(n)}\leq 2\overline{%
%\alpha }.$
\end{proof}

\begin{lemma}[Lower bound of $\protect\alpha ^{\ast }(\pi)$]
	\label{lemma:ld} Suppose that Assumption \ref{assump:pos_dens}.1 is
	imposed, we have
$\alpha ^{\ast } (\pi)>0$.
\end{lemma}
\begin{proof}[Proof]
To ease the notation, we abbreviate $Y(\pi (X))$ as $Y.$ It is easy to check
the density of $f_{Y}$ has a lower bound $\underline{b}>0.$ Since $f_{Y}$ is
a continuous function on a compact space, we have $f_{Y}$ is upper bounded.
Let $\overline{b}(\pi)=\sup_{x\in \lbrack 0,M]}f_{Y}(x).$

First, note that $\lim_{\alpha \rightarrow 0}\phi \left( \pi ,\alpha \right)
=0.$ We only need to show $\lim \inf_{\alpha \rightarrow 0}\frac{\partial
\phi \left( \pi ,\alpha \right) }{\partial \alpha }>0.$

The derivative of $\phi \left( \pi ,\alpha \right) $ is given by
\begin{equation*}
\frac{\partial \phi \left( \pi ,\alpha \right) }{\partial \alpha }=-\frac{%
\mathbf{E}\left[ Y/\alpha \exp \left( -Y/\alpha \right) \right] }{\mathbf{E}%
\left[ \exp \left( -Y/\alpha \right) \right] }-\log \left( \mathbf{E}\left[
\exp \left( -Y/\alpha \right) \right] \right) -\delta .
\end{equation*}%
Since $\mathbf{P}_{0 }$ has a continuous density, we have $\log \left(
\mathbf{E}\left[ \exp \left( -Y/\alpha \right) \right] \right) \rightarrow
-\infty .$ Notice that
\begin{equation*}
\mathbf{E}\left[ Y/\alpha \exp \left( -Y/\alpha \right) \right] \leq \alpha
\overline{b}\text{ and }\lim \inf_{\alpha \rightarrow 0}\mathbf{E}\left[
\exp \left( -Y/\alpha \right) \right] /\alpha \geq \underline{b}.
\end{equation*}%
Therefore, we have%
\begin{equation*}
\lim \sup_{\alpha \rightarrow 0}\frac{\mathbf{E}\left[ Y/\alpha \exp \left(
-Y/\alpha \right) \right] }{\mathbf{E}\left[ \exp \left( -Y/\alpha \right) %
\right] }\leq \frac{\overline{b}(\pi)}{\underline{b}}.
\end{equation*}%
Finally, we arrive the desired result
\begin{equation*}
\lim \inf_{\alpha \rightarrow 0}\frac{\partial \phi \left( \pi ,\alpha
\right) }{\partial \alpha }\rightarrow +\infty,
\end{equation*}
which completes the proof.
\end{proof}

\begin{lemma} Suppose Assumption \ref{assump:classical}.1 is enforced. We have pointwise central limit theorem,
\begin{eqnarray*}
\sqrt{n}\left( \hat{W}_{n}(\pi ,\alpha )-\mathbf{E}\left[ W_{i}\left( \pi
,\alpha \right) \right] \right) \Rightarrow \mathcal{N}\left( 0,\mathbf{E}\left[ \frac{1}{\pi _{0}\left( \pi
(X)\mid X\right) } \left( \exp \left( -Y(\pi (X))/\alpha \right)
-\mathbf{E}\left[ \exp \left( -Y(\pi (X))/\alpha \right) \right] \right)
^{2} \right] \right) ,
\end{eqnarray*}%
for any $\pi \in \Pi$ and $\alpha>0$.
\label{lma:CLT_pre}
\end{lemma}
\begin{proof}[Proof] After reformulation, we have
\begin{eqnarray*}
\hat{W}_{n}(\pi ,\alpha )-\mathbf{E}\left[ W_{i}\left( \pi ,\alpha \right) %
\right]  &=&\frac{\frac{1}{n}\sum_{i=1}^{n}W_{i}\left( \pi ,\alpha \right)
-\left( \frac{1}{n}\sum_{i=1}^{n}\frac{\mathbf{1}\left\{ \pi
(X_{i})=A_{i}\right\} }{\pi _{0}\left( A_{i}|X_{i}\right) }\right) \mathbf{E}%
\left[ W_{i}\left( \pi ,\alpha \right) \right] }{\frac{1}{n}\sum_{i=1}^{n}%
\frac{\mathbf{1}\left\{ \pi (X_{i})=A_{i}\right\} }{\pi _{0}\left(
A_{i}|X_{i}\right) }} \\
&=&\frac{\frac{1}{n}\sum_{i=1}^{n}\left( W_{i}\left( \pi ,\alpha \right) -%
\frac{\mathbf{1}\left\{ \pi (X_{i})=A_{i}\right\} }{\pi _{0}\left(
A_{i}|X_{i}\right) }\mathbf{E}\left[ W_{i}\left( \pi ,\alpha \right) \right]
\right) }{\frac{1}{n}\sum_{i=1}^{n}\frac{\mathbf{1}\left\{ \pi
(X_{i})=A_{i}\right\} }{\pi _{0}\left( A_{i}|X_{i}\right) }}.
\end{eqnarray*}%
The denominator $\frac{1}{n}\sum_{i=1}^{n}\frac{\mathbf{1}\left\{ \pi
(X_{i})=A_{i}\right\} }{\pi _{0}\left( A_{i}|X_{i}\right) }\overset{p}{%
\rightarrow }1$ and the nominator converges as
\begin{eqnarray*}
&&\frac{1}{\sqrt{n}}\sum_{i=1}^{n}\left( W_{i}\left( \pi ,\alpha \right) -%
\frac{\mathbf{1}\left\{ \pi (X_{i})=A_{i}\right\} }{\pi _{0}\left(
A_{i}|X_{i}\right) }\mathbf{E}\left[ W_{i}\left( \pi ,\alpha \right) \right]
\right)  \\
&\Rightarrow &N(0,\var \left( W_{i}\left( \pi ,\alpha \right) -\frac{%
\mathbf{1}\left\{ \pi (X_{i})=A_{i}\right\} }{\pi _{0}\left(
A_{i}|X_{i}\right) }\mathbf{E}\left[ W_{i}\left( \pi ,\alpha \right) \right]
\right) ,
\end{eqnarray*}%
where
\begin{eqnarray}
&&\mathbf{Var}\left( W_{i}\left( \pi ,\alpha \right) -\frac{\mathbf{1}%
\left\{ \pi (X_{i})=A_{i}\right\} }{\pi _{0}\left( A_{i}|X_{i}\right) }%
\mathbf{E}\left[ W_{i}\left( \pi ,\alpha \right) \right] \right) \notag \\
&=&\mathbf{E}\left[ W_{i}\left( \pi ,\alpha \right) ^{2}\right] -2\mathbf{E}%
\left[ \left( \frac{\mathbf{1}\left\{ \pi (X_{i})=A_{i}\right\} }{\pi
_{0}\left( A_{i}|X_{i}\right) }\right) ^{2}\exp \left( -Y(\pi (X))/\alpha
\right) \right] \mathbf{E}\left[ W_{i}\left( \pi ,\alpha \right) \right]  \notag \\
&&+\mathbf{E}\left[ \frac{\mathbf{1}\left\{ \pi (X_{i})=A_{i}\right\} }{\pi
_{0}\left( A_{i}|X_{i}\right) }\right] ^{2}\mathbf{E}\left[ W_{i}\left( \pi
,\alpha \right) \right] ^{2}  \label{eq:var_new} \\
&=&\mathbf{E}\left[ \frac{1}{\pi _{0}\left( \pi (X)|X\right) }\mathbf{E}%
\left[ \exp \left( -2Y(\pi (X))/\alpha \right) |X\right] \right] -2\mathbf{E}%
\left[ \frac{1}{\pi _{0}\left( \pi (X)|X\right) }\mathbf{E}\left[ \exp
\left( -Y(\pi (X))/\alpha \right) |X\right] \right] \mathbf{E}\left[
W_{i}\left( \pi ,\alpha \right) \right]  \notag \\
&&+\mathbf{E}\left[ \frac{1}{\pi _{0}\left( \pi (X)|X\right) }\right]
\mathbf{E}\left[ \exp \left( -Y(\pi (X))/\alpha \right) \right] ^{2} \notag\\
&=&\mathbf{E}\left[ \frac{1}{\pi _{0}\left( \pi (X)|X\right) }%
\left( \exp \left( -Y(\pi (X))/\alpha \right) -\mathbf{E}\left[ \exp
\left( -Y(\pi (X))/\alpha \right) \right] \right) ^{2} \right]  .  \notag
\end{eqnarray}%
By Slutsky's theorem, the desired result follows.

 \end{proof}
To ease the proofs below, we define $\mathbf{\hat{P}}^\pi_{n}$ as the weighted empirical distribution by
\begin{equation*}
\mathbf{\hat{P}}_{n}^\pi\triangleq\frac{1}{nS_{n}^{\pi }}\sum_{i=1}^{n}\frac{%
\mathbf{1\{}\pi (X_{i})=A_{i}\mathbf{\}}}{\pi _{0}\left( A_{i}|X_{i}\right) }%
\Delta \left\{ Y_{i},X_{i}\right\},
\end{equation*}%
where $\Delta \{\cdot\}$ denote a dirac measure.
\begin{lemma}
\label{lma:discrete_conv}
Suppose  Assumptions \ref{assump:classical}.1 and \ref{assump:pos_dens}.2 are enforced. Then, we have
\[
\lim_{n\rightarrow +\infty }\sup_{v\in \mathbb{D}}\left( |\hat{\P }_{n}^{\pi
}(Y=v)-\P _{0}(Y=v)|\right) =0\text{ almost surely.}
\]
\end{lemma}
\begin{proof}[Proof]
Note that%
\begin{equation*}
\hat{\P}_{n}^{\pi }(Y=v)=\frac{1}{nS_{n}^{\pi }}\sum_{i=1}^{n}\frac{%
\mathbf{1\{}\pi (X_{i})=A_{i}\mathbf{\}1\{}Y_{i}=v\mathbf{\}}}{\pi
_{0}\left( A_{i}|X_{i}\right) },
\end{equation*}%
and by Assumption \ref{assump:classical}.1,
\begin{eqnarray*}
\mathbf{E}\left[ \frac{\mathbf{1\{}\pi (X_{i})=A_{i}\mathbf{\}1\{}Y_{i}=v%
\mathbf{\}}}{\pi _{0}\left( A_{i}|X_{i}\right) }\right]  &=&\mathbf{E}\left[
\frac{\mathbf{E}\left[ \mathbf{1\{}\pi (X_{i})=A_{i}|X\mathbf{\}}\right]
\mathbf{E}\left[ \mathbf{1\{}Y_{i}=v|X\mathbf{\}}\right] }{\pi _{0}\left(
A_{i}|X_{i}\right) }\right]
=\P _{0}(Y=v).
\end{eqnarray*}%
By the law of large number, we have $S_{n}^{\pi }\rightarrow 1$ almost
surely, and
\begin{equation*}
\frac{1}{n}\sum_{i=1}^{n}\frac{\mathbf{1\{}\pi (X_{i})=A_{i}\mathbf{\}1\{}%
Y_{i}=v\mathbf{\}}}{\pi _{0}\left( A_{i}|X_{i}\right) }\rightarrow \P %
_{0}(Y=v)\text{ almost surely.}
\end{equation*}%
By Slutsky's lemma \citep[e.g.][Theorem
1.8.10]{lehmann2006theory}), we have $\hat{\P }_{n}^{\pi }(Y=v)-\P _{0}(Y=v) \rightarrow 0$
almost surely. Since $\mathbb{D}$ is a finite set, we arrive the desired
results.
\end{proof}
\begin{lemma}Suppose  Assumptions \ref{assump:classical}.1 and \ref{assump:pos_dens}.2 are enforced. We have that when $\alpha = 0$,
\[
\lim_{n\rightarrow \infty}\P_\otimes(\hat{\phi}_n(\pi,0) ={\phi}(\pi,0)) = 1,
\]
where  $\P_\otimes$ denotes  the product measure $\prod_{i=1}^{\infty}\P_0$, which is guaranteed to be unique by  the Kolmogorov extension theorem. and thus
\begin{eqnarray*}
\sqrt{n}(\hat{\phi}_n(\pi,0)-{\phi}(\pi,0)) \rightarrow 0 \text{ in probability.}
\end{eqnarray*}%
\label{lma:alpha0}
\end{lemma}
\begin{proof}[Proof]
By Remark \ref{remark:alpha=0}, we have
\[
\hat{\phi}_n(\pi,0) = \essinf \hat{Y}_n^\pi,
\]
where $\hat{Y}_n^\pi$ is the distribution of $Y$ under the measure $\P_n^\pi$. Therefore,
\[
\P_\otimes(\hat{\phi}_n(\pi,0) ={\phi}(\pi,0))=\P_\otimes(\hat{\P}_n^\pi(Y=\min_{v\in \mathbb{D}}{v })>0) \rightarrow 1,
\]
because $\hat{\P}_n^\pi(Y=\min_{v\in \mathbb{D}}{v })\rightarrow \P_0(Y=\min_{v\in \mathbb{D}}{v })$.
\end{proof}

Now we are ready to show the proof of Theorem \ref{DRO_CLT}.

\proof{Proof of Theorem \ref{DRO_CLT}}
We divide the proof into two parts: continuous case and discrete case.

(1) Continuous case, i.e., Assumption 2.1 is satisfied: Note that
\begin{equation*}
\sqrt{n}\left( \hat{W}_{n}(\pi ,\alpha )-\mathbf{E}[W_{i}(\pi ,\alpha
)]\right) \Rightarrow Z_1\left( \alpha \right) ,
\end{equation*}
where by Lemma \ref{lma:CLT_pre}
\begin{equation*}
Z_1\left( \alpha \right) \sim \mathcal{N}\left( 0,\mathbf{E}\left[ \frac{1}{\pi _{0}\left( \pi (X)|X\right) }%
\left( \exp \left( -Y(\pi (X))/\alpha \right) -\mathbf{E}\left[ \exp
\left( -Y(\pi (X))/\alpha \right) \right] \right) ^{2} \right] \right).
\end{equation*}
By Lemma \ref{lemma:ld}, there exists $\underline{\alpha}(\pi)>0$ such that $\alpha^*(\pi) \geq \underline{\alpha}(\pi)$. To ease the notation, we abbreviate $\underline{\alpha}(\pi),\alpha^*(\pi)$ as $\underline{\alpha},\alpha^*$.
Since $\hat{W}_n(\pi ,\alpha )$ is Lipschitz continuous over the set $\alpha \in [\underline{\alpha}/2,2\overline{\alpha}]$, we have
\begin{equation}
\sqrt{n}\left( \hat{W}_{n}(\pi ,\cdot )-\mathbf{E}[W_{i}(\pi ,\cdot
)]\right) \Rightarrow Z_1\left( \cdot \right) ,
\label{eqn:uniform_conv_alpha}
\end{equation}
uniformly in Banach space $\mathcal{C}([\underline{\alpha}/2,2\overline{\alpha}])$ of
continuous functions $\psi :[\underline{\alpha}/2,2\overline{\alpha}]%
\mathcal{\ \rightarrow }\reals$ equipped with the the sup-norm $%
\left\Vert \psi \right\Vert :=\sup_{x\in \lbrack \underline{\alpha}/2,2%
\overline{\alpha} ]}\psi (x)$ \citep[e.g.][Corollary 7.17]{araujo1980central}. $Z_1$ is a random element in $\mathcal{C}([%
\underline{\alpha}/2,2\overline{\alpha}])$.

Define the functionals
\begin{equation*}
G_1(\psi,\alpha )=\alpha \log \left( \psi (\alpha )\right) +\alpha \delta, \text{ and } V_1(\psi )=\inf_{\alpha \in \lbrack \underline{\alpha }/2,2\overline{\alpha }%
]}G_1(\psi,\alpha ),
\end{equation*}%
for $\psi>0$. By the Danskin theorem \citep[Theorem 4.13]{bonnansperturbation}, $%
V_1\left( \cdot \right) $ is directionally differentiable at any $\mu \in
\mathcal{C}([\underline{\alpha }/2,2\overline{\alpha }])$ with $\mu >0$ and
\begin{equation*}
V_{1,\mu }^{\prime }\left( \nu \right) =\inf_{\alpha \in \bar{X}\left( \mu
\right) }\alpha \left( 1/\mu (\alpha )\right) \nu (\alpha ),\text{ }\forall
\nu \in \mathcal{C}([\underline{\alpha }/2,2\overline{\alpha }])\mathcal{)}%
\text{,}
\end{equation*}%
where $\bar{X}\left( \mu \right) =\arg \min_{\alpha \in \lbrack \underline{%
\alpha }/2,2\overline{\alpha }])}\alpha \log \left( \mu (\alpha)\right)
+\alpha \delta $ and $V_{1,\mu }^{\prime }\left( \nu \right) $ is the
directional derivative of $V_1\left( \cdot \right) $ at $\mu $ in the
direction of $\nu .$ On the other hand, $V_1(\psi )$ is Lipschitz continuous
if $\psi \left( \cdot \right) $ is bounded away from zero. Notice that
\begin{equation}
\mathbf{E}[W_{i}(\pi ,\alpha )]=\mathbf{E}[\exp \left( -Y(\pi (X))/\alpha
\right) ]\geq \exp \left( -2M/\underline{\alpha }\right).
\label{E_ld}
\end{equation}%
Therefore, $V_1\left( \cdot \right) $ is Hadamard directionally differentiable
at $\mu =\mathbf{E}[W_{i}(\pi ,\cdot )]$ (see, for example, Proposition 7.57
in \cite{shapiro2009lectures}). By the Delta theorem (Theorem 7.59 in \cite%
{shapiro2009lectures}), we have
\begin{equation*}
\sqrt{n}\left( V_1(\hat{W}_{n}(\pi ,\cdot ))-V_1(\mathbf{E}[W_{i}(\pi ,\cdot
)])\right) \Rightarrow V_{1,\mathbf{E}[W_{i}(\pi ,\cdot )]}^{\prime }\left(
Z_1\right) .
\end{equation*}%
Furthermore, we know that $\log \left( \mathbf{E}\left( \exp \left( -\beta
Y\right) \right) \right) $ is strictly convex w.r.t $\beta $ given $\mathbf{%
Var}\left( Y\right) >0$ and $xf(1/x)$ is strictly convex if $f(x)$ is
strictly convex. Therefore, $\alpha \log \left( \mathbf{E}[W_{i}(\pi ,\alpha
)]\right) +\alpha \delta $ is strictly convex for $\alpha >0$ and thus
\begin{eqnarray*}
&&V_{1,\mathbf{E}[W_{i}(\pi ,\cdot )]}^{\prime }\left( Z_1\right) =\alpha ^{\ast
}\left( 1/\mathbf{E}[W_{i}(\pi ,\alpha ^{\ast })]\right) Z_1\left( \alpha
^{\ast }\right) \\
&\overset{d}{=}&\mathcal{N}\left( 0,\frac{\left( \alpha ^{\ast }\right)
^{2}}{ \mathbf{E}\left[ W_{i}(\pi ,\alpha ^{\ast })\right] ^{2}}\mathbf{E}\left[ \frac{1}{\pi _{0}\left( \pi (X)|X\right) }%
\left( \exp \left( -Y(\pi (X))/\alpha \right) -\mathbf{E}\left[ \exp
\left( -Y(\pi (X))/\alpha \right) \right] \right) ^{2} \right] \right),
\end{eqnarray*}
where $\overset{d}{=}$ denotes equality in distribution.
By Lemma \ref{thm:strong_duality}, we have that
\begin{equation*}
\hat{Q}_{\mathrm{\rm DRO}}(\pi )=-\inf_{\alpha \geq 0}\left( \alpha \log \left(
\hat{W}_{n}(\pi ,\alpha )\right) +\alpha \delta \right) \text{,}
\end{equation*}
and
\begin{equation*}
Q_{\mathrm{\rm DRO}}(\pi )=-\inf_{\alpha \geq 0}\left( \alpha \log \left(
\mathbf{E}[W_{i}(\pi ,\alpha )]\right) +\alpha \delta \right) =-V_1(\mathbf{E}
[W_{i}(\pi ,\alpha )]).
\end{equation*}

We remain to show $\mathbf{P}$($\hat{Q}_{\mathrm{\rm DRO}}(\pi )\neq -V_1(\hat{W}%
_{n}(\pi ,\alpha )))\rightarrow 0$, as $n\rightarrow \infty .$  The weak convergence \eqref{eqn:uniform_conv_alpha} also implies the uniform convergence,%
\begin{equation*}
\sup_{\alpha \in \lbrack \underline{\alpha }/2,2\overline{\alpha }%
])}\left\vert \hat{W}_{n}(\pi ,\alpha )-\mathbf{E}[W_{i}(\pi ,\alpha
)]\right\vert \rightarrow 0\text{ a.s..}
\end{equation*}%
Therefore, we further have
\begin{equation*}
\sup_{\alpha \in \lbrack \underline{\alpha }/2,2\overline{\alpha }%
])}\left\vert \left( \alpha \log \left( \hat{W}_{n}(\pi ,\alpha )\right)
+\alpha \delta \right) -\left( \alpha \log \left( \mathbf{E}[W_{i}(\pi
,\alpha )]\right) +\alpha \delta \right) \right\vert \rightarrow 0\text{ a.s.%
}
\end{equation*}%
given $\mathbf{E}[W_{i}(\pi ,\alpha )]$ is bounded away from zero in \eqref{E_ld}.  Let
\begin{equation*}
\epsilon =\min \left\{ \underline{\alpha }/2\log \left( \mathbf{E}%
[W_{i}(\pi ,\underline{\alpha }/2)]\right) +\underline{\alpha }\delta /2,2%
\overline{\alpha }\log \left( \mathbf{E}[W_{i}(\pi ,2\overline{\alpha }%
)]\right) +2\overline{\alpha }\delta \right\} -\left( \alpha ^{\ast }\log
\left( \mathbf{E}[W_{i}(\pi ,\alpha ^{\ast })]\right) +\alpha ^{\ast }\delta
\right)>0 .
\end{equation*}%
Then, given the event
\begin{equation*}
\left\{ \sup_{\alpha \in \lbrack \underline{\alpha }/2,2\overline{\alpha }%
])}\left\vert \left( \alpha \log \left( \hat{W}_{n}(\pi ,\alpha )\right)
+\alpha \delta \right) -\left( \alpha \log \left( \mathbf{E}[W_{i}(\pi
,\alpha )]\right) +\alpha \delta \right) \right\vert <\epsilon /2\right\} ,
\end{equation*}%
we have
\begin{equation*}
\alpha ^{\ast }\log \left( \hat{W}_{n}(\pi ,\alpha )\right) +\alpha ^{\ast
}\delta <\min \left\{ \underline{\alpha }/2\log \left( \hat{W}_{n}(\pi ,%
\underline{\alpha }/2)\right) +\underline{\alpha }\delta /2,2\overline{%
\alpha }\log \left( \hat{W}_{n}(\pi ,2\overline{\alpha })\right) +2\overline{%
\alpha }\delta \right\} ,
\end{equation*}%
which means $\hat{Q}_{\mathrm{\rm DRO}}(\pi )=-V_1(\hat{W}_{n}(\pi ,\alpha ))$ by
the convexity of $\alpha \log \left( \hat{W}_{n}(\pi ,\alpha )\right)
+\alpha \delta .$

Finally, we complete the proof by Slutsky's lemma \citep[e.g.][Theorem 1.8.10]{lehmann2006theory}).:%
\begin{eqnarray*}
\sqrt{n}\left( \hat{Q}_{\mathrm{\rm DRO}}(\pi )-Q_{\mathrm{\rm DRO}}(\pi )\right)
&=&\sqrt{n}\left( \hat{Q}_{\mathrm{\rm DRO}}(\pi )+V_1(\hat{W}_{n}(\pi ,\alpha
))\right) +\sqrt{n}\left( V_1(\mathbf{E}[W_{i}(\pi ,\alpha )])-V_1(\hat{W}%
_{n}(\pi ,\alpha ))\right)  \\
&\Rightarrow &0+\mathcal{N}\left( 0,\sigma ^{2}(\alpha^* ) \right)
\overset{d}{=}\mathcal{N}\left( 0,\sigma ^{2}(\alpha^* ) \right),
\end{eqnarray*}
where \[\sigma ^{2}(\alpha )=\frac{ \alpha^{2} }{\E\left[ W_i(\pi, \alpha)\right]^2} \mathbf{E}\left[ \frac{1}{\pi _{0}\left( \pi
(X)|X\right) } \left( \exp \left( -Y(\pi (X))/\alpha \right)
-\mathbf{E}\left[ \exp \left( -Y(\pi (X))/\alpha \right) \right] \right)
^{2} \right]. \]

(2) Discrete case, i.e., Assumption 2.2 is satisfied: {First, if $\mathbf{Var}(Y(\pi(X)))=0$, we have $\hat{Q}_{\rm DRO}(\pi) = {Q}_{\rm DRO}(\pi) = Y$ almost surely with $\alpha^*(\pi) = 0$. Therefore, in the following, we focus on the case $\mathbf{Var}(Y(\pi(X)))>0$}. Without loss of generality, we
assume the smallest element in the set $\mathbb{D}$ is zero since we can
always translate $Y$. Note that%
\begin{equation*}
\sqrt{n}\left( \hat{W}_{n}(\pi ,\alpha )-\mathbf{E}[W_{i}(\pi ,\alpha
)]\right) \Rightarrow Z_1\left( \alpha \right) ,
\end{equation*}
for $\alpha >0$, 
and
\begin{equation*}
\sqrt{n}\left( \hat{W}_{n}(\pi ,0 )-\mathbf{E}[W_{i}(\pi ,0
)]\right) \Rightarrow  \mathcal{N}\left( 0,\mathbf{E}\left[ \frac{1}{\pi _{0}\left( \pi
(X)\mid X\right) } \left( 1 - \P[Y=0]\right)
^{2} \right] \right).
\end{equation*}
Further, in the discrete case, we have
\begin{equation}
\mathbf{E}[W_{i}(\pi ,\alpha )]=\mathbf{E}[\exp \left( -Y(\pi (X))/\alpha
\right) ]\geq \underline{b} \text{  for } \alpha \geq 0. 
\label{E_ld2}
\end{equation}%

Since $\mathbf{Var}(Y)>0$, the proof of continuous case shows that $\phi(\pi,\alpha)$ is strictly concave for $\alpha >0$. 
Then, $\phi(\pi,\alpha)$ has a unique maximizer in $[0,\overline{\alpha}]$.
The remaining proof is the same as the continuous case.
\endproof
\subsection{Proof of the statistical performance guarantee in Section \ref{sec:uniform_convergence}}
\label{sec:proof_uniform}
\begin{proof}[Proof of Lemma \ref{lma:discrete_key}]
For any probability measure $\mathbf{P}_{1}$ supported on $\mathbb{D}$%
, we have
\[
\sup_{\alpha \geq 0}\left\{ -\alpha \log \mathbf{E}_{\mathbf{P}_{1}}\left[
\exp \left( -Y/\alpha \right) \right] -\alpha \delta \right\} =\sup_{\alpha
\geq 0}\left\{ -\alpha \log \left( \sum_{d\in \mathbb{D}}\left[ \exp \left(
-d/\alpha \right) \mathbf{P}_{1}\left( Y=d\right) \right] \right) -\alpha
\delta \right\} .
\]%
Therefore, we have
\begin{eqnarray}
&&\left\vert \sup_{\alpha \geq 0}\left\{ -\alpha \log \mathbf{E}_{\mathbf{P}%
_{1}}\left[ \exp \left( -Y/\alpha \right) \right] -\alpha \delta \right\}
-\sup_{\alpha \geq 0}\left\{ -\alpha \log \mathbf{E}_{\mathbf{P}_{2}}\left[
\exp \left( -Y/\alpha \right) \right] -\alpha \delta \right\} \right\vert
\nonumber \\
&\leq &\sup_{\alpha \geq 0}\left\vert \alpha \log \left( \frac{\sum_{d\in
\mathbb{D}}\left[ \exp \left( -d/\alpha \right) \mathbf{P}_{1}\left(
Y=d\right) \right] }{\sum_{d\in \mathbb{D}}\left[ \exp \left( -d/\alpha
\right) \mathbf{P}_{2}\left( Y=d\right) \right] }\right) \right\vert
\label{eqn:discrete_log} \\
&=&\sup_{\alpha \geq 0}\left\vert \alpha \log \left( \frac{\sum_{d\in
\mathbb{D}}\left[ \exp \left( -d/\alpha \right) \left( \mathbf{P}_{1}\left(
Y=d\right) -\mathbf{P}_{2}\left( Y=d\right) \right) \right] }{\sum_{d\in
\mathbb{D}}\left[ \exp \left( -d/\alpha \right) \mathbf{P}_{2}\left(
Y=d\right) \right] }+1\right) \right\vert .  \nonumber
\end{eqnarray}%
Since we can always divide the denominator and nominator of (\ref%
{eqn:discrete_log}) by $\exp \left( -\min_{d\in \mathbb{D}}d/\alpha \right) ,
$ we can assume $\min_{d\in \mathbb{D}}d=0$ without loss of generality.
Therefore, we have
\[
\sum_{d\in \mathbb{D}}\left[ \exp \left( -d/\alpha \right) \mathbf{P}%
_{2}\left( Y=d\right) \right] =\mathbf{P}_{2}\left( Y=0\right) +\sum_{d\in
\mathbb{D},d\neq 0}\left[ \exp \left( -d/\alpha \right) \mathbf{P}_{2}\left(
Y=d\right) \right] \geq \underline{b}.
\]%
Then, if $\left\vert \sum_{d\in \mathbb{D}}\left[ \exp \left( -d/\alpha
\right) \left( \mathbf{P}_{1}\left( Y=d\right) -\mathbf{P}_{2}\left(
Y=d\right) \right) \right] \right\vert <\underline{b}/2,$ we have
\begin{eqnarray}
&&\sup_{\alpha \geq 0}\left\vert \alpha \log \left( \frac{\sum_{d\in \mathbb{%
D}}\left[ \exp \left( -d/\alpha \right) \left( \mathbf{P}_{1}\left(
Y=d\right) -\mathbf{P}_{2}\left( Y=d\right) \right) \right] }{\sum_{d\in
\mathbb{D}}\left[ \exp \left( -d/\alpha \right) \mathbf{P}_{2}\left(
Y=d\right) \right] }+1\right) \right\vert   \nonumber \\
&\leq &2\sup_{\alpha \geq 0}\left\{ \alpha \frac{\left\vert \sum_{d\in
\mathbb{D}}\left[ \exp \left( -d/\alpha \right) \left( \mathbf{P}_{1}\left(
Y=d\right) -\mathbf{P}_{2}\left( Y=d\right) \right) \right] \right\vert }{%
\sum_{d\in \mathbb{D}}\left[ \exp \left( -d/\alpha \right) \mathbf{P}%
_{2}\left( Y=d\right) \right] }\right\}   \nonumber \\
&\leq &\frac{2}{\underline{b}}\sup_{\alpha \geq 0}\left\{ \alpha \left\vert
\sum_{d\in \mathbb{D}}\left[ \exp \left( -d/\alpha \right) \left( \mathbf{P}%
_{1}\left( Y=d\right) -\mathbf{P}_{2}\left( Y=d\right) \right) \right]
\right\vert \right\} .  \label{eqn:bd_discrete_log}
\end{eqnarray}
Then, we turn to (\ref{eqn:bd_discrete_log}),
\begin{eqnarray*}
&&\left\vert \sum_{d\in \mathbb{D}}\left[ \exp \left( -d/\alpha \right)
\left( \mathbf{P}_{1}\left( Y=d\right) -\mathbf{P}_{2}\left( Y=d\right)
\right) \right] \right\vert  \\
&=&\left\vert \left( \sum_{\substack{d: \mathbf{P}_{1}\left( Y=d\right)
\\  \geq\mathbf{P}_{2}\left( Y=d\right) }}\left[ e^{-d/\alpha }\left( \mathbf{P}%
_{1}\left( Y=d\right) -\mathbf{P}_{2}\left( Y=d\right) \right) \right]
\right) -\right. \left. \left( \sum_{\substack{d: \mathbf{P}_{1}\left(
Y=d\right)  \\ <\mathbf{P}_{2}\left( Y=d\right) }}\left[ e^{-d/\alpha
}\left( \mathbf{P}_{2}\left( Y=d\right) -\mathbf{P}_{1}\left( Y=d\right)
\right) \right] \right) \right\vert ,
\end{eqnarray*}%
which can be bounded by
\begin{align*}
&\left( \sum_{\substack{d: \mathbf{P}_{1}\left( Y=d\right)  \\ \geq \mathbf{P}%
_{2}\left( Y=d\right) }}e^{-M/\alpha }\left( \mathbf{P}_{1}\left( Y=d\right)
-\mathbf{P}_{2}\left( Y=d\right) \right) \right) -\left( \sum_{\substack{
d:\mathbf{P}_{1}\left( Y=d\right)  \\ <\mathbf{P}_{2}\left( Y=d\right) }}%
\left( \mathbf{P}_{2}\left( Y=d\right) -\mathbf{P}_{1}\left( Y=d\right)
\right) \right)  \\
&\leq \left( \sum_{\substack{d: \mathbf{P}_{1}\left( Y=d\right)   \\
\geq\mathbf{P}_{2}\left( Y=d\right) }}\left[ e^{-d/\alpha }\left( \mathbf{P}%
_{1}\left( Y=d\right) -\mathbf{P}_{2}\left( Y=d\right) \right) \right]
\right) -\left( \sum_{\substack{d: \mathbf{P}_{1}\left( Y=d\right)  \\ <%
\mathbf{P}_{2}\left( Y=d\right) }}\left[ e^{-d/\alpha }\left( \mathbf{P}%
_{2}\left( Y=d\right) -\mathbf{P}_{1}\left( Y=d\right) \right) \right]
\right)  \\
&\leq \left( \sum_{\substack{d: \mathbf{P}_{1}\left( Y=d\right)  \\ \geq
\mathbf{P}_{2}\left( Y=d\right) }}\left( \mathbf{P}_{1}\left( Y=d\right) -%
\mathbf{P}_{2}\left( Y=d\right) \right) \right) -\left( \sum_{\substack{
d:\mathbf{P}_{1}\left( Y=d\right)  \\  <\mathbf{P}_{2}\left( Y=d\right) }}\left[
e^{-M/\alpha }\left( \mathbf{P}_{2}\left( Y=d\right) -\mathbf{P}_{1}\left(
Y=d\right) \right) \right] \right) .
\end{align*}
Further, we observe that
\begin{eqnarray*}
\mathrm{TV}(\mathbf{P}_{1},\mathbf{P}_{2}) &=&\sum_{d:\mathbf{P}_{1}\left( Y=d\right)
\geq \mathbf{P}_{2}\left( Y=d\right) }\left( \mathbf{P}_{1}\left( Y=d\right)
-\mathbf{P}_{2}\left( Y=d\right) \right)  \\
&=&\sum_{d:\mathbf{P}_{1}\left( Y=d\right) <\mathbf{P}_{2}\left( Y=d\right)
}\left( \mathbf{P}_{2}\left( Y=d\right) -\mathbf{P}_{1}\left( Y=d\right)
\right) ,
\end{eqnarray*}%
where $\mathrm{TV}(\mathbf{P}_{1},\mathbf{P}_{2})$ denotes the total variation
between $\mathbf{P}_{1}$ and $\mathbf{P}_{2}.$ Therefore, we have \
\begin{eqnarray*}
&&\sup_{\alpha \geq 0}\left\{ \alpha \left\vert \sum_{d\in \mathbb{D}}\left[
\exp \left( -d/\alpha \right) \left( \mathbf{P}_{1}\left( Y=d\right) -%
\mathbf{P}_{2}\left( Y=d\right) \right) \right] \right\vert \right\}  \\
&\leq &\sup_{\alpha \geq 0}\left\{ \alpha (1-\exp \left( -M/\alpha \right)
)\right\} \mathrm{TV}(\mathbf{P}_{1},\mathbf{P}_{2}).
\end{eqnarray*}%
And it is easy to verify that
\[
\sup_{\alpha \geq 0}\left\{ \alpha (1-\exp \left( -M/\alpha \right)
)\right\} \leq M.
\]%
By combining all the above together, we have when $\mathrm{TV}(\mathbf{P}_{1},\mathbf{%
P}_{2})<\underline{b}/2,$
\begin{eqnarray*}
&&\left\vert \sup_{\alpha \geq 0}\left\{ -\alpha \log \mathbf{E}_{\mathbf{P}%
_{1}}\left[ \exp \left( -Y/\alpha \right) \right] -\alpha \delta \right\}
-\sup_{\alpha \geq 0}\left\{ -\alpha \log \mathbf{E}_{\mathbf{P}_{2}}\left[
\exp \left( -Y/\alpha \right) \right] -\alpha \delta \right\} \right\vert  \\
&\leq &\frac{2M}{\underline{b}}\mathrm{TV}(\mathbf{P}_{1},\mathbf{P}_{2}).
\end{eqnarray*}
\end{proof}
We first give the proof of Lemma \ref{lma:quantile}.

\begin{proof}[Proof of Lemma \ref{lma:quantile}]
Notice that
\begin{eqnarray*}
&&\left\vert \sup_{\alpha \geq 0}\left\{ -\alpha \log \mathbf{E}_{\mathbf{P}%
_{1}}\left\{ \exp \left( -Y/\alpha \right) \right\} -\alpha \delta \right\}
-\sup_{\alpha \geq 0}\left\{ -\alpha \log \mathbf{E}_{\mathbf{P}_{2}}\left\{
\exp \left( -Y/\alpha \right) \right\} -\alpha \delta \right\} \right\vert
\\
&\leq &\sup_{\alpha \geq 0}\left\vert \alpha \log \mathbf{E}_{\mathbf{P}_{1}}%
\left[ \exp \left( -Y/\alpha \right) \right] -\alpha \log \mathbf{E}_{%
\mathbf{P}_{2}}\left[ \exp \left( -Y/\alpha \right) \right] \right\vert  \\
&=&\sup_{\alpha \geq 0}\left\vert \alpha \log \mathbf{E}_{\mathbf{P}_{U}}%
\left[ \exp \left( -q_{\mathbf{P}_{1}}\left( U\right) /\alpha \right) \right]
-\alpha \log \mathbf{E}_{\mathbf{P}_{U}}\left[ \exp \left( -q_{\mathbf{P}%
_{2}}\left( U\right) /\alpha \right) \right] \right\vert ,
\end{eqnarray*}%
where $\mathbf{P}_{U}\sim U([0,1])$ and the last equality is based on the
fact that $q_{\mathbf{P}}\left( U\right) \overset{d}{=}\mathbf{P.}$

Denote $T=\sup_{t\in \lbrack 0,1]}\left\vert q_{\mathbf{P}_{1}}\left(
t\right) -q_{\mathbf{P}_{2}}\left( t\right) \right\vert $ and we have%
\begin{eqnarray*}
&&\alpha \log \mathbf{E}_{\mathbf{P}_{U}}\left[ \exp \left( -q_{\mathbf{P}%
_{1}}\left( U\right) /\alpha \right) \right] -\alpha \log \left[ \mathbf{E}_{%
\mathbf{P}_{U}}\exp \left( -q_{\mathbf{P}_{2}}\left( U\right) /\alpha
\right) \right]  \\
&=&\alpha \log \left[ \int_{0}^{1}\exp \left( -q_{\mathbf{P}_{1}}\left(
u\right) /\alpha \right) {\rm d}u\right] -\alpha \log \left[ \int_{0}^{1}\exp
\left( -q_{\mathbf{P}_{2}}\left( u\right) /\alpha \right) {\rm d}u\right]  \\
&\leq &\alpha \log \left[ \int_{0}^{1}\exp \left( -q_{\mathbf{P}_{2}}\left(
u\right) /\alpha \right) \exp \left( T/\alpha \right) {\rm d}u\right] -\alpha \log %
\left[ \int_{0}^{1}\exp \left( -q_{\mathbf{P}_{2}}\left( u\right) /\alpha
\right) {\rm d}u\right]  \\
&=&T.
\end{eqnarray*}%
Similarly, we have
\begin{equation*}
\alpha \log \left[ \mathbf{E}_{\mathbf{P}_{U}}\exp \left( -q_{\mathbf{P}%
_{2}}\left( U\right) /\alpha \right) \right] -\alpha \log \left[\mathbf{E}_{%
\mathbf{P}_{U}} \exp \left( -q_{\mathbf{P}_{1}}\left( U\right) /\alpha
\right) \right] \leq T.
\end{equation*}%
The desired result then follows.

\end{proof}

\begin{lemma}
\label{lma:qtl_bd}
Suppose probability measures $\mathbf{P}_{1}$ and $\mathbf{P}_{2}$ supported
on $[0,M]$ and $\mathbf{P}_{1}$ has a positive density $f_{\mathbf{P}%
_{1}}\left( \cdot \right) $ with a lower bound $f_{\mathbf{P}_{1}}\geq
\underline{b}$ over the interval $[0,M].$ Then, we have
\begin{equation*}
\sup_{t\in \lbrack 0,1]}\left\vert q_{\mathbf{P}_{1}}\left( t\right) -q_{%
\mathbf{P}_{2}}\left( t\right) \right\vert \leq \frac{1}{\underline{b}}%
\sup_{x\in \lbrack 0,M]}\left\vert F_{\mathbf{P}_{1}}\left( x\right) -F_{%
\mathbf{P}_{2}}\left( x\right) \right\vert .
\end{equation*}
\end{lemma}
\begin{proof}[Proof]
For ease of notation, let $x_{1}=q_{\mathbf{P}_{1}}\left( t\right) $ and $%
x_{2}=q_{\mathbf{P}_{2}}\left( t\right) .$ Since the distribution is
right-continuous with left limits and $\mathbf{P}_{1}$ is continuous, we
have
\begin{equation*}
F_{\mathbf{P}_{2}}\left( x_{2}-\right) \leq t,F_{\mathbf{P}_{2}}\left(
x_{2}\right) \geq t,\text{ and }F_{\mathbf{P}_{1}}\left( x_{1}\right) =t.
\end{equation*}%
If $x_{1}\geq x_{2},$ we have%
\begin{eqnarray*}
x_{1}-x_{2} &\leq &\frac{1}{\underline{b}}\left( F_{\mathbf{P}_{1}}\left(
x_{1}\right) -F_{\mathbf{P}_{1}}\left( x_{2}\right) \right)  \\
&\leq &\frac{1}{\underline{b}}\left( \left( F_{\mathbf{P}_{1}}\left(
x_{1}\right) -F_{\mathbf{P}_{1}}\left( x_{2}\right) \right) +\left( F_{%
\mathbf{P}_{2}}\left( x_{2}\right) -F_{\mathbf{P}_{1}}\left( x_{1}\right)
\right) \right)  \\
&=&\frac{1}{\underline{b}}\left( F_{\mathbf{P}_{2}}\left( x_{2}\right) -F_{%
\mathbf{P}_{1}}\left( x_{2}\right) \right) .
\end{eqnarray*}%
If $x_{1}<x_{2},$ we construct a monotone increasing sequence $%
x^{(1)},x^{(2)},\ldots $ with $x^{(n)}\uparrow x_{2}.$ Since $\mathbf{P}_{1}$
is continuous, we have  $F_{\mathbf{P}_{1}}\left( x^{(n)}\right) \uparrow F_{%
\mathbf{P}_{1}}\left( x_{2}\right) .$ Then, notice that
\begin{eqnarray*}
x_{2}-x_{1} &\leq &\frac{1}{\underline{b}}\left( F_{\mathbf{P}_{1}}\left(
x_{2}\right) -F_{\mathbf{P}_{1}}\left( x_{1}\right) \right)  \\
&\leq &\frac{1}{\underline{b}}\left( \left( F_{\mathbf{P}_{1}}\left(
x_{2}\right) -F_{\mathbf{P}_{1}}\left( x_{1}\right) \right) +\left( F_{%
\mathbf{P}_{1}}\left( x_{1}\right) -F_{\mathbf{P}_{2}}\left( x^{(n)}\right)
\right) \right)  \\
&=&\frac{1}{\underline{b}}\left( F_{\mathbf{P}_{1}}\left( x_{2}\right) -F_{%
\mathbf{P}_{2}}\left( x^{(n)}\right) \right)  \\
&=&\lim_{n\rightarrow \infty }\frac{1}{\underline{b}}\left( F_{\mathbf{P}%
_{1}}\left( x^{(n)}\right) -F_{\mathbf{P}_{2}}\left( x^{(n)}\right) \right) .
\end{eqnarray*}%
Therefore, for every $t$, we have
\begin{equation*}
\left\vert q_{\mathbf{P}_{1}}\left( t\right) -q_{%
\mathbf{P}_{2}}\left( t\right) \right\vert=\left\vert x_{1}-x_{2}\right\vert \leq \frac{1}{\underline{b}}\sup_{x\in
\lbrack 0,M]}\left\vert F_{\mathbf{P}_{1}}\left( x\right) -F_{\mathbf{P}%
_{2}}\left( x\right) \right\vert .
\end{equation*}%
The desired results then follows.
\end{proof}

By utilizing Lemmas \ref{lma:quantile}, \ref{lma:discrete_key} and \ref{lma:qtl_bd}, we are ready to
prove Theorem \ref{DRO_Uniform}.

\begin{proof}[Proof of Theorem \ref{DRO_Uniform}]
Recall $R_{\rm DRO}(\hat{\pi}_{\rm DRO})=Q_{\rm DRO}(\pi ^{\ast}_{\rm DRO})-Q_{\rm DRO}(\hat{\pi}_{\rm DRO}).$ Then,
\begin{eqnarray*}
&&R_{\rm DRO}(\hat{\pi}_{\rm DRO}) \\
&=&Q_{\rm DRO}(\pi ^{\ast}_{\rm DRO})-\hat{Q}_{\rm DRO}(\hat{\pi}_{\rm DRO})+\hat{Q}_{\rm DRO}(\hat{\pi}_{\rm DRO}%
)-Q_{\rm DRO}(\hat{\pi}_{\rm DRO}) \\
&\leq &\left( Q_{\rm DRO}(\pi ^{\ast}_{\rm DRO})-\hat{Q}_{\rm DRO}(\pi ^{\ast }_{\rm DRO})\right)
+\left( \hat{Q}_{\rm DRO}(\hat{\pi}_{\rm DRO})-Q_{\rm DRO}(\hat{\pi}_{\rm DRO})\right) \\
&\leq &2\sup_{\pi \in \Pi }\left\vert \hat{Q}_{\rm DRO}(\pi )-Q_{\rm DRO}(\pi
)\right\vert.
\end{eqnarray*}%
Note that
\begin{equation}
\hat{Q}_{\rm DRO}(\pi )
=\sup_{\alpha \geq 0}\left\{ -\alpha \log \left[ \mathbf{E%
}_{\mathbf{\hat{P}}_{n}^\pi}\exp \left( -Y_{i}(\pi (X_{i}))\right) /\alpha )%
\right]  -\alpha \delta \right\}
.
\end{equation}%

Therefore, we have
\begin{eqnarray}
&&R_{\rm DRO}(\hat{\pi}_{\rm DRO})  \notag \\
&\leq &2\sup_{\pi \in \Pi }\left\vert \sup_{\alpha \geq 0}\left\{ -\alpha
\log \left[ \mathbf{E}_{\mathbf{\hat{P}}_{n}^\pi}\exp \left( -Y_{i}(\pi
(X_{i})\right) /\alpha )\right] -\alpha \delta \right\} -\sup_{\alpha \geq
0}\left\{ -\alpha \log \left[ \mathbf{E}_{\mathbf{P}_{0}}\exp \left( -Y(\pi
(X)\right) /\alpha )\right] -\alpha \delta \right\} \right\vert
\label{uniform_bd_1}
\end{eqnarray}

We then divide the proof into two parts: continuous case and discrete case.

(1) Continuous case, i.e., Assumption \ref{assump:pos_dens}.1 is satisfied: By Lemmas \ref{lma:quantile}, \ref{lma:qtl_bd} and
Assumption \ref{assump:classical}, we have
\begin{eqnarray}
&&\sup_{\pi \in \Pi }\left\vert \sup_{\alpha \geq 0}\left\{ -\alpha \log
\left[ \mathbf{E}_{\mathbf{\hat{P}}_{n}^\pi}\exp \left( -Y_{i}(\pi
(X_{i}))\right) /\alpha )\right] -\alpha \delta \right\} -\sup_{\alpha \geq
0}\left\{ -\alpha \log \left[ \mathbf{E}_{\mathbf{P}_{0}}\exp \left( -Y(\pi
(X))\right) /\alpha )\right] -\alpha \delta \right\} \right\vert  \notag \\
&\leq &\sup_{\pi \in \Pi }\sup_{x\in \lbrack 0,M]}\frac{1}{\underline{b}}%
\left\vert \mathbf{E}_{\mathbf{\hat{P}}_{n}^\pi}\left[ \mathbf{1}\left\{
Y_{i}(\pi (X_{i}))\leq x\right\} \right] -\mathbf{E}_{_{\mathbf{P}_{0}}}\left[
\mathbf{1}\left\{ Y_{i}(\pi (X_{i}))\leq x\right\} \right] \right\vert  \notag
\\
&=&\sup_{\pi \in \Pi ,x\in \lbrack 0,M]}\frac{1}{\underline{b}}\left\vert
\left( \frac{1}{nS_{n}^{\pi }}\sum_{i=1}^{n}\frac{\mathbf{1\{}\pi
(X_{i}))=A_{i}\mathbf{\}}}{\pi _{0}\left( A_{i}|X_{i}\right) }\mathbf{1}%
\left\{ Y_{i}(\pi (X_{i}))\leq x\right\} \right) -\mathbf{E}_{_{\mathbf{P}%
_{0}}}\left[ \mathbf{1}\left\{ Y(\pi (X))\leq x\right\} \right] \right\vert
\notag \\
&\leq &\sup_{\pi \in \Pi ,x\in \lbrack 0,M]}\frac{1}{\underline{b}}%
\left\vert \left( \frac{1}{n}\sum_{i=1}^{n}\frac{\mathbf{1\{}\pi
(X_{i}))=A_{i}\mathbf{\}}}{\pi _{0}\left( A_{i}|X_{i}\right) }\mathbf{1}%
\left\{ Y_{i}(\pi (X_{i}))\leq x\right\} \right) -\mathbf{E}_{_{\mathbf{P}%
_{0}}}\left[ \frac{\mathbf{1\{}\pi (X)=A_{i}\mathbf{\}}}{\pi _{0}\left(
A|X\right) }\mathbf{1}\left\{ Y(\pi (X))\leq x\right\} \right] \right\vert
\label{uniform_bd_1_1} \\
&&+\sup_{\pi \in \Pi ,x\in \lbrack 0,M]}\frac{1}{\underline{b}}\left\vert
\left( \frac{S_{n}^{\pi }-1}{nS_{n}^{\pi }}\sum_{i=1}^{n}\frac{\mathbf{1\{}%
\pi (X_{i})=A_{i}\mathbf{\}}}{\pi _{0}\left( A_{i}|X_{i}\right) }\mathbf{1}%
\left\{ Y(\pi (X_{i}))\leq x\right\} \right) \right\vert .
\label{uniform_bd_1_2}
\end{eqnarray}%
For \eqref{uniform_bd_1_1}, by \citet[Theorem 4.10]{wainwright2019high}, we have with
probability $1-\exp \left( -n\epsilon ^{2}\eta ^{2}/2\right) $
\begin{eqnarray}
&&\sup_{\pi \in \Pi ,x\in \lbrack 0,M]}\frac{1}{\underline{b}}\left\vert
\left( \frac{1}{n}\sum_{i=1}^{n}\frac{\mathbf{1\{}\pi (X_{i})=A_{i}\mathbf{\}%
}}{\pi _{0}\left( A_{i}|X_{i}\right) }\mathbf{1}\left\{ Y_{i}(\pi
(X_{i})\leq x\right\} \right) -\mathbf{E}_{_{\mathbf{P}_{0}}}\left[ \frac{%
\mathbf{1\{}\pi (X)=A\mathbf{\}}}{\pi _{0}\left( A|X\right) }\mathbf{1}%
\left\{ Y(\pi (X)\leq x\right\} \right] \right\vert  \notag \\
&\leq &\frac{1}{\underline{b}}\left( 2\mathcal{R}_{n}\left( \mathcal{F}_{\Pi
,x}\right)  +\epsilon \right) ,  \label{final_2}
\end{eqnarray}%
where the function class $\mathcal{F}_{\Pi ,x}$ is defined as \[\mathcal{F}_{\Pi ,x}\triangleq\left\{ f_{\pi ,x}(X,Y,A)=%
\left. \frac{\mathbf{1\{}\pi (X)=A\mathbf{\}1\{}Y(\pi (X))\leq x\}}{\pi_0(A|X)}\right|\pi \in \Pi ,x\in
\lbrack 0,M]\right\} .\]For \eqref{uniform_bd_1_2}, we have
\begin{equation}
\frac{1}{\underline{b}}\left\vert \left( \frac{S_{n}^{\pi }-1}{nS_{n}^{\pi }}%
\sum_{i=1}^{n}\frac{\mathbf{1\{}\pi (X_{i})=A_{i}\mathbf{\}}}{\pi _{0}\left(
A_{i}|X_{i}\right) }\mathbf{1}\left\{ Y(\pi (X_{i})\leq x\right\} \right)
\right\vert \leq \frac{1}{\underline{b}}\left\vert S_{n}^{\pi }%
-1\right\vert,  \label{final_3}
\end{equation}
Further, by \citet[Theorem 4.10]{wainwright2019high} and the fact $%
\frac{\mathbf{1}}{\pi_0(A|X)}\leq \frac{1}{\eta }$ and $\mathbf{E}\left[ \frac{\mathbf{1}\{\pi
(X)=A\}}{\pi_0(A|X)}\right] =1$, we have with probability
at least $1-\exp \left( -{n\epsilon ^{2}\eta ^{2}}/{2}\right) ,$
\begin{equation}
\sup_{\pi \in \Pi }\left\vert S_n^{\pi}-1\right\vert
\leq 2\mathcal{R}_{n}\left( \mathcal{F}_{\Pi }\right)
+\epsilon ,
\label{eq:final_4}
\end{equation}%
where $\mathcal{F}_{\Pi }\triangleq\left \{\left. f_{\pi }(X,A)=\frac{\mathbf{1}\{\pi
(X)=A\}}{\pi_0(A|X)}\right|\pi \in \Pi \right\}.$

By combining \eqref{final_2}, \eqref{final_3}, and \eqref{eq:final_4}, we have with probability at least $1-2\exp \left( -%
n\epsilon ^{2}\eta ^{2}/2\right) $,
\begin{equation}
R_{\rm DRO}(\hat{\pi}_{\rm DRO})\leq \frac{2}{\underline{b}}\left( 2\mathcal{R}_{n}\left(
\mathcal{F}_{\Pi ,x}\right)  +\epsilon \right) + \frac{2}{%
\underline{b}} \left( 2\mathcal{R}_{n}\left(
\mathcal{F}_{\Pi }\right)  +\epsilon \right) .  \notag
\end{equation}%
Now, we turn to the Rademacher complexity of the classes $\mathcal{F}_{\Pi ,x}$ and $\mathcal{F}_{\Pi}$. First, for the class $\mathcal{F}_{\Pi}$, notice that
\begin{eqnarray*}
&&\sqrt{\sum_{i=1}^{n}\frac{1}{n}\left( \frac{\mathbf{1}\{\pi _{1}(X_{i})=A_{i}\}}{\pi_0(A_i|X_i)}%
-\frac{\mathbf{1}\{\pi _{2}(X_{i})=A_{i}\}}{\pi_0(A_i|X_i)}\right) ^{2}} \\
&=&\sqrt{\frac{1}{n}\sum_{i=1}^{n}\left(\frac{\mathbf{1\{}\pi _{1}(X_{i})\neq
\pi _{2}(X_{i})\mathbf{\}}}{\pi_0(A_i|X_i)} \right) ^{2}} \\
&\leq& \frac{1}{\eta} \sqrt{H\left( \pi _{1},\pi _{2}\right) }.
\end{eqnarray*}%
Therefore, the covering number
\begin{equation*}
N\left( t,\mathcal{F}_{\Pi },\left\Vert \cdot \right\Vert _{\P_{n}}\right)
\leq N\left( t,\Pi ,\sqrt{H\left( \cdot ,\cdot \right) }/\eta\right)
=N_{H}^{\left( n\right) }\left(\eta^2 t^{2},\Pi,\{x_1,\ldots,x_n \} \right) \leq  N_{H}^{\left( n\right) }\left( \eta^2t^{2},\Pi \right).
\end{equation*}%
For the class $\mathcal{F}_{\Pi ,x}$,  we claim \[N\left( t,\mathcal{F}_{\Pi,x },\left\Vert \cdot \right\Vert
_{\P_{n}}\right) \leq N_{H}^{\left( n\right) }\left(\eta^2 t^{2}/2,\Pi \right)
\sup_{\P}N\left(\eta t/\sqrt{2},\mathcal{F}_{I},\left\Vert \cdot \right\Vert
_{\P}\right) ,\]
 where $\mathcal{F}_{I}=\{f(t)\triangleq\mathbf{1}\left\{ t\leq
x\right\} |x\in \lbrack 0,M]\}.$ For ease of notation,  let
\[N_\Pi(t)=N_{H}^{(n)}\left(\eta^2 t^{2}/2,\Pi \right), \text{ and }N_I(t)=\sup_{\P}N\left( \eta t/\sqrt{2},\mathcal{F}_{I},\left\Vert \cdot \right\Vert
_{\P}\right).
\]
Suppose $\{\pi _{1},\pi _{2},\ldots ,\pi
_{N_\Pi(t) }\}$ is a cover for $%
\Pi $ and $\left\{ \mathbf{1}^{\pi }\left\{ t\leq x_{1}\right\} ,\mathbf{1}%
^{\pi }\left\{ t\leq x_{2}\right\} ,\ldots ,\mathbf{1}^{\pi }\left\{ t\leq
x_{N_I(t)}\right\} \right\} $ is a cover for $\mathcal{F}_{I}$ under the
distance $\left\Vert \cdot \right\Vert _{\hat{\P}_{\pi }},$ defined by%
\begin{equation*}
\hat{\P}_{\pi }\triangleq\frac{1}{n}\sum_{i=1}^n \Delta \left\{ Y_{i}(\pi (X_{i}))\right\} .
\end{equation*}%
Then, we claim that $\mathcal{F}_{\Pi ,x}^{t}$ is a $t$-cover set for $\mathcal{F%
}_{\Pi ,x},$ where $\mathcal{F}_{\Pi ,x}^{t}$ is defined as
\begin{equation*}
\mathcal{F}_{\Pi ,x}^{t}\triangleq\left\{\left.\frac{\mathbf{1\{}\pi _{i}(X)=A\}\mathbf{1}^{\pi_i}\{Y(\pi
_{i}(X))\leq x_{j}\}}{\pi_0(A|X)}\right|i\leq N_\Pi(t)
 ,j\leq N_I(t) \right\}.
\end{equation*}%
For $f_{\pi ,x}(X,Y,A)\in \mathcal{F}_{\Pi ,x},$ we can pick $\tilde{\pi},\tilde{x%
}$ such that $f_{\tilde{\pi},\tilde{x}}(X,Y,A)\in \mathcal{F}_{\Pi ,x}^{t}$,
$H\left( \pi ,\tilde{\pi}\right) \leq t^{2}/2$ and $\left\Vert \mathbf{1}%
\left\{ Y\leq x\right\} -\mathbf{1}\left\{ Y\leq \tilde{x}\right\}
\right\Vert _{\hat{P}_{\tilde{\pi}}}\leq t/\sqrt{2}.$ Then, we have
\begin{eqnarray*}
&&\sqrt{\sum_{i=1}^{n}\frac{1}{n}\left(\frac{ \mathbf{1\{}\pi (X_{i})=A_{i}\mathbf{%
\}1\{}Y(\pi (X_{i}))\leq x_{1}\}}{\pi_0(A_i|X_i)}-\frac{\mathbf{1\{}\tilde{\pi}(X_{i})=A_{i}\mathbf{%
\}1\{}Y(\tilde{\pi}(X_{i}))\leq \tilde{x}\}}{\pi_0(A_i|X_i)}\right) ^{2}} \\
&\leq &\frac{1}{\eta}\sqrt{H\left( \pi ,\tilde{\pi}\right) +\frac{1}{n}\sum_{i=1}^{n}%
\mathbf{1\{}\pi (X_{i})=\tilde{\pi}(X_{i})\mathbf{\}}\left( \mathbf{1\{}Y(%
\tilde{\pi}(X_{i}))\leq x\}-\mathbf{1\{}Y(\tilde{\pi}(X_{i}))\leq \tilde{x}%
\}\right) ^{2}} \\
&\leq &\frac{1}{\eta}\sqrt{\eta^2t^{2}/2+\eta^2t^{2}/2}=t.
\end{eqnarray*}%
From Lemma 19.15 and Example 19.16 in \cite{van2000asymptotic}, we know%
\begin{equation*}
\sup_{\P}N\left( t,\mathcal{F}_{I},\left\Vert \cdot \right\Vert _{\P}\right)
\leq K\left( \frac{1}{t }\right) ^{2},
\end{equation*}%
where $K$ is a universal constant. Finally, by Dudley's theorem \citep[e.g.][(5.48)]{wainwright2019high}, we have
\begin{equation*}
\mathcal{R}_{n}\left( \mathcal{F}_{\Pi }\right) \leq 24\mathbf{E} \left[
\int_{0}^{2/\eta}\sqrt{\frac{\log N\left( t,\mathcal{F}_{\Pi },\left\Vert \cdot
\right\Vert _{\P_{n}}\right) }{n}}{\rm d}t \right]\leq \frac{24\kappa ^{(n)}\left(
\Pi \right) }{\eta \sqrt{n}},
\end{equation*}%
and
\begin{eqnarray*}
\mathcal{R}_{n}\left( \mathcal{F}_{\Pi ,x}\right) &\leq &\frac{24}{\sqrt{n}}%
\int_{0}^{2/\eta}\sqrt{\log \left( N_{H}^{\left( n\right) }\left(\eta^2 t^{2}/2,\Pi
\right) \right) +\log \left( \sup_{\P}N\left(\eta t/\sqrt{2},\mathcal{F}%
_{I},\left\Vert \cdot \right\Vert _{\P}\right) \right) }{\rm d}t \\
&= &\frac{24}{\eta\sqrt{n}}%
\int_{0}^{2}\sqrt{\log \left( N_{H}^{\left( n\right) }\left(s^2/2,\Pi
\right) \right) +\log \left( \sup_{\P}N\left(s/\sqrt{2},\mathcal{F}%
_{I},\left\Vert \cdot \right\Vert _{\P}\right) \right) }{\rm d}s  \\
&\leq &\frac{24}{\eta\sqrt{n}}\int_{0}^{\sqrt{2}}\left( \sqrt{\log \left( N_{H}^{\left(
n\right) }\left( s^{2}/2,\Pi \right) \right) }+\sqrt{\log K}+\sqrt{4\log
\left( 1/s \right) }\right) {\rm d}s \\
&\leq &\frac{24\sqrt{2}\kappa ^{(n)}\left( \Pi \right) }{\eta\sqrt{n}}+C/\sqrt{n},
\end{eqnarray*}%
where $C$ is a universal constant. Therefore, by picking $\varepsilon' =2\exp\left(-n \epsilon^2 \eta^2 /2 \right),$ we have
with probability $1-\varepsilon'$,
\begin{eqnarray}
R_{\rm DRO}(\hat{\pi}_{\rm DRO})  &\leq& \frac{4}{\underline{b}\eta\sqrt{n}}\left( 24(\sqrt{2}+1)\kappa ^{(n)}\left( \Pi \right)  +\sqrt{2\log\left(\frac{2}{\varepsilon'}\right)}+C\right).
\end{eqnarray}

(2) Discrete case, i.e., Assumption \ref{assump:pos_dens}.2 is satisfied:
By Lemmas \ref{lma:discrete_key} and Assumption \ref{assump:classical}, when
$\sup_{\pi \in \Pi }\mathrm{TV}\left( \mathbf{\hat{P}}_{n}^{\pi },\mathbf{P}%
_{0}\right) \leq \underline{b}/2,$ we have
\begin{eqnarray*}
&&\sup_{\pi \in \Pi }\left\vert \sup_{\alpha \geq 0}\left\{ -\alpha \log
\left[ \mathbf{E}_{\mathbf{\hat{P}}_{n}^{\pi }}\exp \left( -Y_{i}(\pi
(X_{i}))\right) /\alpha )\right] -\alpha \delta \right\} -\sup_{\alpha \geq
0}\left\{ -\alpha \log \left[ \mathbf{E}_{\mathbf{P}_{0}}\exp \left( -Y(\pi
(X))\right) /\alpha )\right] -\alpha \delta \right\} \right\vert \\
&\leq &\sup_{\pi \in \Pi }\frac{2M}{\underline{b}}\mathrm{TV}\left( \mathbf{\hat{P}}%
_{n}^{\pi },\mathbf{P}_{0}\right) \\
&=&\sup_{\pi \in \Pi }\sup_{S\in \mathcal{S}_{\mathbb{D}}}\frac{2M}{%
\underline{b}}\left\vert \mathbf{\hat{P}}_{n}^{\pi }\left[ S\right] -%
\mathbf{P}_{0}\left[ S\right] \right\vert ,
\end{eqnarray*}%
where $\mathcal{S}_{\mathbb{D}}=\left\{ S:S\subset \mathbb{D}\right\} $
contains all subsets of $\mathbb{D}.$ Then
\begin{eqnarray}
&&\sup_{\pi \in \Pi }\sup_{S\in \mathcal{S}_{\mathbb{D}}}\frac{2M}{%
\underline{b}}\left\vert \mathbf{\hat{P}}_{n}^{\pi }\left[ S\right] -%
\mathbf{P}_{0}\left[ S\right] \right\vert  \notag \\
&=&\sup_{\pi \in \Pi ,S\in \mathcal{S}_{\mathbb{D}}}\frac{2M}{\underline{b}}%
\left( \frac{1}{nS_{n}^{\pi }}\sum_{i=1}^{n}\frac{\mathbf{1\{}\pi
(X_{i})=A_{i}\mathbf{\}1\{}Y_{i}\in S\mathbf{\}}}{\pi _{0}\left(
A_{i}|X_{i}\right) }-\mathbf{E}_{\mathbf{P}_{0}}\left[ \mathbf{1\{}Y_{i}\in S%
\mathbf{\}}\right] \right)  \notag \\
&\leq &\sup_{\pi \in \Pi ,S\in \mathcal{S}_{\mathbb{D}}}\frac{2M}{\underline{%
b}}\left( \frac{1}{n}\sum_{i=1}^{n}\frac{\mathbf{1\{}\pi (X_{i})=A_{i}%
\mathbf{\}1\{}Y_{i}\in S\mathbf{\}}}{\pi _{0}\left( A_{i}|X_{i}\right) }-%
\mathbf{E}_{\mathbf{P}_{0}}\left[ \mathbf{1\{}Y_{i}\in S\mathbf{\}}\right]
\right)  \label{eqn:bd_discrete_1} \\
&&+\sup_{\pi \in \Pi ,S\in \mathcal{S}_{\mathbb{D}}}\frac{2M}{\underline{b}}%
\left( \frac{S_{n}^{\pi }-1}{nS_{n}^{\pi }}\sum_{i=1}^{n}\frac{\mathbf{1\{}%
\pi (X_{i})=A_{i}\mathbf{\}1\{}Y_{i}\in S\mathbf{\}}}{\pi _{0}\left(
A_{i}|X_{i}\right) }\right)  \label{eqn:bd_discrete_2}
\end{eqnarray}
For (\ref{eqn:bd_discrete_1}), by \citet[Theorem 4.10]{wainwright2019high},
we have with probability $1-\exp \left( -n\epsilon ^{2}\eta ^{2}/2\right) $
\begin{eqnarray}
&&\sup_{\pi \in \Pi ,S\in \mathcal{S}_{\mathbb{D}}}\frac{2M}{\underline{b}}%
\left( \frac{1}{n}\sum_{i=1}^{n}\frac{\mathbf{1\{}\pi (X_{i})=A_{i}\mathbf{%
\}1\{}Y_{i}\in S\mathbf{\}}}{\pi _{0}\left( A_{i}|X_{i}\right) }-\mathbf{E}_{%
\mathbf{P}_{0}}\left[ \mathbf{1\{}Y_{i}\in S\mathbf{\}}\right] \right)
\notag \\
&\leq &\frac{2M}{\underline{b}}\left( 2\mathcal{R}_{n}\left( \mathcal{F}%
_{\Pi ,\mathbb{D}}\right) +\epsilon \right) ,  \label{final_discrete_2}
\end{eqnarray}%
where the function class $\mathcal{F}_{\Pi ,\mathbb{D}}$ is defined as
\begin{equation*}
\mathcal{F}_{\Pi ,\mathbb{D}}\triangleq \left\{ f_{\pi ,S}(X,Y,A)=\left.
\frac{\mathbf{1\{}\pi (X_{i})=A_{i}\mathbf{\}1\{}Y_{i}\in S\mathbf{\}}}{\pi
_{0}(A|X)}\right\vert \pi \in \Pi ,S\subset \mathbb{D}\right\} .
\end{equation*}%
For (\ref{eqn:bd_discrete_2}), we have
\begin{equation}
\frac{2M}{\underline{b}}\left( \frac{S_{n}^{\pi }-1}{nS_{n}^{\pi }}%
\sum_{i=1}^{n}\frac{\mathbf{1\{}\pi (X_{i})=A_{i}\mathbf{\}1\{}Y_{i}\in S%
\mathbf{\}}}{\pi _{0}\left( A_{i}|X_{i}\right) }\right) \leq \frac{2M}{%
\underline{b}}\left\vert S_{n}^{\pi }-1\right\vert .
\label{eqn:final_discrete_3}
\end{equation}

By combining \eqref{final_discrete_2}, \eqref{eqn:final_discrete_3} and %
\eqref{eq:final_4}, we have with probability at least $1-2\exp \left(
-n\epsilon ^{2}\eta ^{2}/2\right) $,
\begin{equation}
R_{\mathrm{DRO}}(\hat{\pi}_{\mathrm{DRO}})\leq \frac{2M}{\underline{b}}%
\left( 2\mathcal{R}_{n}\left( \mathcal{F}_{\Pi ,\mathbb{D}}\right) +\epsilon
\right) +\frac{2M}{\underline{b}}\left( 2\mathcal{R}_{n}\left( \mathcal{F}%
_{\Pi }\right) +\epsilon \right) .  \notag
\end{equation}%
Now, we turn to the Rademacher complexity of the classes $\mathcal{F}_{\Pi ,%
\mathbb{D}}.$ We claim
\begin{equation*}
N\left( t,\mathcal{F}_{\Pi ,\mathbb{D}},\left\Vert \cdot \right\Vert _{\P %
_{n}}\right) \leq N_{H}^{\left( n\right) }\left( \eta ^{2}t^{2}/2,\Pi
\right) 2^{|\mathbb{D}|},
\end{equation*}%
where $|\mathbb{D}|$ denotes the cardinality of set $\mathbb{D}$. By similar
argument with the continuous case, we have $\mathcal{F}_{\Pi ,\mathbb{D}}^{t}
$ is a $t$-cover set for $\mathcal{F}_{\Pi ,\mathbb{D}},$ where $\mathcal{F}%
_{\Pi ,\mathbb{D}}^{t}$ is defined as
\begin{equation*}
\mathcal{F}_{\Pi ,\mathbb{D}}^{t}\triangleq \left\{ \left. \frac{\mathbf{1\{}%
\pi _{i}(X)=A\}\mathbf{1\{}Y_{i}(\pi _{i}(X_{i}))\in S\mathbf{\}}}{\pi
_{0}(A|X)}\right\vert i\leq N_{\Pi }(t),S\subset \mathbb{D}\right\} .
\end{equation*}%
Finally, by Dudley's theorem \citep[e.g.][(5.48)]{wainwright2019high}
\begin{eqnarray*}
\mathcal{R}_{n}\left( \mathcal{F}_{\Pi ,\mathbb{D}}\right)  &\leq &\frac{24}{%
\sqrt{n}}\int_{0}^{2/\eta }\sqrt{\log \left( N_{H}^{\left( n\right) }\left(
\eta ^{2}t^{2}/2,\Pi \right) \right) +|\mathbb{D}|\log \left( 2\right) }%
\mathrm{d}t \\
&=&\frac{24}{\eta \sqrt{n}}\int_{0}^{2}\sqrt{\log \left( N_{H}^{\left(
n\right) }\left( s^{2}/2,\Pi \right) \right) +|\mathbb{D}|\log \left(
2\right) }\mathrm{d}s \\
&\leq &\frac{24}{\eta \sqrt{n}}\left( \int_{0}^{\sqrt{2}}\left( \sqrt{\log
\left( N_{H}^{\left( n\right) }\left( s^{2}/2,\Pi \right) \right) }\right)
\mathrm{d}s+2\sqrt{|\mathbb{D}|\log \left( 2\right) }\right)  \\
&\leq &\frac{24\sqrt{2}\kappa ^{(n)}\left( \Pi \right) +48\sqrt{|\mathbb{D}%
|\log \left( 2\right) }}{\eta \sqrt{n}}
\end{eqnarray*}
Therefore, by picking $\varepsilon ^{\prime }=2\exp \left( -n\epsilon
^{2}\eta ^{2}/2\right) ,$ we have with probability $1-\varepsilon ^{\prime }$%
,
\begin{equation*}
R_{\mathrm{DRO}}(\hat{\pi}_{\mathrm{DRO}})\leq \frac{4M}{\underline{b}\eta
\sqrt{n}}\left( 24(\sqrt{2}+1)\kappa ^{(n)}\left( \Pi \right) +48\sqrt{|%
\mathbb{D}|\log \left( 2\right) }+\sqrt{2\log \left( \frac{2}{\varepsilon
^{\prime }}\right) }\right) .
\end{equation*}
\end{proof}
\subsection{Proof of the statistical lower bound in Section \ref{sec:lower_bd}}
\label{sec:proof_ld}
We first define some useful notions. For $p,q\in \lbrack 0,1],$ Let
\[
D_{\mathrm{KL}}(p||q)=p\log (p/q)+(1-p)\log ((1-p)/(1-q)).
\]%
Let $g(p)=\inf_{D_{\mathrm{KL}}(p||q)\leq \delta }q.$

\begin{lemma}
\label{lma:g(p)}
For $\delta \leq 0.226,$ $g(p)$ is differentiable and $g^{\prime }(p)\geq 1/2
$ for $p\in \lbrack 0.4,0.6].$
\end{lemma}
\begin{proof}[Proof]
Since $\delta \leq 0.226$ and $p\geq 0.4,$ we have $p\geq g(p)\geq 0.1.$ By
\citet[Lemma B.12]{yang2021towards}, we have
\[
g^{\prime }(p)=\frac{\log (p/g(p))-\log ((1-p)/(1-g(p)))}{p/g(p)-(1-p)/(1-g(p))},
\]%
and $g^{\prime }(p)$ is increasing. Therefore, $g^{\prime }(p)\geq g^{\prime
}(0.4)\geq 0.5.$
\end{proof}
\begin{proof}[Proof of Theorem \ref{thm:ld}] Since we consider two-action scenario, we denote the
actions to be $0$ and $1.$ We first follow the lines in the proof   in \citet[Theorem
2.2]{kitagawa2018should}. By Lemma \ref{lemma:VC_bound}, the VC dimension of the policy class $\Pi $ is larger or equal than $v\triangleq \lceil 4/25 \kappa^{(n)}(\Pi)^2 \rceil$.
Let $x_{1},x_{2},\ldots ,x_{v}$ be $v$  points that are shattered by policy
$\Pi .$ Let $\mathbf{b}=\left\{ b_{1},b_{2},\ldots ,b_{v}\right\} \in
\left\{ 0,1\right\} ^{v}.$ By definition, for each $\mathbf{b}\in \left\{
0,1\right\} ^{v},$ there exist a $\pi \in \Pi $ such that $\pi (x_{i})=b_{i}$
for $i=1,2,\ldots ,v.$ We use $\pi =\left\{ \pi _{1},\pi _{2},\ldots ,\pi
_{v}\right\} \in \left\{ 0,1\right\} ^{v}$ to denote the policy $\pi $
restricted in $\left\{ x_{1},x_{2},\ldots ,x_{v}\right\} .$ Now, we consider
you two distributions
\[
Y_{0}=\left\{
\begin{array}{c}
M \\
0%
\end{array}%
\right.
\begin{array}{l}
\text{with prob. }1-p-\gamma  \\
\text{with prob. }p+\gamma
\end{array}%
,\text{ } Y_{1}=\left\{
\begin{array}{c}
M \\
0%
\end{array}%
\right.
\begin{array}{l}
\text{with prob. }1-p+\gamma  \\
\text{with prob. }p-\gamma
\end{array}%
,
\]%
for $p,\gamma >0,$ which will be determined later. Then, we construct\ $%
\mathbf{P}_{\mathbf{b}}*\pi_{\mathbf{b},0}\in \mathcal{P}(M)$ for every $\mathbf{b}\in \left\{
0,1\right\} ^{v}.$ The marginal distribution of $X$ in $\mathbf{P}_{\mathbf{b%
}}$ supports uniformly on $\left\{ x_{1},x_{2},\ldots ,x_{v}\right\} ,$ with
equal mass $1/v.$ Further,
\[
\pi_{\mathbf{b},0}(A=b_{i}|X=x_{i})=\pi_{\mathbf{b},0}(A=1-b_{i}|X=x_{i})=\frac{1}{2},
\]%
and conditional on $x_{i},$ $Y(b_{i})$ follows the distribution of $Y_{1}$
and $Y(1-b_{i})$ follows the distribution of $Y_{0}.$ We also define $d_{H}(%
\mathbf{b,b}^{\prime })=\sum_{i=1}^{v}\left\vert b_{i}-b_{i}^{\prime
}\right\vert .$

To emphasize the dependence on the underlying distrbution $\mathbf{P}_{%
\mathbf{b}},$ we rewrite $Q_{\mathrm{DRO}}(\pi )=Q_{\mathrm{DRO}}(\pi ,%
\mathbf{P}_{\mathbf{b}}).$ Now, by Lemma \ref{thm:strong_duality}, we have
\begin{eqnarray*}
\sup_{\pi \in \Pi }Q_{\mathrm{DRO}}(\pi ,\mathbf{P}_{\mathbf{b}})
&=&\sup_{\pi \in \Pi }\sup_{\alpha \geq 0}\left\{ -\alpha \log \left(
\mathbf{E}_{\mathbf{P}_{\mathbf{b}}}\left[ \exp (-Y(\pi (X))/\alpha )\right]
\right) -\alpha \delta \right\}  \\
&=&\sup_{\alpha \geq 0}\sup_{\pi \in \Pi }\left\{ -\alpha \log \left(
\mathbf{E}_{\mathbf{P}_{\mathbf{b}}}\left[ \exp (-Y(\pi (X))/\alpha )\right]
\right) -\alpha \delta \right\} .
\end{eqnarray*}%
Since for every $\alpha \geq 0,$ we have $\mathbf{E}\left[ \exp
(-Y_{1}/\alpha )\right] \leq \mathbf{E}\left[ \exp (-Y_{0}/\alpha )\right] ,$
we have the optimal policy for the distribution $\mathbf{P}_{\mathbf{b}}$ is
$\pi _{\mathbf{b}}^*=\mathbf{b}$. Further,
\begin{eqnarray*}
\sup_{\pi \in \Pi }Q_{\mathrm{DRO}}(\pi ,\mathbf{P}_{\mathbf{b}}) &=&Q_{%
\mathrm{DRO}}(\pi _{\mathbf{b}}^*,\mathbf{P}_{\mathbf{b}})=\sup_{\alpha \geq
0}\left\{ -\alpha \log \left( \mathbf{E}\left[ \exp (-Y_{1}/\alpha )\right]
\right) -\alpha \delta \right\}  \\
&=&\inf_{D(\mathbf{P}||Y_{1})\leq \delta }\mathbf{E}_{Y\sim \mathbf{P}}\left[
Y\right]  \\
&=&Mg(1-p+\gamma ).
\end{eqnarray*}
Then, for any $\pi \in $ $\Pi $ for any $\mathbf{P}_{\mathbf{b}},$ we have%
\begin{eqnarray*}
Q_{\mathrm{DRO}}(\pi ,\mathbf{P}_{\mathbf{b}}) &=&\sup_{\alpha \geq
0}\left\{ -\alpha \log \left( \frac{1}{v}\sum_{i=1}^{v}\exp (-Y(\pi
_{i}))/\alpha )\right) -\alpha \delta \right\}  \\
&=&\sup_{\alpha \geq 0}\left\{ -\alpha \log \left( \frac{1}{v}\left(
\sum_{i=1,\pi _{i}=b_{i}}^{v}\mathbf{E}\left[ \exp (-Y_{1}/\alpha )\right]
+\sum_{i=1,\pi _{i}\neq b_{i}}^{v}\mathbf{E}\left[ \exp (-Y_{0}/\alpha )%
\right] \right) \right) -\alpha \delta \right\}  \\
&=&\sup_{\alpha \geq 0}\left\{ -\alpha \log \left( \frac{v-d_{H}(\mathbf{b}%
,\pi )}{v}\mathbf{E}\left[ \exp (-Y_{1}/\alpha )\right] +\frac{d_{H}(\mathbf{%
b},\pi )}{v}\mathbf{E}\left[ \exp (-Y_{0}/\alpha )\right] \right) -\alpha
\delta \right\} .
\end{eqnarray*}

To simplify the notation, let $m=d_{H}(\mathbf{b},\pi ).$ Then, we have
\begin{eqnarray*}
&&Q_{\mathrm{DRO}}(\pi ,\mathbf{P}_{\mathbf{b}}) \\
&=&\sup_{\alpha \geq 0}\left\{ -\alpha \log \left( \frac{v-m}{v}\mathbf{E}%
\left[ \exp (-Y_{1}/\alpha )\right] +\frac{m}{v}\mathbf{E}\left[ \exp
(-Y_{0}/\alpha )\right] \right) -\alpha \delta \right\}  \\
&=&\sup_{\alpha \geq 0}\left\{ -\alpha \log \left( \left( \frac{v-m}{v}%
\left( 1-p+\gamma \right) +\frac{m}{v}\left( 1-p-\gamma \right) \right) \mathbf{E%
}\left[ \exp (-M/\alpha )\right] +\left( \frac{m}{v}\left( p+\gamma
\right) +\frac{v-m}{v}\left( p-\gamma \right) \right) \right) -\alpha
\delta \right\}  \\
&=&\sup_{\alpha \geq 0}\left\{ -\alpha \log \left( \left( 1-p+\frac{v-2m}{v}%
\gamma \right) \mathbf{E}\left[ \exp (-M/\alpha )\right] +\left( p-\frac{%
v-2m}{v}\gamma \right) \right) -\alpha \delta \right\}
\end{eqnarray*}%
We construct the distribution of $\tilde{Y}_{m}:$
\[
\tilde{Y}_{m}=\left\{
\begin{array}{c}
M \\
0%
\end{array}%
\right.
\begin{array}{l}
\text{with prob. }1-p+\frac{v-2m}{v}\gamma  \\
\text{with prob. }p-\frac{v-2m}{v}\gamma .%
\end{array}%
\]%
Then, $Q_{\mathrm{DRO}}(\pi ,\mathbf{P}_{\mathbf{b}})$ becomes
\[
Q_{\mathrm{DRO}}(\pi ,\mathbf{P}_{\mathbf{b}})=\inf_{D(\mathbf{P}||\tilde{Y}%
_{m})\leq \delta }\mathbf{E}_{Y\sim \mathbf{P}}\left[ Y\right] =Mg\left(
1-p+\gamma -\frac{2m}{v}\gamma \right) .
\]%
By Lemma \ref{lma:g(p)}, if $p=1/2$ and $\gamma \leq 0.1,$ we have
\begin{eqnarray*}
R_{\mathrm{DRO}}(\pi ,\mathbf{P}_{\mathbf{b}}) &=&Mg(1-p+\gamma )-Mg\left(
1-p+\gamma -\frac{2m}{v}\gamma \right)  \\
&\geq &\frac{2m}{v}M\gamma \min_{x\in \left[1- p-\gamma ,1-p+\gamma \right]
}g^{\prime }(x)\geq \frac{d_{H}(\mathbf{b},\pi )M\gamma }{v}.
\end{eqnarray*}%
%Let $\mathbf{P}^{O}$ be the distribution of $\left\{ X,A,Y(A)\right\} ,$
%where $\left\{ X,A,Y(0),Y(1)\right\} \sim \mathbf{P.}$
Then, by Assouad Lemma   \citep[Theorem 2.12
(ii)]{tsybakov2009introduction}, we have
\begin{align*}
&\max_{\mathbf{P}_{0} * \pi_{0}\in \mathcal{P}(M)}\mathbf{E}_{\left(\mathbf{P}^{\pi_0}_0\right)^n}%
\left[ R_{\mathrm{DRO}}(\pi ,\mathbf{P}_0)\right]  \\
\geq &\frac{M\gamma }{v}\max_{%
\mathbf{P}_{0}* \pi_{0}\in \mathcal{P}(M )}\mathbf{E}_{\left(\mathbf{P}^{\pi_0}_0\right)^n}\left[
d_{H}(\mathbf{b},\pi )\right]  \\
\geq &\frac{M\gamma }{2}\left( 1-\max_{d_{H}(%
\mathbf{b,b}^{\prime })=1}\mathrm{TV}\left( \left( \mathbf{P}_{\mathbf{b}}^{ \pi_{\mathbf{b},0}}\right)
^{n},\left( \mathbf{P}_{\mathbf{b}^{\prime }}^{ \pi_{\mathbf{b}',0}}\right) ^{n}\right) \right)
,
\end{align*}%
where $\mathrm{TV}(\mathbf{\cdot ,\cdot })$ denotes the total variation distance
between two measures. By Pinsker's inequality (\citep[Lemma 2.5]{tsybakov2009introduction}), we have
\[
\mathrm{TV}\left( \left( \mathbf{P}_{\mathbf{b}}^{ \pi_{\mathbf{b},0}}\right) ^{n},\left( \mathbf{P}_{%
\mathbf{b}^{\prime }}^{ \pi_{\mathbf{b}',0}}\right) ^{n}\right) \leq \sqrt{D\left( \left(
\mathbf{P}_{\mathbf{b}}^ {\pi_{\mathbf{b},0}}\right) ^{n}||\left( \mathbf{P}_{\mathbf{b}%
^{\prime }}^{ \pi_{\mathbf{b}',0}}\right) ^{n}\right) /2}=\sqrt{nD\left( \mathbf{P}_{\mathbf{b}%
}^{ \pi_{\mathbf{b},0}}||\mathbf{P}_{\mathbf{b}^{\prime }}^{ \pi_{\mathbf{b}',0}}\right) /2}.
\]%
For $\mathbf{b,b}^{\prime }$ such that $d_{H}(\mathbf{b,b}^{\prime })=1,$
Let $b_{l}\neq b_{l}^{\prime }$ and without loss of generality, we assume $%
b_{l}=1.$ Then, wave
\begin{eqnarray*}
D\left( \mathbf{P}_{\mathbf{b}%
}^{ \pi_{\mathbf{b},0}}||\mathbf{P}_{\mathbf{b}^{\prime }}^{ \pi_{\mathbf{b}',0}}\right)&=&\sum_{i=1}^{v}\sum_{j=0}^{1}\sum_{k=0}^{1}\mathbf{P}_{%
\mathbf{b}}^{ \pi_{\mathbf{b},0}}\left( X=x_{i},A=j,Y=Mk\right) \log \left( \frac{\mathbf{P}_{%
\mathbf{b}}^{ \pi_{\mathbf{b},0}}\left( X=x_{i},A=j,Y=Mk\right) }{\mathbf{P}_{\mathbf{b}^{\prime
}}^{ \pi_{\mathbf{b}',0}}\left( X=x_{i},A=j,Y=Mk\right) }\right)  \\
&=&\sum_{j=0}^{1}\sum_{k=0}^{1}\mathbf{P}_{\mathbf{b}}^{ \pi_{\mathbf{b},0}}\left(
X=x_{l},A=j,Y=Mk\right) \log \left( \frac{\mathbf{P}_{\mathbf{b}}^{ \pi_{\mathbf{b},0}}\left(
X=x_{l},A=j,Y=Mk\right) }{\mathbf{P}_{\mathbf{b}^{\prime }}^{ \pi_{\mathbf{b}',0}}\left(
X=x_{l},A=j,Y=Mk\right) }\right)  \\
&=&\frac{1}{2v}\left( p+\gamma \right) \log \left( \frac{ p+\gamma
 }{p-\gamma }\right) +\frac{1}{2v}\left( 1-p-\gamma \right) \log
\left( \frac{1-p-\gamma }{1-p+\gamma }\right)  \\
&&+\frac{1}{2v}\left( p-\gamma \right) \log \left( \frac{p-\gamma }{p+\gamma
}\right) +\frac{1}{2v}\left( 1-p+\gamma \right) \log \left( \frac{1-p+\gamma
}{1-p-\gamma }\right)  \\
&=&\frac{1}{v}D_{\mathrm{KL}}(p+\gamma ||p-\gamma ).
\end{eqnarray*}%
For $p=1/2$ and $\gamma \leq 0.1,$ we have by \citet[Lemma 2.7]{tsybakov2009introduction}
\[
D_{\mathrm{KL}}(p+\gamma ||p-\gamma )\leq (2\gamma )^{2}/\left( p^{2}-\gamma
^{2}\right) .
\]%
By picking $\gamma =\frac{1}{4}\sqrt{\frac{v}{n}}\leq 0.1,$ which requires $n \geq \kappa^{(n)}(\Pi)^2$, we have
\[
\max_{\mathbf{P}_{0} * \pi_{0}\in \mathcal{P}(M)}\mathbf{E}_{\left(\mathbf{P}^{\pi_0}_0\right)^n}%
\left[ R_{\mathrm{DRO}}(\pi ,\mathbf{P}_0)\right]\geq \frac{M}{40}\sqrt{\frac{%
v}{n}}.
\]%
Recall that $v=\lceil 4/25 \kappa^{(n)}(\Pi)^2 \rceil$. Therefore, we
have
\[
\max_{\mathbf{P}_{0} * \pi_{0}\in \mathcal{P}(M)}\mathbf{E}_{\left(\mathbf{P}^{\pi_0}_0\right)^n}%
\left[ R_{\mathrm{DRO}}(\pi ,\mathbf{P}_0)\right] \geq \frac{M \kappa ^{(n)}(\Pi )}{100\sqrt{n}}.
\]
\end{proof}
\subsection{Proof of the Bayes DRO policy result in Section \ref{sec:numerical_bayes}}
\label{sec:proof_numerical_small}
\begin{proof}[Proof of Proposition \ref{thm:pi_policy}]
By Lemma \ref{thm:strong_duality}, we have
\begin{eqnarray}
	Q_{\rm DRO} (\pi^*_{\rm DRO})
	&=& \sup_{\pi \in \overline{\Pi}} \sup_{\alpha\geq 0}\left\{ -\alpha \log\E_{\P_0}\left[\exp(-Y(\pi(X))/\alpha)\right] - \alpha \delta\right\}  \notag \\
&=& \sup_{\alpha\geq 0}\sup_{\pi \in \overline{\Pi}}\left\{ -\alpha \log\E_{\P_0}\left[\exp(-Y(\pi(X))/\alpha)\right] - \alpha \delta\right\}.
\label{eq:numerical_duality}
\end{eqnarray}
The inner maximization \eqref{eq:numerical_duality} can be further simplified as
\begin{eqnarray*}
&&\sup_{\pi \in \overline{\Pi}}\left\{ -\alpha \log\E_{\P_0}\left[\exp(-Y(\pi(X))/\alpha)\right] - \alpha \delta\right\} \\
&=&\sup_{\pi \in \overline{\Pi}}\left\{ -\alpha \log\E_{\P_0}[ \E\left[\exp(-Y(\pi(X))/\alpha)|X]\right] - \alpha \delta\right\} \\
&=&-\alpha \log\E_{\P_0}\left [\inf_{\pi \in \overline{\Pi}} \E\left[\exp(-Y(\pi(X))/\alpha)|X\right] \right] - \alpha \delta.
\end{eqnarray*}
Since $ \overline{\Pi}$ contains all measurable policies, we have
\begin{eqnarray*}
\inf_{\pi \in \overline{\Pi}} \E\left[\exp(-Y(\pi(X))/\alpha)|X\right]=\min_{a\in \mathcal{A}}\left\{\E\left[\exp(-Y(a)/\alpha)|X\right]  \right\},
\end{eqnarray*}
and the optimal dual variable is
\begin{equation*}
\alpha^*(\pi^*_{\rm DRO})=\argmax_{\alpha \geq 0} \left \{-\alpha \log\E_{\P_0}\left [\min_{a\in \mathcal{A}}\left\{ \E_{\P_0}\left[\left. \exp\left(-Y(a)/\alpha\right)\right| X\right] \right\}\right] - \alpha \delta \right \}.
\end{equation*}
Finally, we have for any $a \in \mathcal{A}$, the set
\[
\{x \in \mathcal{X}: \pi^*_{\rm DRO} (x) = a\} = \left\{ x \in \mathcal{X}:\E_{\P_0}\left[\left. \exp\left(-Y(a)/\alpha\right)\right| X=x\right]
\leq \E_{\P_0}\left[\left. \exp\left(-Y(a')/\alpha\right)\right| X=x\right],\text{ for }\forall a' \in \mathcal{A}/\{a\}\right\}
\]
is measurable.
\end{proof}
\subsection{Proof of the extension results in Section \ref{sec:extension}}
\begin{proof}[Proof of Lemma \protect\ref{lma:quantile_f_divergence}]
If $k=1$, we have $c_k(\delta)=1$, and thus Lemma \ref{lma:quantile_f_divergence} recovers Lemma \ref{lma:quantile}. For $k\in(1,+\infty)$, notice that
\begin{eqnarray*}
&&\left\vert \sup_{\alpha \in \reals}\left\{ -c_{k}\left( \delta \right)
\mathbf{E}_{\mathbf{P}_{1}}\left[ \left( -Y+\alpha \right) _{+}^{k_{\ast }}%
\right] ^{\frac{1}{k_{\ast }}}+\alpha \right\} -\sup_{\alpha \in \reals%
}\left\{ -c_{k}\left( \delta \right) \mathbf{E}_{\mathbf{P}_{2}}\left[
\left( -Y+\alpha \right) _{+}^{k_{\ast }}\right] ^{\frac{1}{k_{\ast }}%
}+\alpha \right\} \right\vert  \\
&\leq &c_{k}\left( \delta \right) \sup_{\alpha \in \reals}\left\vert
\mathbf{E}_{\mathbf{P}_{1}}\left[ \left( -Y+\alpha \right) _{+}^{k_{\ast }}%
\right] ^{\frac{1}{k_{\ast }}}-\mathbf{E}_{\mathbf{P}_{2}}\left[ \left(
-Y+\alpha \right) _{+}^{k_{\ast }}\right] ^{\frac{1}{k_{\ast }}}\right\vert
\\
&=&c_{k}\left( \delta \right) \sup_{\alpha \in \reals}\left\vert \mathbf{%
E}_{_{\mathbf{P}_{U}}}\left[ \left( -q_{\mathbf{P}_{1}}\left( U\right)
+\alpha \right) _{+}^{k_{\ast }}\right] ^{\frac{1}{k_{\ast }}}-\mathbf{E}_{%
\mathbf{P}_{U}}\left[ \left( -q_{\mathbf{P}_{2}}\left( U\right) +\alpha
\right) _{+}^{k_{\ast }}\right] ^{\frac{1}{k_{\ast }}}\right\vert ,
\end{eqnarray*}%
where $\mathbf{P}_{U}\sim U([0,1])$ and the last equality is based on the
fact that $q_{\mathbf{P}}\left( U\right) \overset{d}{=}\mathbf{P.}$

By the triangular inequality in $L^{k_*}\left( U\right) $ space, we have%
\begin{eqnarray*}
&&\left\vert \mathbf{E}_{_{\mathbf{P}_{U}}}\left[ \left( -q_{\mathbf{P}%
_{1}}\left( U\right) +\alpha \right) _{+}^{k_{\ast }}\right] ^{\frac{1}{%
k_{\ast }}}-\mathbf{E}_{\mathbf{P}_{U}}\left[ \left( -q_{\mathbf{P}%
_{2}}\left( U\right) +\alpha \right) _{+}^{k_{\ast }}\right] ^{\frac{1}{%
k_{\ast }}}\right\vert  \\
&\leq &\mathbf{E}_{_{\mathbf{P}_{U}}}\left[ \left\vert q_{\mathbf{P}%
_{1}}\left( U\right) -q_{\mathbf{P}_{2}}\left( U\right) \right\vert
^{^{k_{\ast }}}\right] ^{\frac{1}{k_{\ast }}} \\
&\leq &\sup_{t\in \lbrack 0,1]}\left\vert q_{\mathbf{P}_{1}}\left( t\right)
-q_{\mathbf{P}_{2}}\left( t\right) \right\vert .
\end{eqnarray*}

\end{proof}
\label{sec:extension_proof}
\begin{proof}[Proof of Lemma \ref{lma:discrete_f_divergence}]
We begin with
\begin{eqnarray}
	&&\left\vert \sup_{\alpha \in \reals}\left\{ -c_{k}\left( \delta \right)
	\left( \sum_{d\in \mathbb{D}}\left[ \left( -d+\alpha \right) _{+}^{k_{\ast }}%
	\mathbf{P}_{1}\left( d\right) \right] \right) ^{\frac{1}{k_{\ast }}}+\alpha
	\right\} -\sup_{\alpha \in \reals}\left\{ -c_{k}\left( \delta \right) \left(
	\sum_{d\in \mathbb{D}}\left[ \left( -d+\alpha \right) _{+}^{k_{\ast }}%
	\mathbf{P}_{2}\left( d\right) \right] \right) ^{\frac{1}{k_{\ast }}}+\alpha
	\right\} \right\vert   \nonumber \\
	&\leq &c_{k}\left( \delta \right) \sup_{\alpha \in \reals}\left\vert \left(
	\sum_{d\in \mathbb{D}}\left[ \left( -d+\alpha \right) _{+}^{k_{\ast }}%
	\mathbf{P}_{1}\left( d\right) \right] \right) ^{\frac{1}{k_{\ast }}}-\left(
	\sum_{d\in \mathbb{D}}\left[ \left( -d+\alpha \right) _{+}^{k_{\ast }}%
	\mathbf{P}_{2}\left( d\right) \right] \right) ^{\frac{1}{k_{\ast }}%
	}\right\vert  \\
	&=&c_{k}\left( \delta \right) \max \left\{ \sup_{\alpha \leq M}\left\vert
	\left( \sum_{d\in \mathbb{D}}\left[ \left( -d+\alpha \right) _{+}^{k_{\ast }}%
	\mathbf{P}_{1}\left( d\right) \right] \right) ^{\frac{1}{k_{\ast }}}-\left(
	\sum_{d\in \mathbb{D}}\left[ \left( -d+\alpha \right) _{+}^{k_{\ast }}%
	\mathbf{P}_{2}\left( d\right) \right] \right) ^{\frac{1}{k_{\ast }}%
	}\right\vert \right. , \\
	&&\left. \sup_{\alpha >M}\left\vert \left( \sum_{d\in \mathbb{D}}\left[
	\left( -d+\alpha \right) _{+}^{k_{\ast }}\mathbf{P}_{1}\left( d\right) %
	\right] \right) ^{\frac{1}{k_{\ast }}}-\left( \sum_{d\in \mathbb{D}}\left[
	\left( -d+\alpha \right) _{+}^{k_{\ast }}\mathbf{P}_{2}\left( d\right) %
	\right] \right) ^{\frac{1}{k_{\ast }}}\right\vert \right\} .
\end{eqnarray}%
We tackle the two cases $\alpha \leq M$ and $\alpha >M$ separately. To ease
of notation, we abbreviate $\mathbf{P}_{1}\left( Y=d\right) ,\mathbf{P}%
_{2}\left( Y=d\right) $ as $\mathbf{P}_{1}\left( d\right) ,\mathbf{P}%
_{2}\left( d\right) $.

1) Case $\alpha \leq M:$ Note that 
\begin{eqnarray}
	&&c_{k}\left( \delta \right) \sup_{\alpha \leq M}\left\vert \left( \left(
	\sum_{d\in \mathbb{D}}\left[ \left( -d+\alpha \right) _{+}^{k_{\ast }}%
	\mathbf{P}_{1}\left( d\right) \right] \right) ^{\frac{1}{k_{\ast }}}-\left(
	\sum_{d\in \mathbb{D}}\left[ \left( -d+\alpha \right) _{+}^{k_{\ast }}%
	\mathbf{P}_{2}\left( d\right) \right] \right) ^{\frac{1}{k_{\ast }}}\right)
	\right\vert   \label{eqn:discrete_log_2} \\
	&=&c_{k}\left( \delta \right) \sup_{\alpha \leq M}\left\vert \left( \left(
	\sum_{d\in \mathbb{D}}\left( \left( -d+\alpha \right) _{+}\left( \mathbf{P}%
	_{1}\left( d\right) \right) ^{\frac{1}{k_{\ast }}}\right) ^{k_{\ast
	}}\right) ^{\frac{1}{k_{\ast }}}-\left( \sum_{d\in \mathbb{D}}\left( \left(
	-d+\alpha \right) _{+}\left( \mathbf{P}_{2}\left( d\right) \right) ^{\frac{1%
		}{k_{\ast }}}\right) ^{k_{\ast }}\right) ^{\frac{1}{k_{\ast }}}\right)
	\right\vert   \nonumber \\
	&\leq &c_{k}\left( \delta \right) \sup_{\alpha \leq M}\max \left\{ \left(
	\sum_{\substack{ d:\mathbf{P}_{1}\left( d\right)  \\ \geq \mathbf{P}%
			_{2}\left( d\right) }}\left( \left( -d+\alpha \right) _{+}\mathbf{P}%
	_{1}\left( d\right) ^{\frac{1}{k_{\ast }}}\right) ^{k_{\ast }}\right) ^{%
		\frac{1}{k_{\ast }}}-\left( \sum_{_{\substack{ d:\mathbf{P}_{1}\left(
				d\right)  \\ \geq \mathbf{P}_{2}\left( d\right) }}}\left( \left( -d+\alpha
	\right) _{+}\mathbf{P}_{2}\left( d\right) ^{\frac{1}{k_{\ast }}}\right)
	^{k_{\ast }}\right) ^{\frac{1}{k_{\ast }}},\right.   \nonumber \\
	&&\left. \left( \sum_{\substack{ d:\mathbf{P}_{1}\left( d\right)  \\ <%
			\mathbf{P}_{2}\left( d\right) }}\left( \left( -d+\alpha \right) _{+}\mathbf{P%
	}_{1}\left( d\right) ^{1/k_{\ast }}\right) ^{k_{\ast }}\right) ^{\frac{1}{%
			k_{\ast }}}-\left( \sum_{_{\substack{ d:\mathbf{P}_{1}\left( d\right)  \\ <%
				\mathbf{P}_{2}\left( d\right) }}}\left( \left( -d+\alpha \right) _{+}\mathbf{%
		P}_{2}\left( d\right) ^{1/k_{\ast }}\right) ^{k_{\ast }}\right) ^{\frac{1}{%
			k_{\ast }}}\right\} .  \nonumber
\end{eqnarray}%
By the $k_{\ast }$-norm triangular inequality and the fact that $\left(
-d+\alpha \right) _{+}\leq M$ for $\alpha \leq M,$ we have%
\begin{eqnarray*}
	&&\left( \sum_{\substack{ d:\mathbf{P}_{1}\left( d\right)  \\ \geq \mathbf{P}%
			_{2}\left( d\right) }}\left( \left( -d+\alpha \right) _{+}\left( \mathbf{P}%
	_{1}\left( d\right) \right) ^{1/k_{\ast }}\right) ^{k_{\ast }}\right) ^{%
		\frac{1}{k_{\ast }}}-\left( \sum_{_{\substack{ d:\mathbf{P}_{1}\left(
				d\right)  \\ \geq \mathbf{P}_{2}\left( d\right) }}}\left( \left( -d+\alpha
	\right) _{+}\left( \mathbf{P}_{2}\left( d\right) \right) ^{1/k_{\ast
	}}\right) ^{k_{\ast }}\right) ^{\frac{1}{k_{\ast }}} \\
	&\leq &c_{k}\left( \delta \right) M\left( \sum_{_{\substack{ d:\mathbf{P}%
				_{1}\left( d\right)  \\ \geq \mathbf{P}_{2}\left( d\right) }}}\left\vert
	\left( \mathbf{P}_{1}\left( d\right) \right) ^{1/k_{\ast }}-\left( \mathbf{P}%
	_{2}\left( d\right) \right) ^{1/k_{\ast }}\right\vert ^{k_{\ast }}\right) ^{%
		\frac{1}{k_{\ast }}}.
\end{eqnarray*}%
Consider the function $h(x)=x^{1/k_{\ast }},$%
\[
h^{\prime }(x)=\frac{1}{k^{\ast }}x^{1/k_{\ast }-1}\leq \frac{1}{k^{\ast }}%
\left( \underline{b}/2\right) ^{1/k_{\ast }-1},\text{ when }x\geq \underline{%
	b}/2.
\]%
Then, when $\mathrm{TV}(\mathbf{P}_{1},\mathbf{P}_{2})\leq \underline{b}/2,$
we have%
\begin{eqnarray*}
	&&c_{k}\left( \delta \right) M\left( \sum_{_{\substack{ d:\mathbf{P}%
				_{1}\left( d\right)  \\ \geq \mathbf{P}_{2}\left( d\right) }}}\left\vert
	\left( \mathbf{P}_{1}\left( d\right) \right) ^{1/k_{\ast }}-\left( \mathbf{P}%
	_{2}\left( d\right) \right) ^{1/k_{\ast }}\right\vert ^{k_{\ast }}\right) ^{%
		\frac{1}{k_{\ast }}} \\
	&\leq &c_{k}\left( \delta \right) M\left( \sum_{_{\substack{ d:\mathbf{P}%
				_{1}\left( d\right)  \\ \geq \mathbf{P}_{2}\left( d\right) }}}\left( \frac{1%
	}{k^{\ast }}\left( \underline{b}/2\right) ^{1/k_{\ast }-1}\left\vert \mathbf{%
		P}_{1}\left( d\right) -\mathbf{P}_{2}\left( d\right) \right\vert \right)
	^{k_{\ast }}\right) ^{\frac{1}{k_{\ast }}} \\
	&=&\frac{c_{k}\left( \delta \right) M}{k^{\ast }}\left( \underline{b}%
	/2\right) ^{1/k_{\ast }-1}\left( \sum_{_{\substack{ d:\mathbf{P}_{1}\left(
				d\right)  \\ \geq \mathbf{P}_{2}\left( d\right) }}}\left\vert \mathbf{P}%
	_{1}\left( d\right) -\mathbf{P}_{2}\left( d\right) \right\vert ^{k_{\ast
	}}\right) ^{\frac{1}{k_{\ast }}} \\
	&\leq &\frac{c_{k}\left( \delta \right) M}{k^{\ast }}\left( \underline{b}%
	/2\right) ^{1/k_{\ast }-1}\sum_{_{\substack{ d:\mathbf{P}_{1}\left( d\right) 
				\\ \geq \mathbf{P}_{2}\left( d\right) }}}\left\vert \mathbf{P}_{1}\left(
	d\right) -\mathbf{P}_{2}\left( d\right) \right\vert  \\
	&\leq &\frac{c_{k}\left( \delta \right) M}{k^{\ast }}\left( \underline{b}%
	/2\right) ^{1/k_{\ast }-1}\mathrm{TV}(\mathbf{P}_{1},\mathbf{P}_{2}).
\end{eqnarray*}%
The same bound holds for $\{d\in \mathbb{D}:\mathbf{P}_{1}\left( d\right) <%
\mathbf{P}_{2}\left( d\right) \},$ which completes this case.

2)  Case $\alpha >M:$ In this case, we have%
\begin{eqnarray*}
	&&c_{k}\left( \delta \right) \sup_{\alpha >M}\left\vert \left( \sum_{d\in 
		\mathbb{D}}\left[ \left( -d+\alpha \right) _{+}^{k_{\ast }}\mathbf{P}%
	_{1}\left( d\right) \right] \right) ^{\frac{1}{k_{\ast }}}-\left( \sum_{d\in 
		\mathbb{D}}\left[ \left( -d+\alpha \right) _{+}^{k_{\ast }}\mathbf{P}%
	_{2}\left( d\right) \right] \right) ^{\frac{1}{k_{\ast }}}\right\vert  \\
	&=&c_{k}\left( \delta \right) \sup_{\alpha >M}\left\vert \left( \sum_{d\in 
		\mathbb{D}}\left[ \left( \alpha -d\right) ^{k_{\ast }}\mathbf{P}_{1}\left(
	d\right) \right] \right) ^{\frac{1}{k_{\ast }}}-\left( \sum_{d\in \mathbb{D}}%
	\left[ \left( \alpha -d\right) ^{k_{\ast }}\mathbf{P}_{2}\left( d\right) %
	\right] \right) ^{\frac{1}{k_{\ast }}}\right\vert .
\end{eqnarray*}%
We will focus on $\left\vert \left( \sum_{d\in \mathbb{D}}\left[ \left(
\alpha -d\right) ^{k_{\ast }}\mathbf{P}_{1}\left( d\right) \right] \right) ^{%
	\frac{1}{k_{\ast }}}-\left( \sum_{d\in \mathbb{D}}\left[ \left( \alpha
-d\right) ^{k_{\ast }}\mathbf{P}_{2}\left( d\right) \right] \right) ^{\frac{1%
	}{k_{\ast }}}\right\vert $  and without loss of generality, we
assume%
\[
\sum_{d\in \mathbb{D}}\left[ \left( \alpha -d\right) ^{k_{\ast }}\mathbf{P}%
_{1}\left( d\right) \right] \geq \sum_{d\in \mathbb{D}}\left[ \left( \alpha
-d\right) ^{k_{\ast }}\mathbf{P}_{2}\left( d\right) \right] .
\]%
Recall that for the function $h(x)=x^{1/k_{\ast }},$ the derivative is 
\[
h^{\prime }(x)=\frac{1}{k^{\ast }}x^{1/k_{\ast }-1}\leq \frac{1}{k^{\ast }}%
\left( \underline{x}\right) ^{1/k_{\ast }-1},\text{ when }x\geq \underline{x}%
.
\]%
Therefore, we have 
\begin{eqnarray}
	&&c_{k}\left( \delta \right) \left( \sum_{d\in \mathbb{D}}\left[ \left(
	\alpha -d\right) ^{k_{\ast }}\mathbf{P}_{1}\left( d\right) \right] \right) ^{%
		\frac{1}{k_{\ast }}}-\left( \sum_{d\in \mathbb{D}}\left[ \left( \alpha
	-d\right) ^{k_{\ast }}\mathbf{P}_{2}\left( d\right) \right] \right) ^{\frac{1%
		}{k_{\ast }}}  \nonumber \\
	&\leq &\frac{c_{k}\left( \delta \right) }{k^{\ast }}\left( \sum_{d\in 
		\mathbb{D}}\left[ \left( \alpha -d\right) ^{k_{\ast }}\mathbf{P}_{2}\left(
	d\right) \right] \right) ^{1/k_{\ast }-1}\left( \sum_{d\in \mathbb{D}}\left[
	\left( \alpha -d\right) ^{k_{\ast }}\left( \mathbf{P}_{1}\left( d\right) -%
	\mathbf{P}_{2}\left( d\right) \right) \right] \right) .  \label{bd:alphaM}
\end{eqnarray}%
Then, when $\mathrm{TV}(\mathbf{P}_{1},\mathbf{P}_{2})\leq \underline{b}/2,$
we have 
\begin{eqnarray}
	&&\sum_{d\in \mathbb{D}}\left[ \left( \alpha -d\right) ^{k_{\ast }}\mathbf{P}%
	_{2}\left( d\right) \right] \geq \frac{\underline{b}}{2}\left( \alpha
	-\min_{d\in \mathbb{D}}d\right) ^{k_{\ast }}  \nonumber \\
	&\Rightarrow &\left( \sum_{d\in \mathbb{D}}\left[ \left( \alpha -d\right)
	^{k_{\ast }}\mathbf{P}_{2}\left( d\right) \right] \right) ^{1/k_{\ast
		}-1}\leq \left( \underline{b}/2\right) ^{1/k_{\ast }-1}\left( \alpha
	-\min_{d\in \mathbb{D}}d\right) ^{1-k_{\ast }}.  \label{denominator:bd}
\end{eqnarray}%
Furthermore, we have 
\begin{eqnarray}
	&&\sum_{d\in \mathbb{D}}\left[ \left( \alpha -d\right) ^{k_{\ast }}\left( 
	\mathbf{P}_{1}\left( d\right) -\mathbf{P}_{2}\left( d\right) \right) \right] 
	\nonumber \\
	&=&\left( \sum_{_{\substack{ d:\mathbf{P}_{1}\left( d\right)  \\ >\mathbf{P}%
				_{2}\left( d\right) }}}\left( \alpha -d\right) ^{k_{\ast }}\left( \mathbf{P}%
	_{1}\left( d\right) -\mathbf{P}_{2}\left( d\right) \right) \right) -\left(
	\sum_{_{\substack{ d:\mathbf{P}_{1}\left( d\right)  \\ <\mathbf{P}_{2}\left(
				d\right) }}}\left( \alpha -d\right) ^{k_{\ast }}\left( \mathbf{P}_{2}\left(
	d\right) -\mathbf{P}_{1}\left( d\right) \right) \right)   \nonumber \\
	&\leq &\left( \alpha -\min_{d\in \mathbb{D}}d\right) ^{k_{\ast }}\sum_{
		_{\substack{ d:\mathbf{P}_{1}\left( d\right)  \\ >\mathbf{P}_{2}\left(
				d\right) }}}\left( \mathbf{P}_{1}\left( d\right) -\mathbf{P}_{2}\left(
	d\right) \right) -\left( \alpha -\max_{d\in \mathbb{D}}d\right) ^{k_{\ast
	}}\sum_{_{\substack{ d:\mathbf{P}_{1}\left( d\right)  \\ <\mathbf{P}%
				_{2}\left( d\right) }}}\left( \mathbf{P}_{2}\left( d\right) -\mathbf{P}%
	_{1}\left( d\right) \right)   \nonumber \\
	&=&\left( \left( \alpha -\min_{d\in \mathbb{D}}d\right) ^{k_{\ast }}-\left(
	\alpha -\max_{d\in \mathbb{D}}d\right) ^{k_{\ast }}\right) \sum_{_{\substack{
				d:\mathbf{P}_{1}\left( d\right)  \\ >\mathbf{P}_{2}\left( d\right) }}}\left( 
	\mathbf{P}_{1}\left( d\right) -\mathbf{P}_{2}\left( d\right) \right) .
	\label{nominator:total}
\end{eqnarray}%
The last equation is due to    
\[
\sum_{_{ d:\mathbf{P}_{1}\left( d\right)   >\mathbf{P}_{2}\left(
			d\right) }}\left( \mathbf{P}_{1}\left( d\right) -\mathbf{P}_{2}\left(
d\right) \right) =\sum_{_{ d:\mathbf{P}_{1}\left( d\right)   <%
			\mathbf{P}_{2}\left( d\right) }}\left( \mathbf{P}_{2}\left( d\right) -%
\mathbf{P}_{1}\left( d\right) \right) .
\]%
We further note that%
\begin{equation}
	\sum_{_{d:\mathbf{P}_{1}\left( d\right) >\mathbf{P}_{2}\left( d\right)
	}}\left( \mathbf{P}_{1}\left( d\right) -\mathbf{P}_{2}\left( d\right)
	\right) =\mathrm{TV}(\mathbf{P}_{1},\mathbf{P}_{2}),  \label{eq:TV:P12}
\end{equation}%
and 
\begin{eqnarray}
	&&\left( \left( \alpha -\min_{d\in \mathbb{D}}d\right) ^{k_{\ast }}-\left(
	\alpha -\max_{d\in \mathbb{D}}d\right) ^{k_{\ast }}\right)   \nonumber \\
	&\leq &\left( \max_{d\in \mathbb{D}}d-\min_{d\in \mathbb{D}}d\right) \left(
	\alpha -\min_{d\in \mathbb{D}}d\right) ^{k_{\ast }-1}  \nonumber \\
	&\leq &M\left( \alpha -\min_{d\in \mathbb{D}}d\right) ^{k_{\ast }-1}.
	\label{iq:nominator:lp}
\end{eqnarray}%
By combining bounds (\ref{bd:alphaM}) - (\ref{iq:nominator:lp}), we have 
\begin{eqnarray*}
	&&c_{k}\left( \delta \right) \left( \sum_{d\in \mathbb{D}}\left[ \left(
	\alpha -d\right) ^{k_{\ast }}\mathbf{P}_{1}\left( d\right) \right] \right) ^{%
		\frac{1}{k_{\ast }}}-\left( \sum_{d\in \mathbb{D}}\left[ \left( \alpha
	-d\right) ^{k_{\ast }}\mathbf{P}_{2}\left( d\right) \right] \right) ^{\frac{1%
		}{k_{\ast }}} \\
	&\leq &\frac{c_{k}\left( \delta \right) M}{k_{\ast }}\left( \underline{b}%
	/2\right) ^{1/k_{\ast }-1}\mathrm{TV}(\mathbf{P}_{1},\mathbf{P}_{2}), \text{ for any } \alpha>M,
\end{eqnarray*}%
which completes the proof.
\end{proof}

%\begin{proof}{Proof of Theorem \ref{thm:ld2}}
%For the lower bound, we define notions similar with $D_{KL}(p||q)$ and g(p) in
%Appendix\ref{sec:proof_ld}. For $p,q\in \lbrack 0,1],$ we define
%\[
%D_k(p||q)=p\log f_k(q/p)+(1-p)\log f_k((1-q)/(1-p)).
%\]%
%Let $g(p)=\inf_{D_k(p||q)\leq \delta }q.$ Then, for $p_k=c_k^{-k/(k-1)}$, by \cite[Lemma 12]{duchi2018learning}, we have
%\[
%g(1-p+\gamma)-g(1-p+\gamma)
%\]
%\end{proof}
\section{Experiments}
\subsection{
Optimization of multi-linear policy
}
\label{sec:appendix-experiment}
This section provides implementation details for how to compute $\argmin_{\Theta\in \mathbf{R}^{(p+1)\times d}}\hat{W}_n(\pi_{\Theta},\alpha)$ where $\pi_{\Theta}$ is the multilinear policy associated with parameter $\Theta$. Recall the definition of $\hat{W}_n(\pi,\alpha)$ from definition \ref{def:phi_alpha} that
\begin{equation*}
\hat{W}_n(\pi,\alpha) =
\frac{1}{nS_n^\pi}\sum_{i=1}^{n} W_i(\pi, \alpha)
=
\frac{\sum_{i=1}^n
\frac{\mathbf{1}\{\pi(X_i) = A_i\}}{\pi_0(A_i\mid X_i)}\exp(-Y_i(A_i)/\alpha)
}{\sum_{i=1}^{n}\frac{\mathbf{1\{}\pi(X_{i})=A_{i}\mathbf{\}}}{\pi _{0}\left( A_{i}|X_{i}\right) }}
.
\end{equation*}

As we did in Section \ref{sec:numerical_linear_class}, we employ the smooth approximation of the indicator function
$$
\mathbf{1\{}\pi_{\Theta}(X_{i})=A_{i}\mathbf{\}}
\approx
\frac{
\exp(\theta_{A_i}^{\top}X_i)}{ \sum_{a = 1}^{d}\exp(\theta_{a}^{\top}X_i)}.
$$
Now for $i = 1,\ldots, n$, we the smooth weight function
$p_i: \mathbf{R}^{(p+1)\times d} \rightarrow \mathbf{R}_{+}$ as
$$
p_i(\Theta) \triangleq
\frac{\exp(\theta_{A_i}^{\top}X_i)}{\pi_{0}\left( A_{i}|X_{i}\right)
\sum_{a = 1}^{d}\exp(\theta_{a}^{\top}X_i)
},
$$
then the estimator $\hat{W}_n(\pi_{\Theta},\alpha)$ admits the smooth approximation
\begin{equation*}
\hat{W}_n(\pi_{\Theta},\alpha)
\approx
\tilde{W}_n(\pi_{\Theta},\alpha)
\triangleq
\frac{\sum_{i=1}^n
p_i(\Theta) \exp(-Y_i(A_i)/\alpha)
}{\sum_{i=1}^{n} p_i(\Theta)}
.
\end{equation*}
In addition, we have
\begin{align*}
\nabla_{\Theta}
\tilde{W}_n(\pi_{\Theta},\alpha)
&
=\tilde{W}_n(\pi_{\Theta},\alpha) \nabla_{\Theta}
\log \left(
\tilde{W}_n(\pi_{\Theta},\alpha)
\right)\\
&\approx
\tilde{W}_n(\pi_{\Theta},\alpha)
\cdot
\left(
\nabla_{\Theta}
\log \left(
\sum_{i=1}^n
p_i(\Theta) \exp(-Y_i(A_i)/\alpha)
\right)
-\nabla_{\Theta} \log \left(
\sum_{i=1}^{n} p_i(\Theta)
\right)\right)
\\
&=
\tilde{W}_n(\pi_{\Theta},\alpha)
\cdot
\left(
\frac{\sum_{i=1}^n
\nabla_{\Theta}p_i(\Theta) \exp(-Y_i(A_i)/\alpha)}{\sum_{i=1}^n
p_i(\Theta) \exp(-Y_i(A_i)/\alpha)}
-
\frac{\sum_{i=1}^n
\nabla_{\Theta}p_i(\Theta) }{\sum_{i=1}^n
p_i(\Theta)}
\right).
\end{align*}
Therefore, we can employ gradient to descent to solve for $\Theta$ that minimizes $\tilde{W}_n(\pi_{\Theta},\alpha)$, which approximately minimizes $\hat{W}_n(\pi_{\Theta},\alpha)$ as well. This is how we solve for $\argmin_{\Theta\in \mathbf{R}^{(p+1)\times d}}\hat{W}_n(\pi_{\Theta},\alpha)$ in implementation.

\subsection{Experimental details of $\delta$ selection in Section \ref{sec:select_delta}}
\label{appendix:experiments_details}
In this section we will provide further details on the $\delta$ selection experiment.

Recall that we intend to estimate $\delta$ based on the data of 100 cities in the training set. To this end, we will
partition the data in 20\% of cities as our validation set with distribution
denoted by $\P ^{20}$, and we use $\P ^{80}$ to denote the distribution of
the rest 80\% of the training set. We will explain the detail of how to estimate $D(\P ^{20}||\P^{80})$ in the rest of this section.

We first explain how to estimate the divergence between marginal distributions of $X$, which is denoted by $D(\P_X^{20}||\P _X^{80})$. A directly computation using the sampled distributions $\P_X^{20}$ and $\P _X^{80}$ may result in infinite value, because (1) some features such as year of birth contains outliers whose value only appears in $\P_X^{20}$ or $\P_X^{80}$; (2) $X$ is a nine-dimensional vector, which exaggerates the prior problem. In view of that the demographic features (year of birth, sex, household size) are weakly correlated with the historical voting records, we will compute the divergence on them separately. In order to avoid infinite KL-divergence, we first regroup two demographic features, Year of Birth (YoB) and Household Size (HS), according to the following rules:
\begin{equation*}
    \mbox{YoB group}
    = \begin{cases}
    1 & \mbox{if YoB} \leq 1943\\
    2 & \mbox{if }1943< \rm{YoB} \leq 1952\\
    3 & \mbox{if }1952< \rm{YoB} \leq 1959\\
    4 & \mbox{if }1959< \rm{YoB} \leq 1966\\
    5 & \mbox{if }1966< \rm{YoB}\\
    \end{cases}
    \qquad
    \mbox{HS group}
    = \begin{cases}
    1 & \mbox{if HS} = 1\\
    2 & \mbox{if HS} = 2\\
    3 & \mbox{if HS} = 3\\
    4 & \mbox{if HS} \geq 4\\
    \end{cases}
\end{equation*}
After the regrouping, we define the demographic feature vector $X_{\rm{demo}} = (\mbox{
HS group, YoB group,
sex})$, and compute the KL-divergence
$D(\P^{20}_{X_{\rm{demo}}}||\P^{80} _{X_{\rm{demo}}})$. The historical voting record vector is defined as $
X_{\rm{rec}} = (\mbox{g2004, g2002, g2000, d2004, d2002, d2000})$, and we compute its divergence $D(\P^{20}_{X_{\rm{rec}}}||\P^{80} _{X_{\rm{rec}}})$ directly. We will use the sum $D(\P^{20}_{X_{\rm{demo}}}||\P^{80} _{X_{\rm{demo}}})+D(\P^{20}_{X_{\rm{rec}}}||\P^{80} _{X_{\rm{rec}}})$ as approximation for the divergence between $\P_X^{20}$ and $\P _X^{80}$.

Next, we will explain how to apply logistic regression to estimate $\E_{\P _X^{20}}[D(
\P_Y^{20}|X||\P_Y^{80}|X)]$. We independently fit two logistic regression of $Y\sim X$, using data corresponding to $\P^{20}$ and $\P^{80}$ respectively. The result of logistic regression implies a fitted conditional distribution of $Y$ given $X$, i.e., $\P(Y=1|X) = (1+\exp(\hat{\beta}_0+\hat{\beta}^{\top}X))^{-1}$, where $\hat{\beta}_0$ and $ \hat{\beta}$ are fitted parameters for the logistic regression. We denoted by $\hat{P}_Y^{20}|X$ as the fitted conditional distribution using the data $\P^{20}$, and $\hat{P}_Y^{80}|X$ as the fitted conditional distribution using the data $\P^{80}$. Since both $\hat{P}_Y^{20}|X$ and $\hat{P}_Y^{80}|X$ are Bernoulli distribution parameterized as a function of $X$, the KL-divergence between them $D(
\hat{\P}_Y^{20}|X||\hat{\P}_Y^{80}|X)$ can also be computed in closed form as a function of $X$. Finally, we compute the average value of the estimated KL-divergence using the distribution of $\P_X^{20}$.

To conclude this section, we provide the full formula used in the computation:
$$
D(\P ^{20}||\P^{80})
\approx
D(\P^{20}_{X_{\rm{demo}}}||\P^{80} _{X_{\rm{demo}}})+D(\P^{20}_{X_{\rm{rec}}}||\P^{80} _{X_{\rm{rec}}})+
\E_{\P _X^{20}}[D(
\hat{\P}_Y^{20}|X||\hat{\P}_Y^{80}|X)].
$$

\end{document}